\documentclass{article}

\PassOptionsToPackage{numbers, compress}{natbib}

\usepackage[main, preprint]{neurips_2026}

\usepackage[utf8]{inputenc} 
\usepackage[T1]{fontenc}    
\usepackage{hyperref}       
\usepackage{url}            
\usepackage{booktabs}       
\usepackage{amsfonts}       
\usepackage{nicefrac}       
\usepackage{microtype}      
\usepackage{xcolor}         

\usepackage[noend]{algorithmic}
\usepackage[ruled]{algorithm}
\usepackage{comment}

\usepackage{mathtools}
\usepackage{amsmath,amsfonts,amssymb,amsthm}
\usepackage{bbm}
\usepackage{enumitem}
\usepackage{cleveref}
\usepackage{bm}
\usepackage{mathrsfs}
\usepackage{multirow}
\usepackage{graphicx}
\usepackage{subcaption}
\usepackage{tikz}
\usetikzlibrary{calc}
\usetikzlibrary{decorations.pathreplacing}
\makeatletter
\newtheorem*{rep@theorem}{\normalfont\bfseries\rep@title}
\newcommand{\newreptheorem}[2]{%
\newenvironment{rep#1}[1]{%
  \par\medskip
 \def\rep@title{#2~\ref{##1}}%
 \begin{rep@theorem}}%
 {\end{rep@theorem}}}
\makeatother
\newtheorem{theorem}{Theorem}[section]
\newreptheorem{theorem}{Theorem}
\newtheorem{definition}{Definition}[section]
\newtheorem{assumption}{Assumption}[section]

\newtheorem{corollary}{Corollary}[theorem]
\newtheorem{lemma}[theorem]{Lemma}
\newreptheorem{lemma}{Lemma}
\newtheorem{proposition}[theorem]{Proposition}
\newreptheorem{proposition}{Proposition}
\newtheorem{remark}{Remark}[section]

\newcommand\mc{\mathcal}
\newcommand\mb{\mathbb}

\DeclareMathOperator{\sdp}{\sigma_{\text{DP}}}
\newcommand{\R}{\mathbb{R}}

\newcommand{\E}{\mathbb{E}}
\newcommand{\Prob}{\mathbb{P}}

\newcommand{\A}{\mathcal{A}}

\allowdisplaybreaks%

\makeatletter
\let\save@mathaccent\mathaccent
\newcommand*\if@single[3]{%
  \setbox0\hbox{${\mathaccent"0362{#1}}^H$}%
  \setbox2\hbox{${\mathaccent"0362{\kern0pt#1}}^H$}%
  \ifdim\ht0=\ht2 #3\else #2\fi
  }
\newcommand*\rel@kern[1]{\kern#1\dimexpr\macc@kerna}
\newcommand*\widebar[1]{\@ifnextchar^{{\wide@bar{#1}{0}}}{\wide@bar{#1}{1}}}
\newcommand*\wide@bar[2]{\if@single{#1}{\wide@bar@{#1}{#2}{1}}{\wide@bar@{#1}{#2}{2}}}
\newcommand*\wide@bar@[3]{%
  \begingroup
  \def\mathaccent##1##2{%
    \let\mathaccent\save@mathaccent
    \if#32 \let\macc@nucleus\first@char \fi
    \setbox\z@\hbox{$\macc@style{\macc@nucleus}_{}$}%
    \setbox\tw@\hbox{$\macc@style{\macc@nucleus}{}_{}$}%
    \dimen@\wd\tw@
    \advance\dimen@-\wd\z@
    \divide\dimen@ 3
    \@tempdima\wd\tw@
    \advance\@tempdima-\scriptspace
    \divide\@tempdima 10
    \advance\dimen@-\@tempdima
    \ifdim\dimen@>\z@ \dimen@0pt\fi
    \rel@kern{0.6}\kern-\dimen@
    \if#31
      \overline{\rel@kern{-0.6}\kern\dimen@\macc@nucleus\rel@kern{0.4}\kern\dimen@}%
      \advance\dimen@0.4\dimexpr\macc@kerna
      \let\final@kern#2%
      \ifdim\dimen@<\z@ \let\final@kern1\fi
      \if\final@kern1 \kern-\dimen@\fi
    \else
      \overline{\rel@kern{-0.6}\kern\dimen@#1}%
    \fi
  }%
  \macc@depth\@ne
  \let\math@bgroup\@empty \let\math@egroup\macc@set@skewchar
  \mathsurround\z@ \frozen@everymath{\mathgroup\macc@group\relax}%
  \macc@set@skewchar\relax
  \let\mathaccentV\macc@nested@a
  \if#31
    \macc@nested@a\relax111{#1}%
  \else
    \def\gobble@till@marker##1\endmarker{}%
    \futurelet\first@char\gobble@till@marker#1\endmarker
    \ifcat\noexpand\first@char A\else
      \def\first@char{}%
    \fi
    \macc@nested@a\relax111{\first@char}%
  \fi
  \endgroup
}
\makeatother

\newcommand{\weakprivacy}{orange!80!red}
\newcommand{\strongprivacy}{violet!60}
\usepackage[disable,textsize=tiny]{todonotes}
\newcommand{\comm}[3][noinline]{\todo[#1, size=\tiny]{#2: #3}}

\newcommand{\tbtodo}[2][noinline]{\comm[color=orange,#1]{TB}{#2}}

\newcommand{\autodo}[2][noinline]{\comm[color=green,#1]{AB}{#2}}

\title{Unveiling the Non-Monotonic Effect of Privacy on Generalization under Byzantine Robustness}

\author{%
  Thomas Boudou \\
  Inria, Université de Montpellier, INSERM\\
  Montpellier, France \\
  \texttt{thomas.boudou@inria.fr} \\
  \And\
  Batiste Le Bars \\
  Inria, Université de Lille \\
  Lille, France \\
  \texttt{batiste.le-bars@inria.fr} \\
  \AND\
  Nirupam Gupta \\
  University of Copenhagen \\
  Copenhagen, Denmark \\
  \texttt{nigu@di.ku.dk} \\
  \And\
  Aurélien Bellet \\
  Inria, Université de Montpellier, INSERM\\
  Montpellier, France \\
  \texttt{aurelien.bellet@inria.fr} \\
}

\begin{document}

\maketitle

\begin{abstract}


    Recent work has established a fundamental trilemma between Byzantine robustness, local differential privacy (LDP), and optimization error in distributed learning. 
    We show that this trilemma does not universally extend to generalization error, but instead depends critically on the privacy regime.
    Specifically, in the high-noise regime (strong privacy), we prove that increasing privacy reduces the generalization error, i.e., there is no tension between robustness and privacy. 
    In the low-noise regime (weaker privacy), however, the tension between robustness and privacy reappears and increasing privacy indeed degrades generalization. 
    Our theory explains this surprising non-monotonic behavior of the generalization error via matching lower and upper bounds on the algorithmic stability of Byzantine-robust distributed learning under LDP constraints. 
    We corroborate and further analyze these theoretical findings with empirical evaluations. 

\end{abstract}

\section{Introduction}\label{I-intro}

Learning high-quality models often requires access to large amounts of data and substantial computational resources, making distributed learning systems such as Federated Learning (FL)~\cite{MAL-083} a promising approach to leverage data and computation across many participants.
FL is typically implemented via distributed gradient descent algorithms coordinating local updates across participants.
However, sharing model updates is known to leak information about the underlying training data, which can reduce participants' willingness to contribute and raises challenges for complying with data protection regulations.
Local Differential Privacy (LDP, or local DP) is a natural privacy framework for distributed learning under the honest-but-curious threat model, where the server follows the protocol but may attempt to infer sensitive information from participant updates~\cite{duchi2013local, MAL-083}. 
This assumption is practically motivated, as the server is often a third party that cannot be fully trusted. 
LDP typically enforces privacy by having each participant locally add noise to its updates before sending them.
In addition to privacy concerns, real-world distributed systems face a range of unpredictable malfunctions, from benign hardware failures to adversarial misbehavior (e.g., poisoning attacks), broadly categorized under the umbrella of Byzantine failures~\cite{gerraoui-byzantine-primer}. 
This motivates a second line of defense known as Byzantine robustness, which aims to ensure reliable learning through robust aggregation rules that mitigate the influence of corrupted or malicious updates.
While both LDP and Byzantine robustness are well-studied in isolation, their interplay remains a largely underexplored frontier despite being highly relevant in practical deployments.

\textbf{Related work.} A first line of work has investigated the impact of LDP alone, establishing a tight privacy-utility trade-off, where LDP constraints inflate the optimization of the empirical risk by a variance penalty~\cite{wang-LDP-JMLR:v21:19-253,pmlr-v130-girgis21a,pmlr-v151-noble22a}.
LDP has also been studied in federated stochastic convex optimization~\cite{duchi2013local,pmlr-v235-reshef24a}, where the focus is on excess population risk guarantees, leading to similar qualitative insights on the privacy–utility trade-off.

Focusing on Byzantine robustness alone, recent work has extensively studied $(f,\kappa)$-robust aggregation as a principled approach to achieve optimal optimization guarantees in the presence of Byzantine failures~\cite{pmlr-v206-allouah23a}. 
These results show that, while robust aggregation can tolerate a fraction of Byzantine participants, it introduces an unavoidable bias that prevents the optimization error from vanishing when data are heterogeneous across participants.
This limitation extends to the statistical learning setting~\cite{boudou2025generalization}, where the generalization error (the discrepancy between training and population risks) is further amplified by the constraints imposed by robust aggregation.
However, their analysis does not consider formal privacy constraints.

Recently, a fundamental \emph{trilemma} has been demonstrated for distributed algorithms that aim to be both Byzantine robust and satisfy LDP~\cite{pmlr-v202-allouah23a}.
More precisely, this work shows that the interaction between Byzantine robustness mechanisms and LDP constraints induces an additional optimization error term that does not appear when considering either requirement in isolation. 
This unavoidable term reveals an intrinsic tension between privacy and robustness.
While this trilemma precisely characterizes the optimization dynamics in terms of empirical risk, the true measure of a model's performance is its population risk, i.e., its behavior on unseen data. 
This risk is typically decomposed into the sum of optimization and generalization errors (see~\Cref{II-problem-formulation})~\cite{hardt2016train}.
This naturally raises the question: \emph{does this trilemma extend to generalization?}
While it is intuitive that injecting privacy noise prior to aggregation can hinder the server's ability to distinguish between honest and Byzantine updates, thereby increasing optimization error, its effect on generalization is less clear. 
Indeed, such noise is often associated with an implicit regularization effect that can help control and even improve generalization \citep{wang2016privacylearning,pmlr-v134-neu21a}.
This creates a fundamental tension: \emph{does LDP noise merely exacerbate the impact of Byzantine robustness, or can it also mitigate them?} 
The interplay between these two opposing effects remains poorly understood and is the focus of our work.

\textbf{Contributions.} We investigate this question and show that the privacy-robustness-optimization trilemma does not directly extend to generalization. 
By tightly analyzing the algorithmic stability of the algorithm introduced in~\cite{pmlr-v202-allouah23a} (i.e., \texttt{Safe-D}(\texttt{S})\texttt{DG}, detailed in~\Cref{safe-dsgd}), which achieves optimal optimization error in the considered setting, we uncover a fundamental non-monotonic relationship between privacy guarantees and generalization error under Byzantine robustness. 
Crucially, this behavior is not driven solely by the raw noise magnitude, but is characterized by the ability of Byzantine participants to reliably detect the presence of individual data points, i.e., the performance of a \emph{Membership Inference Attack} (MIA). Indeed, we show that membership information of honest participants' data is crucial to Byzantine participants' power to attack the learning process. Specifically, we identify a threshold corresponding to a qualitative change in the privacy regime, where an optimal MIA achieves a non-trivial probability of success. 
This transition delineates two distinct regimes, each exhibiting fundamentally different generalization behavior.
\textbf{(i) Weak privacy regime (low noise):} In this regime, the injected noise interacts adversely with Byzantine-robust aggregation, leading to a degradation of generalization performance.
Quantitatively, the generalization error increases with the noise level and with the fraction of Byzantine participants, causing an initial deterioration.
\textbf{(ii) Strong privacy regime (high noise):} Beyond the MIA threshold, the attack becomes ineffective and the system enters a regime consistent with classical differential privacy behavior. In this regime, increasing the noise level improves stability, and the generalization error decreases, eventually vanishing as privacy strengthens.

We validate our theoretical findings, including our tight stability bounds, through empirical evaluations. 
These confirm that the identified worst-case behavior is relevant in realistic scenarios. 
Overall, our results refine the characterization of the privacy-robustness-utility trade-off in distributed learning.\looseness=-1

\section{Preliminaries}\label{II-Background-motivation}\label{II-problem-formulation}

\textbf{Problem setting.} We consider a distributed setup with $n$ participants coordinated by a central server. 
We assume this system is subject to Byzantine failures, and requires to guarantee LDP.

We consider the Byzantine failures threat model~\cite{lamport1982byzantine,gerraoui-byzantine-primer}, in which at most $f$ $(< \frac{n}{2})$ participants can act arbitrarily. Their identities are unknown to the server and they may deviate from the prescribed algorithm, collude, and send adversarial updates while having access to all information exchanged between honest participants and the server. 
We note $\mc{H} \subseteq [n] \vcentcolon= \{1, \ldots, n\}$ the set of \textit{honest} participants, with $|\mc{H}| = n-f$.
Each honest participant $i \in \mc{H}$ holds a local dataset $\mc{D}_i = \{z^{(i,1)}, \ldots, z^{(i,m)} \}$ composed of $m$ i.i.d.\ data points (or samples) from an input space $\mathcal{Z}$ drawn from a distribution $p_i$.\looseness=-1

We denote by $\mc{S} = \cup_{i\in\mc{H}}\mc{D}_i$ the overall dataset. 
Given a parameter vector $\theta \in \Theta \subset \mb{R}^d$ representing the model, a data point $z \in \mc{Z}$ incurs a loss defined by a real-valued function $\ell(\theta; z)$. 
The goal is to minimize the \emph{population} risk over the honest participants
\[
    \textstyle 
    R_{\mc{H}}(\theta) = \frac{1}{|\mc{H}|} \sum_{i\in\mc{H}} \mb{E}_{z \sim p_i}[ \ell(\theta, z) ].
\]
Since only a finite number of samples are available from each distribution, we approximate this objective by minimizing the corresponding \emph{empirical} risk over the (unknown) set of honest participants
\[
    \textstyle 
    \widehat{R}_{\mc{H}}(\theta) = \frac{1}{|\mc{H}|} \sum_{i\in\mc{H}} \widehat{R}_{i}(\theta) = \frac{1}{n-f} \sum_{i\in\mc{H}} \frac{1}{m} \sum_{z \in \mc{D}_i} \ell(\theta, z).
\]
The expected excess population risk of a distributed learning algorithm's output $\A: \mathcal{Z}^{(n-f)m} \to \varTheta, \mc{S} \mapsto \mc{A(\mc{S})}$ can be decomposed into generalization and optimization errors (e.g.~\cite{hardt2016train}) as follows\looseness=-1
\begin{multline}\label{gen-error-ineq}
    \textstyle 
    \mb{E} \big[ R_{\mc{H}}(\mc{A(\mc{S})}) - \inf\limits_{\theta \in \Theta} R_{\mc{H}}(\theta) \big]
    \! \leq \! \mb{E} \big[ R_{\mc{H}}(\mc{A(\mc{S})})  - \widehat{R}_{\mc{H}}(\mc{A(\mc{S})}) \big]
    + \mb{E} \big[ \widehat{R}_{\mc{H}}(\mc{A(\mc{S})}) - \inf\limits_{\theta \in \Theta} \widehat{R}_{\mc{H}}(\theta) \big]. \!
\end{multline}
Prior work on robust distributed learning has focused primarily on the optimization error, i.e., the second term on the right-hand side. In contrast, our goal is to study the generalization error, corresponding to the first term.

\textbf{Robust distributed optimization.} 
Robust distributed optimization algorithms aim to achieve low empirical risk. 
They are typically adaptations of standard first-order iterative optimization algorithms like gradient descent ($\mathrm{GD}$) or stochastic gradient descent ($\mathrm{SGD}$). 
These algorithms are made resilient by replacing the server-side averaging operator by a robust aggregation rule $F: (\mb{R}^d)^{n} \to \mb{R}^d$~\citep{gerraoui-byzantine-primer}.
Formally, given a learning rate $\gamma$, a parameter $\theta_t$ at iteration $t \in \{0, \ldots, T-1\}$ is updated as follows
\[
    \textstyle \theta_{t+1} = G^{F}_{\gamma}(\theta_t) := \theta_t - \gamma F(g^{(1)}_t, \ldots, g^{(n)}_t),
\]
where $g^{(i)}_t$ denotes the update sent by participant $i$ at step $t$. 
For an honest participant $i \in \mc{H}$, we have $g^{(i)}_t = \nabla \widehat{R}_i (\theta_t)$ under $\mathrm{GD}$, or $g^{(i)}_t = \nabla \ell (\theta_t; z^{(i)}_t)$ under $\mathrm{SGD}$, with $z^{(i)}_t$ uniformly sampled from the local dataset $\mc{D}_i$. 
For a misbehaving participant $i \notin \mc{H}$, $g^{(i)}_t$ is an arbitrary vector in $\mathbb{R}^d$.

Recent advances in robust distributed optimization have identified key properties sufficient for a robust aggregation rule to ensure the strong notion of $(f,\rho)$-resilience with respect to the empirical risk~\citep{pmlr-v206-allouah23a,pmlr-v202-allouah23a}. 
The following definition encompasses a wide range of aggregation rules and has been shown to yield tight resilience guarantees across a variety of practical distributed learning settings.

\begin{definition}\label{robustness-definition}
    Let $\|\cdot\|_{\operatorname{sp}}$ denote the spectral norm for p.s.d.\ matrices (i.e., the largest eigenvalue). 
    Let $n \geq 1$, $0 \leq f < n/2$ and $\kappa \geq 0$. An aggregation rule $F$ is \textit{$(f, \kappa)$-robust} if for any $g_1, \ldots, g_n \in \mb{R}^d$, 
    any set $S \subset [n]$ of size $n-f$, with $\overline{g}_S = \frac{1}{|S|} \sum_{i \in S} g_i$ and $\Sigma_{S} = \frac{1}{|S|} \sum_{i \in S} (g_i-\overline{g}_S) {(g_i-\overline{g}_S)}^{\intercal}$,
    \[
        \textstyle \| F(g_1, \ldots, g_n) - \overline{g}_S \|_2^2 \leq \kappa \|\Sigma_S\|_{\operatorname{sp}}.
    \]
    $f$ and $\kappa$ are referred to as the robustness parameter and robustness coefficient of $F$, respectively.
\end{definition}
An aggregation rule $F$ that is $(f, \mathcal{O}(f/n))$-robust achieves optimal $(f, \rho)$-resilience. 
Specifically, iterative methods using such $F$ attain an optimization error that matches the information-theoretic lower bound for first-order optimization under $f$ Byzantine participants~\cite{pmlr-v202-allouah23a}. 
An example is the \textit{smallest maximum eigenvalue averaging} ($\mathrm{SMEA}$)~\cite{pmlr-v202-allouah23a}.
\begin{definition}\label{smea-definition}
    Let $f, n \in \mb{N}$, $f < n/2$ and $g_1, \ldots, g_n \in \mb{R}^d$,
    the $\mathrm{SMEA}$ outputs
    \[
        \overline{g}_{S^*} = \frac{1}{|S^*|} \sum_{i\in S^*_t} g_{i} 
        \quad\text{with}\quad S^* \in \underset{S \subseteq \{1, \ldots, n\}, |S|=n-f}{\operatorname{argmin}} \lambda_{\max} \bigg( \frac{1}{|S|} \sum_{i \in S} \left(g_i-\overline{g}_S \right) {\left(g_i - \overline{g}_S \right)}^\intercal \bigg).
    \]
\end{definition}
Note that the use of the spectral norm in~\Cref{robustness-definition} is motivated by our setting. 
Under LDP constraints, which are usually achieved by adding isotropic noise to gradients, the empirical covariance matrix of honest participants' gradients trace typically scales with the dimension, which is not the case when using the spectral norm, hence the usefulness of this matrix norm in high dimension.
This highlights the focus on $\mathrm{SMEA}$, which was specifically designed for our setting~\cite{pmlr-v202-allouah23a}.
The recent Covariance Bound Agnostic-Filter ($\mathrm{CAF}$)~\cite{allouah2025towards} is a computationally efficient variant of $\mathrm{SMEA}$. 
\looseness=-1

\textbf{Differential privacy.} 
A primary objective for honest participants is to protect the privacy of their datasets. 
To provide a rigorous, quantifiable guarantee, we leverage differential privacy (DP)~\citep{dwork-dp-2006}. 
DP ensures that an algorithm's output is indistinguishable whether or not any single individual's data is included in the input dataset. 
We adopt the standard definition of $(\varepsilon, \delta)$-DP.
\looseness=-1
\begin{definition}
    Let $\varepsilon > 0$, $0 \leq \delta \leq 1$. 
    A randomized algorithm $\mathcal{A}$ is $(\varepsilon, \delta)$-DP if for all neighboring honest datasets $S$ and $S'$ (differing in a single entry) and for all measurable subset of output $O$,
    \[
        \textstyle \mb{P}(\mathcal{A}(S) \in O) \le e^\varepsilon \mb{P}(\mathcal{A}(S') \in O) + \delta
    \]
    In the distributed setting, let $Z_i(S)$ denote the (random) transcript of communications between the server and honest participant $i$ generated by $\mathcal{A}(S)$. 
    $\mathcal{A}$ is $(\varepsilon, \delta)$-LDP if $\forall i \in \mathcal{H}$, $Z_i$ is $(\varepsilon, \delta)$-DP.\looseness=-1
\end{definition}
To summarize, our threat model considers two types of adversaries. 
First, we model the central server as honest-but-curious: it follows the prescribed protocol (i.e., it is honest) but may simultaneously attempt to infer sensitive, private information about individual participants' datasets from model updates it observes (i.e., it is curious). 
Second, our system operates under the assumption of Byzantine failures. 
Note that this second adversary is strictly more powerful: it observes as much information as the server (since it can eavesdrop on its communications with honest participants), but also controls $f$ Byzantine participants that can deviate from the protocol arbitrarily.

\textbf{The \texttt{Safe-D(S)GD} algorithm.}
Owing to its strong empirical performance and baseline optimality guarantees, \texttt{Safe-D(S)GD}~\cite{pmlr-v202-allouah23a} constitutes a natural choice of algorithm for our study. We consider the variant without momentum (Algorithm~\ref{safe-dsgd}). 
In the deterministic (GD) setting, momentum is not required to achieve optimal error rates, making \texttt{Safe-DGD} theoretically optimal. 
While momentum is necessary in the stochastic regime to control the noise floor induced by Byzantine failures in the optimization error guarantee~\cite{karimireddy2021learning}, omitting it here simplifies our analysis of algorithmic stability, which is already technically involved, and helps isolate the effects of privacy and Byzantine robustness. 
This simplification is further justified by prior work showing that momentum can negatively impact stability in distributed optimization~\cite{attia2021algorithmic,JMLR:v25:22-0068stabilitymomentum}.

\begin{algorithm}[H]
\caption{\texttt{Safe-D(S)GD}}\label{safe-dsgd}
\begin{algorithmic}[1]
\STATE\textbf{Initialization:} $\theta_0$, 
\textcolor{olive}{aggregation $F$}, 
\textcolor{violet}{noise $\sigma_{\text{DP}}$}, number of steps $T$, learning rate $\gamma$
\FOR{$t = 0$ {\bf to} $T-1$}
    \STATE Server broadcasts $\theta_t$ to all participants
    \FORALL{\textbf{honest participant $i \in \mathcal{H}$}}
        \STATE Compute $g^{(i)}_t = \frac{1}{m} \sum_{j=1}^m \nabla \ell(\theta_t; z^{(i,j)})$ \\ or $g^{(i)}_t = \nabla \ell(\theta_t; z^{(i,J^{(i)}_t)})$ where $z^{(i,J^{(i)}_t)}$ from $\mathcal{D}_i$ is sampled uniformly at random
        \STATE\textcolor{violet}{Add noise $\widetilde{g}^{(i)}_t = g^{(i)}_t + \xi^{(i)}_t$; $\xi^{(i)}_t \sim \mathcal{N}(0,\sigma_{\text{DP}}^2 I_d)$}
        \STATE\textcolor{darkgray}{Send $\widetilde{g}^{(i)}_t$ to the server}
    \ENDFOR
    \STATE\textcolor{olive}{Server aggregates $R_t = F(\widetilde{g}^{(1)}_t,\ldots,\widetilde{g}^{(n)}_t)$} \hfill \textcolor{lightgray}{\# $f$ participants may send an arbitrary vector}
    \STATE Server updates model $\theta_{t+1} = \Pi_{\varTheta}\left( \theta_t - \gamma_t R_t \right)$
\ENDFOR
\STATE\! \textbf{return} $\theta_T$ or $\hat{\theta}$ sampled uniformly from $\{\theta_0,\ldots,\theta_T\}$
\end{algorithmic}
\end{algorithm}
The privacy analysis of \texttt{Safe-D(S)GD} follows the same, classic line of reasoning as in~\cite[Theorem 4.1]{pmlr-v202-allouah23a}, namely the composition of the subsampled Gaussian mechanism. 
We state the theorem for \texttt{Safe-DGD} for simplicity, though \texttt{Safe-DSGD} yields similar guarantees.
\begin{theorem}
    If $\sigma_{\operatorname{DP}} = \frac{2C\sqrt{2T\ln(1/\delta)}}{m\varepsilon}$, then \textnormal{\texttt{Safe-DGD}} is $(\varepsilon, \delta)$-$\operatorname{LDP}$.
\end{theorem}

\textbf{Stability under misbehaving participants.} 
Our study leverages algorithmic stability to understand the generalization error of robust distributed learning algorithms. 
We do so via uniform stability, which measures the sensitivity of a learning algorithm to changes in its training data. 
This approach was introduced in~\cite{vapnik74theory, DevroyeWagner79, RogersWagner78}, later popularized by~\cite{Bousquet2002StabilityAG, JMLR:v11:shalev-shwartz10a, hardt2016train}, and has recently been successfully applied to Byzantine-robust distributed learning~\cite{boudou2025generalization}.
We revisit this framework to study generalization in private and robust distributed learning, where changes occur only in the honest participants' data.

\begin{definition}\label{uniform-stability-def}
    Consider a distributed setting with up to $f$ Byzantine participants (out of $n$).
    A distributed algorithm $\mc{A}$ is $\varepsilon$-uniformly stable if, for all honest datasets $\mc{S}, \mc{S}' \in \mc{Z}^{(n-f)m}$ that are \textit{neighboring}---i.e., differing in at most one sample—we have,
    \[
        \textstyle 
        \sup_{z \in \mc{Z}} \mb{E}_{\mc{A}}[\ell(\mc{A}(\mc{S}); z) - \ell(\mc{A}(\mc{S}'); z)] \leq \varepsilon.
    \]
\end{definition}
This property leads to an upper bound on the generalization error, proved in~\Cref{gen-and-Byzantine}.

As established by DP-to-stability arguments~\cite[Lemma 23]{wang2016privacylearning}, a $(\varepsilon, \delta)$-DP algorithm already provides a bound on its algorithmic stability. 
For completeness, we provide a proof in~\Cref{app:dpimpliesstability}.
\begin{lemma}\label{lem:dp-to-stability}
    Any $(\varepsilon,\delta)$-DP algorithm is $\ell_\infty (e^\varepsilon - 1 + \delta)$-uniformly stable under an $\ell_\infty$-bounded~loss.
\end{lemma}
In other words, stronger privacy guarantees, i.e., smaller $\varepsilon$ and $\delta$ (corresponding to higher noise variance in \texttt{Safe-D(S)GD}), lead to improved generalization.
However, because this guarantee applies universally to any $(\varepsilon, \delta)$-DP algorithm, this baseline stability bound is expected to be loose. 
To address this, we provide a refined, algorithm-specific analysis in the next section.

\section{Non-monotonic privacy effect for Byzantine-robust generalization}\label{sec:trilemma}

We present our main results in the form of tight stability bounds for \texttt{Safe-D(S)GD}. 
We subsequently use these bounds to demonstrate a non-monotonic effect of LDP on the generalization error under Byzantine robustness, depicting a comprehensive understanding of the interplay between LDP and Byzantine robustness on generalization error.


\autodo{I think we could try to further improve this main part by introducing the MIA view earlier to emphasize it better, giving the intuition that the power of adversaries in the worst case is related to the ability of performing an MIA, and try to state an intermediate result formalizing this. But it probably requires some work - we could keep this for later.}

\paragraph{Upper bounds.}
We start by deriving stability upper bounds. 
The proof is deferred to~\Cref{app:trilemma-ub}.
\begin{theorem}\label{th:byz-ub-convex}
    Consider the setting described in~\Cref{II-problem-formulation}.
    Let $\sigma_{\operatorname{DP}} = \frac{2C\sqrt{2T\ln(1/\delta)}}{m\varepsilon}$ and $\gamma \leq \frac{1}{L}$.
    Suppose $\forall z \in \mc{Z}$, $\ell(\cdot;z)$ $C$-Lipschitz and $L$-smooth.
    Then $\mathcal{A}$ is $(\varepsilon, \delta)$-$\operatorname{LDP}$ and,\
    up to numerical constants explicitly provided in~\Cref{app:trilemma-ub},
    \begin{enumerate}[label=\roman*, topsep=-3pt, itemsep=2pt, parsep=0pt, leftmargin = 23pt]
        \item[\textit{(i)}] If $\ell(\cdot;z)$ is convex $\forall z \in \mc{Z}$, then the uniform stability of $\mc{A}$ is upper bounded by
        \begin{equation}\label{eq:byz-cvx-eq}  
            \textstyle 
            2 \gamma C^2 T \left( \frac{1}{(n-f)m} + \sqrt{\kappa } + \sqrt{\kappa} \frac{\sigma_{\operatorname{DP}}}{C} \sqrt{1 + \frac{d}{n-f}} \right).
        \end{equation}
        \item[\textit{(ii)}] If $\ell(\cdot;z)$ is $\mu$-strongly convex $\forall z \in \mc{Z}$, then the uniform stability of $\mc{A}$ is upper bounded by
        \begin{equation}\label{eq:byz-strgcvx-eq}  
            \textstyle 
            \frac{2 C^2}{\mu} \left( \frac{1}{(n-f)m} + \sqrt{\kappa} +  \sqrt{\kappa} \frac{\sigma_{\operatorname{DP}}}{C} \sqrt{1 + \frac{d}{n-f}} \right).
        \end{equation}
    \end{enumerate}
\end{theorem}
When there is no Byzantine participant ($f = 0$, hence $\kappa = 0$), both~\cref{eq:byz-cvx-eq,eq:byz-strgcvx-eq} reduce to the classical stability bounds established in~\cite{hardt2016train}.
When $\sigma_{\operatorname{DP}} = 0$, we recover the bound from~\cite{boudou2025generalization}.
However, when both $f, \sigma_{\operatorname{DP}} > 0$, the analysis reveals an additive degradation term arising from the interaction between privacy and robustness. 
In the convex setting, this degradation accumulates over the $T$ iterations, whereas in the strongly convex case, it appears as a one-time term.
The bound scales linearly with the maximal sensitivity of the gradient update ($2\gamma C$), but also, perhaps surprisingly, with the privacy noise magnitude. 
Thus, stronger privacy (more noise) degrades stability, in contrast with standard DP-to-stability arguments~\cite{wang2016privacylearning}. 
We discuss this apparent paradox in the tightness analysis below.
\looseness=-1

\tbtodo{Explain the interesting story of the upper bounds here to motivate the LB.}

\paragraph{Lower bounds.}
Byzantine failures expose a vulnerability of robust aggregation: their adaptivity, typically beneficial, makes stability sensitive to independent noise addition. 
While this may seem counter-intuitive, as DP is usually associated with improved generalization, the following lower bounds show that this effect is in fact unavoidable in some regimes. In particular, these bounds hold under a weak privacy regime, where $\frac{\varepsilon}{\sqrt{\ln(1/\delta)}} \in \Omega(1)$; we defer the discussion of the two key different privacy regimes to the next paragraph.
The proof is deferred to~\Cref{app:trilemma-lb}.
\begin{theorem}\label{th:trilemma-smea}
    Consider the setting in~\Cref{II-problem-formulation} with full batch and the $\mathrm{SMEA}$ aggregation. 
    Suppose there exist $\alpha_{\min}, \mu, \nu > 0$ such that $\frac{f}{n-f} \in [\alpha_{\min}, \frac{1}{2})$ and $\frac{d}{n-f} \in [\mu, \nu]$.
    Let $\mu_{\min}, N_{\min}, c_{\min}$ be problem-dependent constants defined in~\Cref{assump:lowerbound-d-high}, and assume $\mu \geq \mu_{\min}$, $n-f-1 \geq N_{\min}$, $T \in \mathcal{O}\big(e^{ c_{\min} \left( \sqrt{d} + \sqrt{n-f-1} \right)^2 }\big)$.
    Further suppose $\frac{\varepsilon}{\sqrt{\ln(1/\delta)}} \in \Omega(1)$ and $\sigma_{\operatorname{DP}} = \frac{2C\sqrt{2T\ln(1/\delta)}}{m\varepsilon}$.
    Then there exists $\ell \in {\mb{R}}^{\Theta \times \mc{Z}}$ such that $\forall z \in \mc{Z}, \ell(\cdot; z)$ is $C$-Lipschitz, $L$-smooth and convex,
    such that the uniform stability of $\mc{A}$ is lower bounded by
    \begin{equation}\label{smea-byz-lb-eq-d}
        \textstyle
        \Omega \left( \gamma C^2 T \left( \frac{1}{(n-f)m} + \frac{f}{n-2f} + \sqrt{\frac{f}{n-2f}} \frac{\sigma_{\operatorname{DP}}}{C} \left( 1 + \sqrt{\frac{d}{n-f}} \right) \right) \right).
    \end{equation}
    If we moreover assume $\frac{n}{3} \leq f < \frac{n}{2}$, the uniform stability of $\mc{A}$ is lower bounded by
    \begin{equation}\label{smea-byz-lb-eq-high-f-d}
        \textstyle
        \Omega\left( \gamma C^2 T \left( \frac{1}{(n-f)m} + \sqrt{\frac{f}{n-2f}} + \sqrt{\frac{f}{n-2f}} \frac{\sigma_{\operatorname{DP}}}{C} \left( 1 + \sqrt{\frac{d}{n-f}} \right) \right) \right).
    \end{equation}
\end{theorem}
The assumptions of~\Cref{th:trilemma-smea}, while technical, are fairly mild. 
They are either almost always satisfied in practice (e.g., the upper bound on $T$) or assumed to enable bounds on expectations and concentration of random matrices (e.g., $\mu_{\min}$ and $N_{\min}$). 
For details, see~\Cref{app:proof-trilemma-dhigh}.
The two regimes on $\frac{f}{n}$ from~\Cref{th:byz-ub-convex} are inherited from~\cite{boudou2025generalization}. 
In contrast, the term $\sqrt{\frac{f}{n-2f}}\frac{\sigma_{\operatorname{DP}}}{C} \big(1 + \sqrt{\frac{d}{n-f}} \big)$, which is common to these two regimes, reflects a different privacy regime, as detailed in the next paragraph. 
It arises as the product of privacy- and robustness-driven terms and vanishes when $f=0$.
Note that, for $\operatorname{SMEA}$, $\sqrt{\kappa} \in \mathcal{O}\Big(\sqrt{\frac{f}{n-f}}\big(1 + \frac{f}{n-2f}\big)\Big)$~\citep[Proposition 5.1]{pmlr-v202-allouah23a}.
Plugging this into our upper bound~\eqref{eq:byz-cvx-eq} establishes tightness with respect to~\eqref{eq:byz-cvx-eq} and~\eqref{smea-byz-lb-eq-high-f-d}, and shows that~\eqref{eq:byz-cvx-eq} matches~\eqref{smea-byz-lb-eq-d} under the mild condition $\frac{\sigma_{\operatorname{DP}}}{C} \big( 1 + \sqrt{\frac{d}{n-f}} \big) \in \Omega(1)$. 
This highlights an unavoidable tension between LDP and Byzantine robustness in the stability of the studied algorithm, which translates to the generalization error as proven in~\cite[Lemma 4.1]{boudou2025generalization}.

\paragraph{A surprising, yet principled non-monotonic effect on generalization.}
We sketch the proof of~\Cref{th:trilemma-smea} and explain how it extends to the strong privacy regime. 
We then show that stability under this regime eventually stops increasing and converges to $0$ as the privacy guarantee strengthens.
\looseness=-1

The Byzantine attack used in the construction of our lower bounds proceeds in two phases.
In the first phase, Byzantine participants send arbitrary vectors to the server (e.g., mimicking a participant or sending $0$'s) while collecting the communicated honest gradients over multiple rounds. 
After accumulating sufficient observations, they perform a membership inference attack (MIA): consistent with the threat model of DP, they know all but one sample (representing the differing sample in stability) and seek to determine whether it is present, triggering the second phase of the attack.
In this second phase, Byzantine participants send a malicious vector that is explicitly conditioned on the outcome of the MIA, with the goal of ensuring that it is consistently picked by the $\operatorname{SMEA}$ aggregation rule. 
In our construction, the success of the two phases are independent. 
That is, the success probability of the MIA, denoted $\mathfrak{P}$, is independent of the probability that Byzantine participants successfully fool the $\operatorname{SMEA}$ aggregation (i.e., the server). 
\looseness=-1

Consequently, as detailed in the proof of~\Cref{th:trilemma-smea}, we demonstrate that our construction yields a lower bound that is the product of two terms: the MIA success $\mathfrak{P}$, that depends on the privacy guarantee, and the instability of the second phase, independent of the privacy guarantee and which we denote by $\mathfrak{I}$. 
This is due to
\begin{equation}\label{eq:LB-proof-technique}
    \sup_{z\in \mathcal{Z}} \E\left| \ell(\A(\mathcal{S}), z) - \ell(\A(\mathcal{S}^\prime), z) \right| 
    \geq \sup_{z\in \mathcal{Z}}\E[ \left| \ell(\A(\mathcal{S}), z) - \ell(\A(\mathcal{S}^\prime), z) \right| \mid E ]\ \Prob(E) \\ 
    = \mathfrak{I}\ \mathfrak{P},
\end{equation}
where the event $E \vcentcolon= \{ \text{MIA success} \}$, with $\Prob(E) = \mathfrak{P}$.
Note that $\mathfrak{I}$ still crucially depends on the privacy noise magnitude since, conditioned on a successful MIA, Byzantine participants can leverage the spectral dispersion of the honest noisy gradients to bias the update in a data-dependent direction. 
This, however, is capped by the spectral filter ($\operatorname{SMEA}$) at the level of the honest spectral dispersion, i.e., $\E \| \Sigma_{\mathcal{H}} \|_{\operatorname{sp}} \in \Theta \big( \sigma_{\operatorname{DP}} \big(1+\sqrt{\frac{d}{n-f}}\big) \big)$.

Within the stability analysis of $\texttt{Safe-DGD}$ and our construction, the MIA success probability satisfies,
\begin{equation*}
    \mathfrak{P} 
    = \Prob\left( z_\alpha - \Delta < Z \leq z_\alpha \right) 
    = \Phi(z_\alpha) - \Phi(z_\alpha - \Delta)
    \geq \phi(z_\alpha) \Delta,
\end{equation*} 
where $Z \sim \mathcal{N}(0, 1)$ has cumulative distribution function $\Phi$ and density $\phi$. 
Here, $\alpha \in (0,1)$ is a fixed false positive rate, with $z_\alpha$ defined by $\Prob(Z > z_\alpha) = \alpha$, and $\Delta$ denotes the worst case sensitivity of local gradients
normalized by the standard deviation of the Gaussian noise aggregated over the first step.
Importantly, $\Delta$ is directly tied to the $(\varepsilon,\delta)$-DP guarantee, scaling as $\frac{\varepsilon}{\sqrt{\ln(1/\delta)}}$.

Let $(\varepsilon_{\operatorname{weak}}, \delta_{\operatorname{weak}})$ denote the privacy parameters associated with the weak privacy regime, defined by $\frac{\varepsilon_{\operatorname{weak}}}{\sqrt{\ln(1/\delta_{\operatorname{weak}})}} \in \Omega(1)$, and $\mathfrak{P}_{\operatorname{weak}}$, $\mathfrak{I}_{\operatorname{weak}}, \operatorname{LB}_{\operatorname{weak}}$ the corresponding MIA success probability, instability, and stability lower bound induced by our proof technique (cf.~\eqref{smea-byz-lb-eq-d} and~\eqref{smea-byz-lb-eq-high-f-d}).
The analogous quantities in the strong privacy regime
are denoted with the subscript $\operatorname{strong}$.
By~\eqref{eq:LB-proof-technique}, our construction yields
\looseness=-1
\begin{equation*}
    \operatorname{LB}_{\operatorname{weak}} 
    = \mathfrak{I}_{\operatorname{weak}} \mathfrak{P}_{\operatorname{weak}} 
    \in \Omega\Big( \mathfrak{I_{\operatorname{weak}}} \frac{\varepsilon_{\operatorname{weak}}}{\sqrt{\ln(1/\delta_{\operatorname{weak}})}} \Big)
    \in \Omega\Big( \mathfrak{I_{\operatorname{weak}}} \Big),
\end{equation*}
where we used $\mathfrak{P}_{\operatorname{weak}} \in \Omega \big( \frac{\varepsilon_{\operatorname{weak}}}{\sqrt{\ln(1/\delta_{\operatorname{weak}})}} \big) \in \Omega (1)$. 
This explains why the stability---and consequently the generalization error~\citep[Lemma 4.1]{boudou2025generalization}---degrades with the noise magnitude in the weak privacy regime.
Moreover, we can prove with a Taylor expansion that for small $\Delta$ (i.e., strong privacy regime), we have $\mathfrak{P} \in \Theta(\Delta)$. 
Hence,
\begin{equation*}
    \operatorname{LB}_{\operatorname{strong}} 
    = \mathfrak{I}_{\operatorname{strong}} \mathfrak{P}_{\operatorname{strong}}
    \in \Theta\Big( \mathfrak{I}_{\operatorname{strong}} \frac{\varepsilon_{\operatorname{strong}}}{\sqrt{\ln(1/\delta_{\operatorname{strong}})}} \Big).
\end{equation*}
Eventually, $\operatorname{LB}_{\operatorname{strong}}$ vanishes as the privacy guarantee strengthens, i.e., as $(\epsilon, \delta) \to (0,0)$.
While deriving an upper bound that matches $\operatorname{LB}_{\operatorname{strong}}$ represents interesting future work, 
this is only a minor gap to be closed.
Indeed, the fact that stability ultimately converges to $0$ is well established in the DP literature: our theory can thus be completed via standard DP-to-stability arguments. 
Specifically, for a $\ell_\infty$-bounded loss function,\footnote{This can be ensured by projecting the parameter onto a compact set $\varTheta$ after the server update (Line 10 in~\Cref{safe-dsgd}).} any $(\varepsilon, \delta)$-DP learning algorithm is guaranteed to be $\ell_\infty(e^\varepsilon - 1 + \delta)$-uniformly stable (Lemma~\ref{lem:dp-to-stability}). 
Note that, although implicit, $\delta \leq \varepsilon$ in all practical scenarios.
This shows that, assuming $\varepsilon_{\operatorname{strong}} \leq 1$,
\begin{equation}
    \label{eq:ub-wang}
    \operatorname{LB}_{\operatorname{strong}} 
    \leq \sup_{z\in \mathcal{Z}} \E\left| \ell(\A(\mathcal{S}), z) - \ell(\A(\mathcal{S}^\prime), z) \right| 
    \leq \ell_\infty (e^{\varepsilon_{\operatorname{strong}}} - 1 + \delta_{\operatorname{strong}}) 
    \leq (e+1) \ell_\infty \varepsilon_{\operatorname{strong}}.
\end{equation}
Consequently, stability can be controlled as the minimum between the upper bounds of~\Cref{th:byz-ub-convex} and the one in~\Cref{eq:ub-wang}, which is itself in $\tilde{\mathcal{O}}(\frac{\sqrt{T}}{m\sigma_{\operatorname{DP}}})$ when we get rid of constants and log factors. 
Hence, for small $\sigma_{\operatorname{DP}}$ (large $\varepsilon$, weak privacy regime), the bound of~\Cref{th:byz-ub-convex} is smaller, tight (\Cref{th:trilemma-smea}), and it linearly grows with $\sigma_{\operatorname{DP}}$. 
At some point, $\sigma_{\operatorname{DP}}$ becomes large enough (small $\varepsilon$, strong privacy regime) and the second upper bound becomes the smallest one, and it decreases with $\sigma_{\operatorname{DP}}$. This transition highlights a surprising non-monotonic dependence of generalization on the level of privacy noise, as summarized by the following corollary.

\begin{corollary}
    Assume the setting of~\Cref{th:trilemma-smea} and~\Cref{lem:dp-to-stability}. 
    Let $\sigma_{\operatorname{DP}} = \frac{2C\sqrt{2T\ln(1/\delta)}}{m\varepsilon}$.
    \begin{minipage}[c]{0.66\linewidth}
    \vspace{0pt}
        In the weak privacy regime $\frac{\varepsilon}{\sqrt{\ln(1/\delta)}} \in \Omega(1)$, or equivalently $\sigma_{\operatorname{DP}} \in \mathcal{O}(\sigma_{\operatorname{base}})$ with $\sigma_{\operatorname{base}} \vcentcolon= \frac{C\sqrt{T}}{m}$, the stabilty scales as
        \begin{equation}\label{eq:weak-privacy}
            \textstyle\textcolor{\weakprivacy}{
            \sigma_{\operatorname{DP}}
            \times 
            \Theta\Big( \gamma C T \sqrt{\frac{f}{n-2f}} \Big( 1 + \sqrt{\frac{d}{n-f}} \Big) \Big)},
        \end{equation}
        and, in the strong privacy regime, it scales as 
        \begin{equation}\label{eq:strong-privacy}
            \textcolor{\strongprivacy}{
            \textstyle
            \frac{1}{\sigma_{\operatorname{DP}}}
            \times 
            \mathcal{O}\Big( \frac{\ell_\infty C \sqrt{T \ln(1/\delta)}}{m} \Big)},
            \qquad\qquad\quad
        \end{equation}
    \end{minipage}
    \hfill
    \begin{minipage}[c]{0.33\linewidth}
        \vspace{5pt}
        \centering
        \begin{tikzpicture}[
                    scale=0.8,
                    yscale=2.0,
                    xscale=1.025,
                    every node/.style={font=\small}
                ]
            Arrow pointing to the weak privacy tight curve
            \draw[gray!80, thin, <-, >=stealth] 
                (1.0, {0.35*1.0 + 0.5}) -- ++(-0.15, 0.15) 
                node[above=-1pt, text=\weakprivacy] {\small tight};

            \draw[gray!80, thin, <-, >=stealth] 
                (4.2, {0.01 + 5.514 / pow(4.2, 2.2)}) -- ++(0.25, 0.12) 
                node[above right=-4pt, text=\strongprivacy] {\small tight};


            \node at (4.7,0) {$\sigma_{\operatorname{DP}}$};
            \draw[-, thin] (0,0) -- (4.3,0);
            \draw[->, thin] (5.08,0) -- (5.3,0);

            \node[rotate=90] at (0,1.22) {gen. error};
            \draw[-, thin] (0,0) -- (0,0.8);
            \draw[->, thin] (0,1.63) -- (0,1.75);

            \fill[gray!40] (1.95,0.005) rectangle (3.0, 1.75);
            \draw[
                line width=0.5pt,
                gray!80
            ] (1.95,-0.025) -- (1.95,0.025);
            \node[gray!80] at (2.0,-0.1) {{\small $\sigma_{\text{base}}$}};
            \draw[
                line width=0.5pt,
                gray!80
            ] (3.0,-0.025) -- (3.0,0.025);
            \node[gray!80] at (3.3,-0.1) {{\small $\sigma_{\text{threshold}}$}};
            \node[right, gray!80, rotate=90] at (2.45,0)
                {\small tightness};
            \node[right, gray!80, rotate=90] at (2.45,0.8)
                {\small unknown};

            \draw[
                very thick,
                line width=2pt,
                draw=\weakprivacy, 
                line cap=round,
                domain=0.1:2.0,
                samples=2
            ]
            plot (\x, {0.35*\x + 0.5});
            \draw[
                very thick,
                line width=2pt,
                draw=\weakprivacy,
                dashed,
                line cap=round,
                domain=2.1:3.49,
                samples=2 
            ]
            plot (\x, {0.35*\x + 0.5});

            \draw[
                very thick,
                line width=2pt,
                draw=\strongprivacy, 
                line cap=round,
                domain=2.0:5.25, 
                samples=40,
                smooth
            ]
            plot (\x, {5.514 / pow(\x, 2.2)}); 
            \draw[
                very thick,
                line width=2pt,
                draw=\strongprivacy,
                dashed,
                line cap=round,
                domain=1.695:1.95,
                samples=40,
                smooth
            ]
            plot (\x, {5.514 / pow(\x, 2.2)}); 

            \begin{scope}[shift={(3.35,1.55)}]
                \draw[
                    very thick,
                    line width=2pt,
                    draw=\weakprivacy,
                    line cap=round,
                    dash pattern=on 1.5pt off 3.0pt
                ] (-0.25,-0.25) -- (0.0,-0.25);
                \node[right] at (0.1,-0.25)
                    {\small $\eqref{eq:weak-privacy}$};
                \draw[
                    very thick,
                    line width=2pt,
                    draw=\strongprivacy,
                    line cap=round,
                    dash pattern=on 1.5pt off 3.0pt
                ] (-0.25,-0.5) -- (0.0,-0.5);
                \node[right] at (0.1,-0.5)
                    {\small $\eqref{eq:strong-privacy}$};
                \coordinate (A) at (-0.25,-0.75);
                \coordinate (B) at (0.0,-0.75);
                \coordinate (M) at ($(A)!0.57!(B)$);
                \draw[very thick, line width=2pt, draw=\weakprivacy, line cap=round] (A) -- (M);
                \draw[very thick, line width=2pt, draw=\strongprivacy, line cap=round] (M) -- (B);
                \node[right] at (0.0,-0.75)
                    {\small $\min\{\!\eqref{eq:weak-privacy},\!\eqref{eq:strong-privacy}\!\}$};
            \end{scope}
        \end{tikzpicture}
        \captionsetup{justification=centering, skip=2pt}
        \captionof{figure}{Generalization error \\as a function of $\sdp$.}\label{fig:nonmonotonic}
    \end{minipage}
    which is asymptotically tight in 
    $\sigma_{\operatorname{DP}}$ due to~\eqref{eq:ub-wang}. 
    In other words, there exists $\varepsilon_{\operatorname{threshold}} \in (0,1)$ such that, for $\varepsilon_{\operatorname{strong}} \leq \varepsilon_{\operatorname{threshold}}$, we have $\operatorname{LB}_{\operatorname{strong}} \in \Theta_{\varepsilon_{\operatorname{strong}}}(\varepsilon_{\operatorname{strong}})$ when we consider parameters other than $\varepsilon_{\operatorname{strong}}$ fixed.
    Generalization exhibits a similar behavior, following from~\cite[Lemma 4.1]{boudou2025generalization}.
    \looseness=-1
\end{corollary}

Overall, our results offer a comprehensive perspective on the interplay between LDP and Byzantine robustness on generalization error, as differential privacy can be characterized by the performance of MIAs~\cite{dong2022gaussian} and stability provides a tight control on the generalization error \cite[Lemma 4.1]{boudou2025generalization}.²

We next corroborate our theoretical results empirically, showing that the predicted non-monotonic behavior is also observed in practice.

\section{Empirical evaluations}\label{sec:numerical}

In this section, we empirically validate the non-monotonic behavior predicted by our theory. 
All simulations 
are implemented via the ByzFL library~\cite{gonzalez2025byzfl}.
More details and evaluations are presented in~\Cref{app:numerical}.

\paragraph{Empirical evaluation settings.}

For each of the three datasets considered, we consider a common evaluation protocol. 
While a single training run is inexpensive, estimating expected generalization requires repetitions over datasets, seeds, and noise multipliers, which is the main bottleneck. 
This motivates the use of a fixed setting with $n=7$ participants\footnote{The computational cost of SMEA scales poorly with $n$, hence the small number of participants.} and $m=200$ local samples per honest participant.  
The model test loss is evaluated on a held-out test set of $N_{\text{test}}=10000$ samples.

We consider three classification tasks: 
(i) synthetic data following~\cite{li2020federated_synthetic_data,pmlr-v151-noble22a} with logistic regression; 
(ii) MNIST~\cite{lecun1998mnist} with logistic regression; 
and (iii) CIFAR-10~\cite{krizhevsky2009learningCIFAR} with a convolutional network. 

All experiments use deterministic initialization and \texttt{Safe-DGD} with learning rate $\gamma$, $T$ iterations, $\operatorname{SMEA}$ aggregation, and Gaussian noise $\sdp = \sigma \times \sigma_{\text{base}}$, calibrated via a multiplier $\sigma$, which serves as a proxy for the privacy level and scales the base noise $\sigma_{\text{base}} \vcentcolon= \frac{2C\sqrt{2T}}{m}$. 
We evaluate $\sigma \in \{0.125, 0.25, 0.5, 1, 2, 4, 8, 16, 32\}$, reporting both individual realizations and Monte Carlo estimates over datasets and seeds.
Byzantine participants ($f=1$ unless otherwise stated) use a state-of-the-art $\operatorname{Opt\text{-}IPM}$ attack~\cite{xie2020FOE}, relying on line search at each iteration to compute the most damaging perturbation bypassing $\operatorname{SMEA}$.

For (i)--(ii), the objective is convex and smooth. 
To be consistent with our theory, we enforce Lipschitzness and smoothness by $L_2$-normalizing features ($\|x\|_2 \leq 2$)~\cite[Proposition F.1]{NEURIPS2022_04b42392}, and bounded loss by projecting parameters onto a ball of radius $100$ after updates, enabling calibrated privacy noise.
For CIFAR-10, gradients are clipped at $3.0$ and parameters layerwise projected onto a ball of radius $10$.
\looseness=-1

Below, we present the results, along with additional experiment-specific details.

\begin{figure*}[t]
    \centering
    \includegraphics[width=0.99\linewidth]{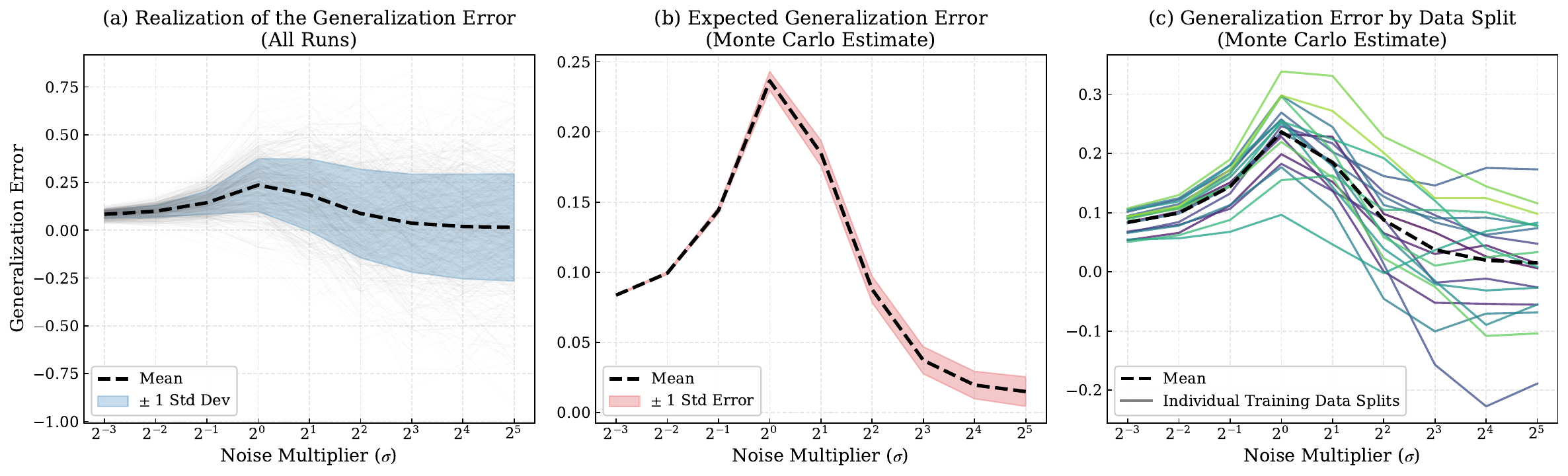}
    \caption{
        \textbf{Synthetic dataset.}
        Impact of LDP on generalization error under Byzantine failures.
        The Byzantine participant ($f=1$) executes the $\operatorname{Opt-IPM}$ attack, the server uses $\operatorname{SMEA}$ as defense, with $\gamma=0.4$, $T=600$, and $d=50$.
        \textbf{(a)} Realizations of the generalization error (as a random variable) demonstrating high variance across all independent runs (variance estimated on all runs). 
        \textbf{(b)} The Monte Carlo estimate of the expected generalization error, where reported variance represents the mean estimation standard error.
        \textbf{(c)} Expected generalization error stratified by individual training dataset, isolating the variance attributable to data samples versus differential privacy noise.
        }\label{fig:mc_gen_error}
\end{figure*}

\begin{figure*}[t]
    \centering
    \includegraphics[width=0.329\linewidth]{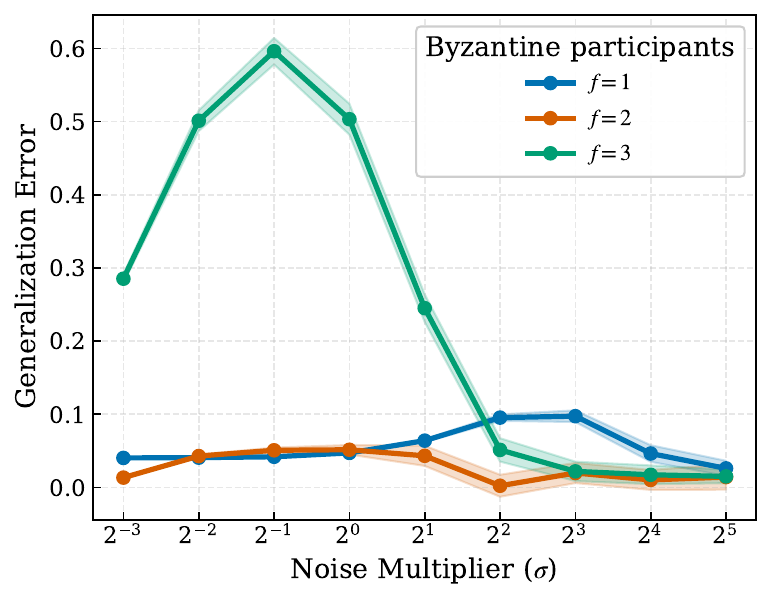}
    \includegraphics[width=0.329\linewidth]{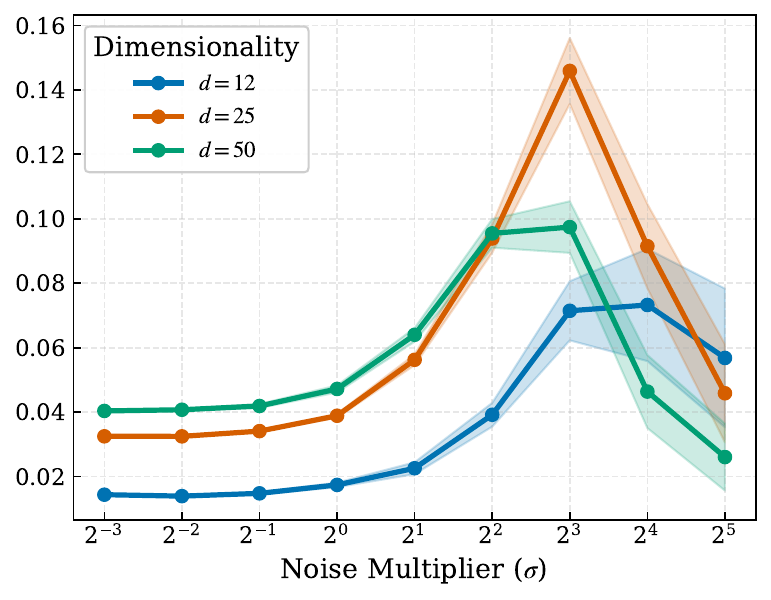}
    \includegraphics[width=0.329\linewidth]{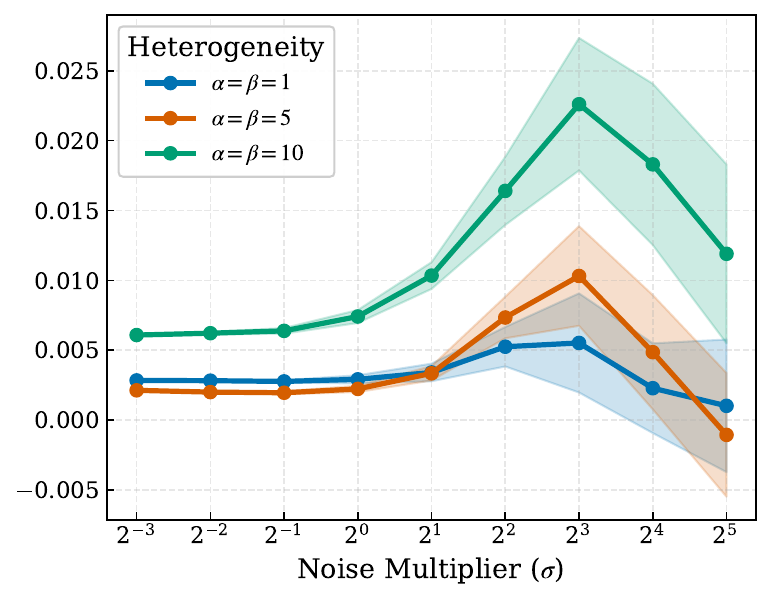}
    \caption{
        \textbf{Synthetic dataset.}
        Absolute expected generalization error under variations of parameters. 
        \textbf{(a)} Left: varying $f \in \{1, 2, 3\}$ for fixed $d=50$. 
        \textbf{(b)} Middle: varying $d \in \{12, 25, 50\}$ for fixed $f=1$. 
        \textbf{(c)} Right: varying heterogeneity for fixed $f=1$, where $\alpha, \beta$ are defined in~\cite{li2020federated_synthetic_data} (the larger $\alpha$ and $\beta$, the more heterogenous).
        All evaluations share identical baseline parameters $T=200$ and $\gamma=0.2$.
        \looseness=-1
    }\label{fig:generalization_error_comparison}
\end{figure*}

\begin{figure*}[t]
    \centering
    \includegraphics[width=0.99\linewidth]{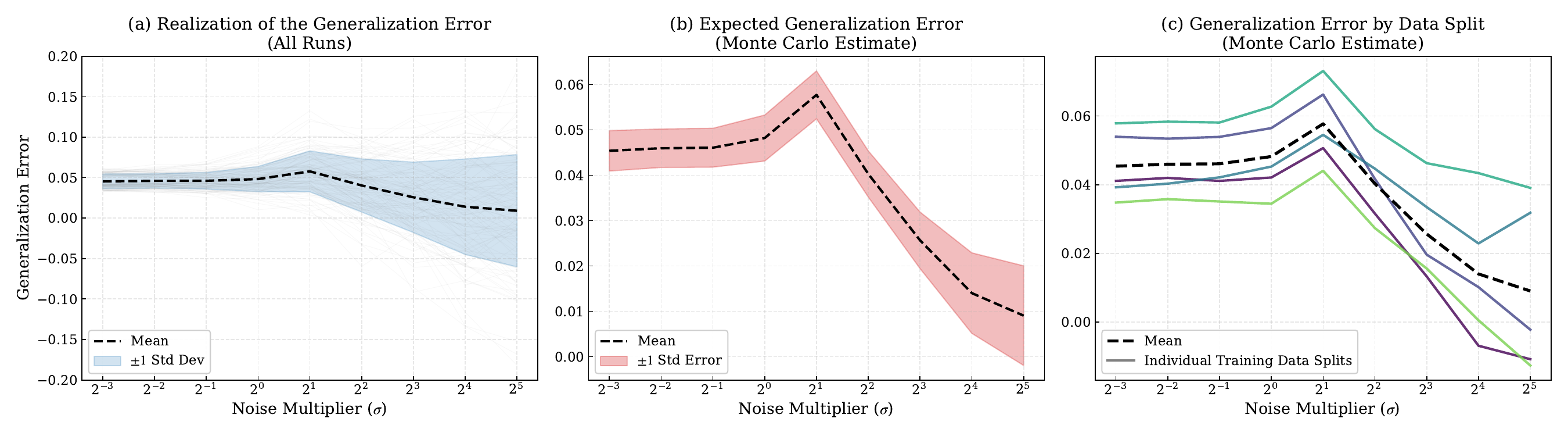}
    \caption{
        \textbf{MNIST}. Analogous to~\Cref{fig:mc_gen_error} with $\gamma = 0.4$, $T=400$, $f=1$.
        }\label{fig:mc_gen_error_mnist-lr}
\end{figure*}

\begin{figure*}[t]
    \centering
    \includegraphics[width=0.99\linewidth]{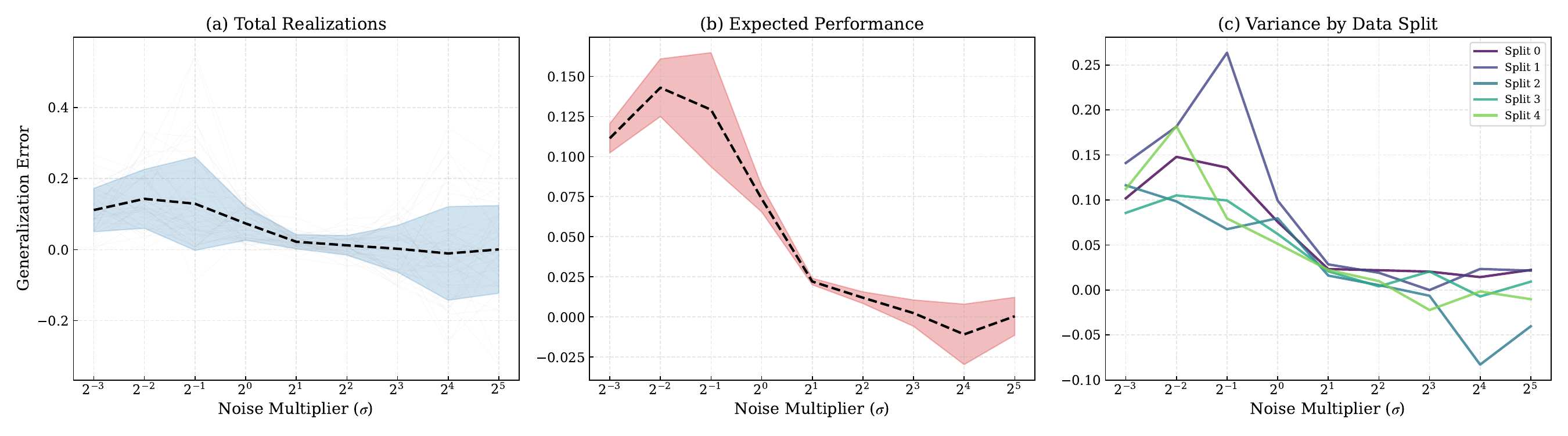}
    \caption{
        \textbf{CIFAR-10}. Analogous to~\Cref{fig:mc_gen_error} with $\gamma = 0.1$, $T=400$, $f=1$.
        The convolutional network architecture consists of three convolutional layers (16, 32, 64 filters), ReLU activations, max-pooling, and a 1024-dim feature map followed by fully connected layers (128, 64, 10). 
        }\label{fig:mc_gen_error_cifar_cnn}
\end{figure*}

\paragraph{Results and discussion.}
Our empirical results in~\Cref{fig:mc_gen_error,fig:generalization_error_comparison,fig:mc_gen_error_mnist-lr,fig:mc_gen_error_cifar_cnn} show consistent and interpretable trends across settings. 
For logistic regression on synthetic data (\Cref{fig:mc_gen_error,fig:generalization_error_comparison}) and MNIST (\Cref{fig:mc_gen_error_mnist-lr}), which satisfy the assumptions of our theory, the observed dynamics closely match our predictions. 
On CIFAR-10 (\Cref{fig:mc_gen_error_cifar_cnn}), where we train a CNN and thus operate outside the convex setting covered by our analysis, the behavior remains qualitatively consistent and exhibits meaningful structure despite more challenging interpretation, suggesting promising directions for further investigation.\looseness=-1

Across~\Cref{fig:mc_gen_error,fig:mc_gen_error_mnist-lr}, the dependence on the noise multiplier aligns with our bounds: it grows linearly in the weak privacy regime.
\Cref{fig:generalization_error_comparison} further highlights other dynamics: in the weak privacy regime, higher dimension consistently leads to worse generalization, in line with the predicted ordering, although the dependence departs from the square-root behavior suggested by the bound.
In the strong privacy regime, the behavior becomes less regular, suggesting the presence of multiple interacting effects.
The dependence on the Byzantine fraction exhibits richer structure.
The results are consistent with the change in regime in our bounds, which distinguish between $\frac{f}{n}\leq\frac{1}{3}$~\eqref{smea-byz-lb-eq-d} ($f=1,2$) and $\frac{1}{3}\leq\frac{f}{n}<\frac{1}{2}$~\eqref{smea-byz-lb-eq-high-f-d} ($f=3$).
For each fixed $f$, the behavior remains non-monotonic, as predicted, while the regime transition depends on $f$.
Varying $\frac{f}{n-f}$, the dynamic is consistent with our bounds (inherited from~\cite{boudou2025generalization}), including a pronounced increase in attack severity at $f=3$.
Finally, the larger the heterogeneity, the less the model generalize.
\looseness=-1 

Overall, these results support the qualitative trends of our analysis while revealing additional structure not captured by the bounds. 
As predicted, we observe an initial increase in generalization error under small noise (weak privacy), followed by a regime switch where the error decreases as noise increases (strong privacy).
We leave several directions future work, including alternative optimizers, varying number of participants, different aggregation rules, and relaxing assumptions beyond convexity. \looseness=-1

\section{Conclusion}\label{sec:conclusion}
Prior optimization-focused theory overlooks the joint impact of LDP and Byzantine robustness on generalization, leaving an incomplete picture. 
We close this gap and uncover a surprising non-monotonic, two-phase behavior. 
Through tight stability bounds, we show that privacy noise actively degrades the generalization error in weak privacy regime, yet drives it to zero as privacy strengthens. 
Our theory offers several predictive insights and provides a qualitative and quantitative characterization of the transition between the two regimes.
We moreover corroborate these findings in empirical evaluations, validating the usefulness of our bounds as predictive of the generalization error dynamics accross multiple critical parameters.
While these findings offer a theoretical counterpoint to the optimization trilemma, practical privacy budgets typically lie in the weak privacy regime. 
Thus, achieving privacy, robustness, and generalization simultaneously in practical distributed learning is fundamentally constrained.
Moving forward, this necessitates the exploration of ways to systematically relax the threat models, whether by integrating cryptographic primitives to bound misbehaviors~\cite{farhadkhani2022equivalence,boudou2025generalization}, 
or by relaxing LDP as in~\cite{allouah2025towards}, where a shared randomness seed among honest clients is assumed, enabling correlated noise mechanisms that yield improved optimization guarantees.


\begin{ack}
The work of Thomas Boudou and Aurélien Bellet is supported by grant ANR 22-PECY-0002 IPOP (Interdisciplinary Project on Privacy) project of the Cybersecurity PEPR.\
\end{ack}

\clearpage
{\small
\bibliography{references.bib}
\bibliographystyle{abbrv}}

\clearpage
\appendix

\section{Stability or DP implies generalization in robust distributed learning}\label{gen-and-Byzantine}
In this section, we focus on the formal link between uniform stability and generalization under Byzantine failures. 
The proof follows classic arguments~\cite{Bousquet2002StabilityAG}.
\begin{proposition}\label{lemmagen1}
    Consider the setting described in~\Cref{II-problem-formulation}. 
    let $\mc{A}$ an $\varepsilon$-\textit{uniformly stable} (randomized) distributed algorithm. Then, 
    \[
        | \mb{E}_{\mc{S}, \mc{A}}[ R_{\mc{H}}(\mc{A}(\mc{S})) - \widehat{R}_{\mc{H}}(\mc{A}(\mc{S})) ] | \leq \varepsilon.  
    \]
\end{proposition}
\begin{proof}
    Recall that $\mc{S}$ denotes the collective dataset of honest participants. 
    Let \( \mc{S}' = \{ z'^{(i,j)}, i \in \mc{H}, j \in \{1, \ldots, m \} \} \) be another independently sampled dataset for honest participants where for all $i \in \mc{H}$, $\mc{D}'_i = \{z'^{(i,1)}, \ldots, z'^{(i,m)} \}$ is composed of $m$ i.i.d.\ data points (or samples) drawn from distribution $p_i$. Let $\mc{S}^{(i,j)}$ be a neighboring dataset of $\mc{S}$ with only a single differing data sample $z'^{(i,j)}$ at index $(i,j) \in \mc{H} \times \{1, \ldots, m\}$. Then, 
    \begin{align*}
        \mb{E}_{\mc{S}, \mc{A}}[\widehat{R}_{\mc{H}}(\mc{A}(\mc{S}))]
        & = \mb{E}_{\mc{S}, \mc{S}', \mc{A}} \left[ 
            \frac{1}{|\mc{H}|} \sum_{i\in\mc{H}} \frac{1}{m} \sum_{j = 1}^m \ell(\mc{A}(\mc{S}^{(i,j)}), z'^{(i,j)}) 
        \right] \\
        & = \mb{E}_{\mc{S}, \mc{S}', \mc{A}} \left[ 
            \frac{1}{|\mc{H}|} \sum_{i\in\mc{H}} \frac{1}{m} \sum_{j = 1}^m \ell(\mc{A}(\mc{S}), z'^{(i,j)}) 
        \right] + \delta \\
        & = \mb{E}_{\mc{S}, \mc{A}}[ R_{\mc{H}}(\mc{A}(\mc{S})) ] + \delta
    \end{align*}
    with
    \begin{multline*}
    \delta 
    = \mb{E}_{\mc{S}, \mc{S}'} \left[ \frac{1}{|\mc{H}|} \sum_{i\in\mc{H}} \frac{1}{m} \sum_{j = 1}^m \mb{E}_{\mc{A}} \left[ \ell(\mc{A}(\mc{S}^{(i,j)}), z'^{(i,j)}) - \ell(\mc{A}(\mc{S}), z'^{(i,j)}) \right] \right] \\ 
    \leq \mb{E}_{\mc{S}}\mb{E}_{\mc{S}'} \left[ \frac{1}{|\mc{H}|} \sum_{i\in\mc{H}} \frac{1}{m} \sum_{j = 1}^m \sup_{z \in \mathcal{Z}}\mb{E}_{\mc{A}} \left[ \ell(\mc{A}(\mc{S}^{(i,j)}), z) - \ell(\mc{A}(\mc{S}), z) \right] \right] 
    \leq \varepsilon 
    \end{multline*}
    Substituting from the above equation proves the lemma.
\end{proof}

\subsection{Differential privacy implies stability}\label{app:dpimpliesstability}

We demonstrate that a \textit{strong} differential privacy constraint is sufficient to establish non-oblivious stability result, even though implementing such \textit{strong} differential privacy may be impractical. 
As a result, each of our bound should be interpreted as $\min\{\ell_\infty, (1), (2)\}$, where $(1)$ represents the bound on uniform algorithmic stability implied by the privacy constraint, and $(2)$ corresponds to the bound on uniform algorithmic stability demonstrated through the sensitivity analysis of a particular iterative optimization algorithm. 
The following result is proven in \cite{wang2016privacylearning}.
\begin{replemma}{lem:dp-to-stability}
    Any $(\varepsilon,\delta)$-DP algorithm is $\ell_\infty (e^\varepsilon - 1 + \delta)$-uniformly stable under an $\ell_\infty$-bounded~loss.
\end{replemma}
\begin{proof}
    For any $S$ and $S'$ two neighboring datasets and any $x\in\mc{X}$. Let the probability density (with respect to the Lebesgue measure) defined by $\mc{A}(S)$ and $\mc{A}(S')$ be $p$ and $p'$ respectively. 
    We define the event $E = \{ \theta \in \Theta; p(\theta) \geq p'(\theta) \}$. Then we can bound the uniform algorithmic stability as follows
    \begin{align*}
        \left| \mb{E}_{\theta \sim \mc{A}(S)} [\ell (\theta, x)] - \mb{E}_{\theta \sim \mc{A}(S')} [\ell (\theta, x)] \right| 
        & = \left| \int_{\theta\in\Theta} \ell (\theta,x)p(\theta)d\theta - \int_{\theta\in\Theta} \ell (\theta,x)p'(\theta)d\theta \right| \\
        & = \left| \int_{\theta\in\Theta} \ell (\theta,x) \left( p(\theta) - p'(\theta) \right) d\theta \right| \\
        & \leq \max\{ \left( \int_{\theta\in\Theta} \ell (\theta,x) \left( p(\theta) - p'(\theta) \right) d\theta \right)_+, \\ 
        & \quad\quad\quad\quad \left( \int_{\theta\in\Theta} \ell (\theta,x) \left( p(\theta) - p'(\theta) \right) d\theta \right)_- \} 
    \end{align*}
    
    The definition of DP is symmetric, so that we have the same bound irrespective of the fact that cumulative positive quantity is greater than cumulative negative quantity or the opposite. 
    We show it here for the positive part.
    \begin{align*}
        \left| \mb{E}_{\theta \sim \mc{A}(S)} [\ell (\theta, x)] - \mb{E}_{\theta \sim \mc{A}(S')} [\ell (\theta, x)] \right|
        & \leq  \left( \int_{\theta\in\Theta} \ell (\theta,x) p(\theta) - p'(\theta) d\theta \right)_+ \\
        & \underset{\ell\geq0}{=} \int_{\theta\in\Theta} \ell (\theta,x) \left( p(\theta) - p'(\theta) \right)_+ d\theta \\
        & \leq \sup_{x\in\mc{X}, \theta \in E}\ell(\theta,x) \int_E \left( p(\theta) - p'(\theta) \right) d\theta \\
        & \leq \ell_\infty \left( \mb{P}(\mc{A}(S) \in E) - \mb{P}(\mc{A}(S') \in E) \right) \\
        &\leq \ell_\infty \left( (e^\varepsilon - 1) \mb{P}(\mc{A}(S') \in E) + \delta \right) \\
        & \underset{\varepsilon > 0}{\leq} \ell_\infty \left( e^\varepsilon - 1 + \delta \right)
    \end{align*}
\end{proof}

\begin{remark}
    For the bound to be meaningful, we need $e^\varepsilon - 1 + \delta \leq 1$, that is $\varepsilon \leq \ln{\left(1+(1-\delta)\right)  }$. This is generally a difficult condition to hold.
\end{remark}

\section{Regularity assumptions and expansivity results}\label{loss-regularities}

To analyze the stability of iterative optimization algorithms, we typically rely on regularity assumptions on the loss function.
We begin by stating standard regularity assumptions~\cite{hardt2016train} used throughout the paper, along with classical results that follow from them.
\begin{definition}[Regularity assumptions]
    Let $\ell: \Theta \to \mb{R}$ differentiable, $\Theta \subset \mb{R}^d$. 
    \begin{itemize}       
        \item[\textit{(i)}] $\ell$ is $C-$Lipschitz continuous (or $C-$Lipschitz) if there exists $C>0$ such that,
        \[
            \forall u, v \in \Theta,\quad |\ell(u)-\ell(v)| \leq C \|u-v\|_2.
        \] 
        This property is equivalent to the norm of the gradient of $\ell$ being uniformly bounded by $C$.
        
        \item[\textit{(ii)}] $\ell$ is $L-$Lipschitz smooth (or $L-$smooth) if its gradient is $L$-Lipschitz. 
        This is equivalent to the following smoothness inequality: 
        \begin{equation}\label{smoothness-ineq2}
            \forall u, v \in \Theta,\quad \ell(u) \leq \ell(v) + \langle \nabla \ell(v), u-v \rangle + \frac{L}{2} \| u-v \|_2^2.
        \end{equation}

        \item[\textit{(iii)}] $\ell$ is convex if and only if, $\forall u, v \in \Theta$,
        $
            \langle \nabla \ell (u) - \nabla \ell (v), u-v\rangle \geq 0.
        $
        
        Moreover it is $\mu-$strongly convex if there exists $\mu > 0$ such that,
        \begin{equation*}
            \forall u, v \in \Theta,\quad \ell(u) \geq \ell(v) + \langle \nabla \ell(v), u-v \rangle + \frac{\mu}{2} \| u-v \|_2^2.
        \end{equation*}
        We note that for a strongly convex function to have bounded gradients, it must be defined on a convex compact set, which then implies boundedness of both the loss and the gradient. 
        To ensure this, we can either penalize the problem or restrict the parameter domain and apply an Euclidean projection at every step. Throughout the paper, we will tacitly assume this when referring to strongly convex functions, 
        which does not limit the applicability of our analysis since the projection does not increase the distance between projected points.
    \end{itemize}
\end{definition}
Notably, convex and $L-$smooth functions satisfy an important property known as co-coercivity (see Proposition 5.4 in~\cite{bach-ltfp}). 
\begin{lemma}
    Let $\ell: \Theta \to \mb{R}$ differentiable, $\Theta \subset \mb{R}^d$. 
    If $\ell$ is a convex and $L-$smooth function, then it satisfies the co-coercivity inequality,
    \begin{equation}\label{co-coercivity2}
        \forall u, v \in \Theta,\quad \langle \nabla \ell (u) - \nabla \ell (v), u-v\rangle \geq \frac{1}{L} \|\nabla \ell (u) - \nabla \ell (v)\|_2^2.
    \end{equation}
\end{lemma}

We focus on robust variants of distributed gradient descent and stochastic gradient descent.
Below we recall properties of the one-step update. Let $\mc{A} \in \{\mathrm{GD}, \mathrm{SGD}\}$, $\ell: \Theta \rightarrow \mb{R}$ differentiable and $\gamma > 0$, we denote:
\[
    G^{\mc{A}}_{\gamma}(\theta) := \theta - \gamma \frac{1}{|\mc{H}|} \sum_{i \in \mc{H}} g^{(i)}.
\]
Here, for each $i \in \mc{H}$, $g^{(i)}$ denotes the exact gradient $\nabla \ell(\theta)$ in the case of $\mathrm{GD}$, or an unbiased estimate thereof—effectively computed with a sample from the local dataset—in the case of $\mathrm{SGD}$.
We recall the expansiveness definition for an update rule from~\cite[Definition 2.3]{hardt2016train}.
\begin{definition}
    An update rule $G: \Theta \rightarrow \Theta$ is said to be $\eta-$expansive if for all $\theta, \omega \in \Theta$,
    \[
        \| G(\theta) - G(\omega) \|_2 \leq \eta \| \theta - \omega \|_2.
    \]
\end{definition}
A straightforward extension of~\cite[Lemma 3.6]{hardt2016train} to the distributed setting yields:
\begin{lemma}\label{lemmaexpansivity}
    Let $\mc{A} \in \{\mathrm{GD}, \mathrm{SGD}\}$ and $\gamma > 0$.
    Assume $\forall z \in \mc{Z}, \ell(\cdot, z): \Theta \to \mb{R}$ is L-smooth,  
    Then $G^{\mc{A}}_{\gamma}$ is $\eta_{G^{\mc{A}}_\gamma}-$expansive, with the following expansive coefficient,
    \begin{itemize}
        \item[\textit{(i)}] $\eta_{G^{\mc{A}}_\gamma} = (1+\gamma L)$.
        \item[\textit{(ii)}] Assume in addition that $\ell$ is convex. Then, for any $\gamma \leq \frac{2}{L}$, $\eta_{G^{\mc{A}}_\gamma} = 1$.
        \item[\textit{(iii)}] Assume in addition that $\ell$ is $\mu$-strongly convex. Then, for any $\gamma \leq \frac{2}{\mu + L}$, $\eta_{G^{\mc{A}}_\gamma} = (1-\gamma\frac{L\mu}{\mu+L})$. 
        For $\gamma \leq \frac{1}{L}$, since we have $L \geq \mu$, we can simplify the contractive constant to $\eta_{G^{\mc{A}}_\gamma} = (1-\gamma\mu)$.
    \end{itemize}
    $\eta_{G^{\mc{A}}_\gamma}$ is referred to as the expansivity coefficient of $\mc{A}$ with learning rate $\gamma$.
\end{lemma}

\section{\texorpdfstring{Defered proofs from section~\ref{sec:trilemma}}{Defered Proofs from Section 3}}\label{app:trilemma}

In what follows, we first provide supporting lemmas, to then prove~\Cref{th:byz-ub-convex,th:trilemma-smea}.

\subsection{Supporting lemmas}

For clarity and simplicity, we chose to emphasize our results under the broad bounded gradient assumption (i.e. Lipschitz-continuity of the loss function) in the main text. 
In this section, we present results derived under refined assumptions—namely, bounded heterogeneity and bounded variance—in place of the more general bounded gradient condition.
These assumptions are more commonly used in the optimization literature (as opposed to the generalization literature).

We formally define these assumptions below.
\begin{assumption}[Bounded heterogeneity]~\label{bounded_heterogeneity}
    There exists $G < \infty$ such that for all $\theta \in \Theta$:
    $
        \frac{1}{|\mc{H}|} \sum_{i \in \mc{H}} \| \nabla \widehat{R}_{i}(\theta) - \nabla \widehat{R}_{\mc{H}}(\theta)  \|_2^2 \leq G^2
    $.
\end{assumption}
\begin{assumption}[Bounded variance]~\label{bounded_variance}
    There exists $\sigma < \infty$ such that for each honest participant $i \in \mc{H}$ and all $\theta \in \Theta$:
    $
        \frac{1}{m} \sum_{x \in \mc{D}_i} \| \nabla \ell (\theta; x) - \nabla \widehat{R}_{i}(\theta)  \|_2^2 \leq \sigma^2
    $.
\end{assumption}

As shown in the following, our analysis techniques can yield tighter bounds when assuming bounded heterogeneity and bounded variance. 
However, we note that these bounds do not provide additional conceptual insights within the scope of our discussion.
Indeed, within the uniform stability framework, if we assume only bounded gradients and nothing more, there exist distributed learning settings where $G$ is a constant multiple of the Lipschitz constant $C$.
For instance, consider a scenario where half of the honest participants produce gradients with norm $C$, and the other half produce gradients with norm $-C$; in this case, $G$ would be $C$. 

We will need an adaptation of the Lemma D.1 from~\cite{pmlr-v202-allouah23a} that use a concentration argument on the empirical covariance matrix of Gaussian random variables:

\begin{lemma}\label{lem:spectral}
    Consider the setting described in~\Cref{II-problem-formulation}. 
    Suppose Assumptions~\labelcref{bounded_heterogeneity},~\labelcref{bounded_variance}. 
    Then, for every $t \in \{1, \ldots, T-1\}$ we have
    \begin{align*}
        \mb{E}_{\mc{A}}[ \lambda_{\max}(\frac{1}{|\mc{H}|} \sum_{i \in \mc{H}} (\widetilde{g}^{(i)}_t-\overline{\widetilde{g}}_t) (\widetilde{g}^{(i)}_t-\overline{\widetilde{g}}_t)^T) ] 
        &\leq 2 \sigma^2 + 72 \sigma^2_{\operatorname{DP}} (1 + \frac{d}{n-f}) + G^2 \\
        &\leq 3 C^2 + 72 \sigma^2_{\operatorname{DP}} (1 + \frac{d}{n-f})
    \end{align*}
    We note that, in the absence of momentum, the upper bound remains time-independent.
\end{lemma}
\begin{proof}
    The proof leverages the ability to evaluate the expectation of the controllable noise---specifically, the differential privacy noise and the noise arising from random sample selection---while conditioning on the uncontrollable noise introduced by adversarial participants. 
    The bound being independent of the latter noise, we establish the result by taking the total expectation.
    
    Let $t \in \{0, \ldots, T-1\}$, we recall:
    \[ 
        \overline{\widetilde{g}}_t 
        = \frac{1}{|\mc{H}|} \sum_{i\in \mc{H}} \widetilde{g}^{(i)}_t 
        = \frac{1}{|\mc{H}|} \sum_{i\in \mc{H}} \nabla \ell (\theta_t, x^{(i,J^{(i)}_t)}) + \xi^{(i)}_t
    \]
    \[ 
        \text{with } 
        \xi^{(i)}_t \sim \mc{N}(0, \sigma_{\operatorname{DP}}^2 I_d) 
        \text{ and } J^{(i)}_t \sim \mc{U}(\{1, \ldots, m \})
    \]
    and an alternative formulation of the maximal eigenvalue
    \[
        \Delta_t 
        = \lambda_{\max}(\frac{1}{|\mc{H}|} \sum_{i \in \mc{H}} (\widetilde{g}^{(i)}_t-\overline{\widetilde{g}}_t) (\widetilde{g}^{(i)}_t-\overline{\widetilde{g}}_t)^T) 
        = \sup_{\|v\|_2 \leq 1} \frac{1}{|\mc{H}|} \sum_{i \in \mc{H}} \langle v , \widetilde{g}^{(i)}_t-\overline{\widetilde{g}}_t \rangle ^2
    \]    
    We decompose
    \[
        \widetilde{g}^{(i)}_t-\overline{\widetilde{g}}_t  = \widetilde{g}^{(i)}_t - \nabla \widehat{R}_{i}(\theta_t) + \nabla \widehat{R}_{i}(\theta_t) - \nabla\widehat{R}_{\mc{H}}(\theta_t) + \nabla\widehat{R}_{\mc{H}}(\theta_t) - \overline{\widetilde{g}}_t
    \]
    so that we have
    \begin{multline*}
        \langle v , \widetilde{g}^{(i)}_t - \overline{\widetilde{g}}_t \rangle^2 
        = \langle v , \nabla\widehat{R}_{i}(\theta_t) - \nabla\widehat{R}_{\mc{H}}(\theta_t) \rangle^2 
        + \langle v , \widetilde{g}^{(i)}_t - \nabla\widehat{R}_{i}(\theta_t) + \nabla\widehat{R}_{\mc{H}}(\theta_t) - \overline{\widetilde{g}}_t \rangle^2 \\ 
        + 2 \langle v , \widetilde{g}^{(i)}_t - \nabla\widehat{R}_{i}(\theta_t) + \nabla\widehat{R}_{\mc{H}}(\theta_t) - \overline{\widetilde{g}}_t \rangle \langle v , \nabla\widehat{R}_{i}(\theta_t) - \nabla\widehat{R}_{\mc{H}}(\theta_t) \rangle
    \end{multline*}
    Upon averaging, taking the supremum over the unit ball, and then taking total expectations, we get:
    \begin{multline} \label{concentrationintermediate}
        \mb{E} \left[ \sup_{\|v\|_2 \leq 1} \frac{1}{|\mc{H}|} \sum_{i \in \mc{H}} \langle v , \widetilde{g}^{(i)}_t - \overline{\widetilde{g}}_t \rangle^2 \right]
        = \mb{E} \left[ \sup_{\|v\|_2 \leq 1} \frac{1}{|\mc{H}|} \sum_{i \in \mc{H}}\langle v , \nabla\widehat{R}_{i}(\theta_t) - \nabla\widehat{R}_{\mc{H}}(\theta_t) \rangle^2 \right] \\
        + \mb{E} \left[ \sup_{\|v\|_2 \leq 1} \frac{1}{|\mc{H}|} \sum_{i \in \mc{H}}\langle v , \widetilde{g}^{(i)}_t - \nabla\widehat{R}_{i}(\theta_t) + \nabla\widehat{R}_{\mc{H}}(\theta_t) - \overline{\widetilde{g}}_t \rangle^2 \right]  \\ 
        + 2 \mb{E} \left[ \sup_{\|v\|_2 \leq 1} \frac{1}{|\mc{H}|} \sum_{i \in \mc{H}} \langle v , \widetilde{g}^{(i)}_t - \nabla\widehat{R}_{i}(\theta_t) + \nabla\widehat{R}_{\mc{H}}(\theta_t) - \overline{\widetilde{g}}_t \rangle \langle v , \nabla\widehat{R}_{i}(\theta_t) - \nabla\widehat{R}_{\mc{H}}(\theta_t) \rangle \right]
    \end{multline}

    We show that the last term in~\cref{concentrationintermediate} is non-positive. 
    In fact, with $M = N + N^T$,
    \[
        N = \sum_{i \in \mc{H}} (\widetilde{g}^{(i)}_t - \nabla\widehat{R}_{i}(\theta_t) + \nabla\widehat{R}_{\mc{H}}(\theta_t) - \overline{\widetilde{g}}_t)(\nabla\widehat{R}_{i}(\theta_t) - \nabla\widehat{R}_{\mc{H}}(\theta_t))^T
    \]
    and using Lemma D.4 from~\cite{pmlr-v202-allouah23a}:
    \begin{multline*}
        \mb{E} \left[ \sup_{\|v\|_2 \leq 1} 2 \sum_{i \in \mc{H}} \langle v , \widetilde{g}^{(i)}_t - \nabla\widehat{R}_{i}(\theta_t) + \nabla\widehat{R}_{\mc{H}}(\theta_t) - \overline{\widetilde{g}}_t \rangle \langle v , \nabla\widehat{R}_{i}(\theta_t) - \nabla\widehat{R}_{\mc{H}}(\theta_t) \rangle \right] \\
        = \mb{E} \left[ \sup_{\|v\|_2 \leq 1} \langle v , Mv \rangle \right]
        \leq 9^d \sup_{\|v\|_2 \leq 1} \mb{E} \left[ \langle  v , Mv \rangle \right] \\
        = 9^d \sup_{\|v\|_2 \leq 1} \mb{E} \left[ 2 \sum_{i \in \mc{H}} \langle v , \underbrace{\mb{E}[ \widetilde{g}^{(i)}_t - \nabla\widehat{R}_{i}(\theta_t) + \nabla\widehat{R}_{\mc{H}}(\theta_t) - \overline{\widetilde{g}}_t | \theta_t ]}_{=0} \rangle \langle v , \nabla\widehat{R}_{i}(\theta_t) - \nabla\widehat{R}_{\mc{H}}(\theta_t) \rangle \right]
    \end{multline*}

    We now bound the second term in~\cref{concentrationintermediate}:
    \begin{multline*}
        \mb{E} \left[ \sup_{\|v\|_2 \leq 1} \frac{1}{|\mc{H}|} \sum_{i \in \mc{H}}\langle v , \widetilde{g}^{(i)}_t - \nabla \widehat{R}_{i}(\theta_t) + \nabla\widehat{R}_{\mc{H}}(\theta_t) - \overline{\widetilde{g}}_t \rangle^2 \right] \\
        \underset{2ab \leq a^2+b^2}{\leq} 2 \mb{E} \left[ \sup_{\|v\|_2 \leq 1} \frac{1}{|\mc{H}|} \sum_{i \in \mc{H}} \langle v , g^{(i)}_t - \nabla\widehat{R}_{i}(\theta_t) + \nabla\widehat{R}_{\mc{H}}(\theta_t) - \overline{g}_t \rangle^2 + \langle v , \xi^{(i)}_t - \overline{\xi}_t \rangle^2 \right] \\
        \underset{\substack{\text{bias-variance} \\ \text{decomposition}}}{\leq} 2 \mb{E} \left[ \sup_{\|v\|_2 \leq 1} \frac{1}{|\mc{H}|} \sum_{i \in \mc{H}} \langle v , g^{(i)}_t - \nabla \widehat{R}_{i}(\theta_t) \rangle^2 + \langle v , \xi^{(i)}_t \rangle^2 \right] \\
        \leq 2 \underbrace{\mb{E} \left[ \sup_{\|v\|_2 \leq 1} \frac{1}{|\mc{H}|} \sum_{i \in \mc{H}} \langle v , \xi^{(i)}_t \rangle^2 \right]}_{\leq\ 36 \sigma^2_{\operatorname{DP}}(1+\frac{d}{n-f}) \text{ Lemma D.6 from \cite{pmlr-v202-allouah23a}}}
        + 2 \mb{E} \left[ \frac{1}{|\mc{H}|} \sum_{i \in \mc{H}} \underbrace{\mb{E}_{J^{(i)}_t}[\| g^{(i)}_t - \nabla \widehat{R}_{i}(\theta_t) \|^2_2] }_{\leq \sigma^2} \right]
    \end{multline*}
    And finally we bound the first term in~\cref{concentrationintermediate} with~\cref{bounded_heterogeneity}:
    \begin{align*}
        \mb{E} \left[ \sup_{\|v\|_2 \leq 1} \frac{1}{|\mc{H}|} \sum_{i \in \mc{H}}\langle v , \nabla \widehat{R}_{i}(\theta_t) - \nabla\widehat{R}_{\mc{H}}(\theta_t) \rangle^2 \right] 
        & \leq \sup_{\theta \in \Theta, \|v\|_2 \leq 1} \frac{1}{|\mc{H}|} \sum_{i \in \mc{H}}\langle v , \nabla\widehat{R}_{i}(\theta_t) - \nabla\widehat{R}_{\mc{H}}(\theta_t) \rangle^2 \\
        & \leq \sup_{\theta \in \Theta} \frac{1}{|\mc{H}|} \sum_{i \in \mc{H}} \| \nabla\widehat{R}_{i}(\theta_t) - \nabla\widehat{R}_{\mc{H}}(\theta_t) \|^2_2 \\
        & \leq G^2
    \end{align*}
\end{proof}

\subsection{\texorpdfstring{Proof of~\Cref{th:byz-ub-convex}}{}}\label{app:trilemma-ub}

\begin{reptheorem}{th:byz-ub-convex}
    Consider the setting described in~\Cref{II-problem-formulation}.
    Let $\sigma_{\operatorname{DP}} > 0$ and $\gamma \leq \frac{1}{L}$.
    Suppose $\forall z \in \mc{Z}$, $\ell(\cdot;z)$ $C$-Lipschitz and $L$-smooth.
    Then $\mathcal{A}$ is $(\varepsilon, \delta)-\operatorname{DP}$ and,\
    \begin{enumerate}[label=\roman*, topsep=-3pt, itemsep=2pt, parsep=0pt, leftmargin = 23pt]
        \item[\textit{(i)}] If $\ell(\cdot;z)$ is convex $\forall z \in \mc{Z}$, then the uniform stability of $\mc{A}$ is upper bounded by
        \begin{equation}\label{eq:byz-cvx-eq-app}  
            2 \gamma C^2 T \left( \frac{1}{(n-f)m} + \sqrt{3 \kappa } + 2\sqrt{18}\sqrt{\kappa} \frac{\sigma_{\operatorname{DP}}}{C} \sqrt{1 + \frac{d}{n-f}} \right).
        \end{equation}
        \item[\textit{(ii)}] If $\ell(\cdot;z)$ is $\mu$-strongly convex $\forall z \in \mc{Z}$, then the uniform stability of $\mc{A}$ is upper bounded by
        \begin{equation}\label{eq:byz-strgcvx-eq-app}  
            \frac{2 C^2}{\mu} \left( \frac{1}{(n-f)m} + \sqrt{3 \kappa} + 2 \sqrt{18} \sqrt{\kappa} \frac{\sigma_{\operatorname{DP}}}{C} \sqrt{1 + \frac{d}{n-f}} \right).
        \end{equation}
    \end{enumerate}
\end{reptheorem}
\begin{proof}
    The proof follows exactly all the arguments of~\cite[Theorem 3.1]{boudou2025generalization}, the only difference is the use of~\Cref{lem:spectral}.
\end{proof}

\subsection{\texorpdfstring{Proof of~\Cref{th:trilemma-smea}}{}}\label{app:trilemma-lb}

In this appendix, we present the proof of~\Cref{th:trilemma-smea} first by focusing on the case $d=1$, as the extension to $d>1$ is more technical yet follows an identical proof structure, making the $d=1$ case sufficient to convey the fundamental insight and construction.
We then provide the proof of~\Cref{th:trilemma-smea} fot the case $d>1$.

\subsubsection{Proof for $d=1$.}\label{app:proof-d-1}

\begin{reptheorem}{th:trilemma-smea}
    Consider the setting in~\Cref{II-problem-formulation} with full batch and the $\mathrm{SMEA}$ aggregation rule. 
    Suppose $d=1$, $T \in \mathcal{O}\big(e^{ c_{\min} ( n-f-1 ) }\big)$, where $c_{\min}$ is a problem-dependent constant, $\frac{\varepsilon}{\sqrt{\ln(1/\delta)}} \in \Omega(1)$ and $\sigma_{\operatorname{DP}} = \frac{2C\sqrt{2T\ln(1/\delta)}}{m\varepsilon}$.
    Then there exist $\ell \in {\mb{R}}^{\Theta \times \mc{Z}}$ such that $\forall z \in \mc{Z}, \ell(\cdot; z)$ is $C$-Lipschitz, $L$-smooth and convex,
    such that the uniform stability of $\mc{A}$ is lower bounded by
    \begin{equation}\label{smea-byz-lb-eq-app}
        \Omega \left( \gamma C^2 T \left( \frac{1}{(n-f)m} + \frac{f}{n-2f} + \frac{\sigma_{\operatorname{DP}}}{C} \sqrt{\frac{f}{n-2f}} \right) \right).
    \end{equation}
    If we moreover assume $\frac{n}{3} < f < \frac{n}{2}$, the uniform stability of $\mc{A}$ is lower bounded by
    \begin{equation}\label{smea-byz-lb-eq-high-f-app}
        \Omega\left( \gamma C^2 T \left( \frac{1}{(n-f)m} + \sqrt{\frac{f}{n-2f}} + \sqrt{\frac{f}{n-2f}} \frac{\sigma_{\operatorname{DP}}}{C} \right) \right).
    \end{equation}
\end{reptheorem}
\begin{proof}
Let $d=1$, $n>0$, $\mathcal{H} = \{1, \ldots, n-f\}$ and let $S, S' \in \mathcal{Z}^{(n-f)m}$ neighboring datasets where every sample equal $0$, except for one sample in $S'$ from the participant $1$, equal to $C$.
Let $\ell: (\theta, z) \rightarrow \ell(\theta; z) = \theta z$\footnote{We only require the gradient to be zero when $z=0$ and $C$ for the perturbed data.}.
For $i\in \mathcal{H}:=\{1, \ldots, n-f\}$, $0 \leq t \leq T$, let $\xi^{(i)}_t \sim \mathcal{N}(0,\sigma_{\operatorname{DP}}^2)$ be independent random gaussian noise.


\textbf{Byzantine participants trigger.}
We analyze two attacker capabilities. 
First, if Byzantine participants have access to the random seeds of honest participants, they can circumvent the differential privacy guarantee.
Indeed, this knowledge allows them to perfectly reconstruct the added noise, subtract it, and deterministically identify, at the first iteration, whether the first participant's true gradient was $\frac{C}{m}$ or $0$.
In the alternative scenario, where attackers only observe the final noisy gradient (lacking the seed), they cannot perform this deterministic inference. 
We must then rely on a statistical test to bound their success probability to distinguish the two cases.
We explain a possible test in what follows. 
\looseness=-1

Let $p_{\text{test}} \in (0,1)$. 
We define the behavior of Byzantine participants based on a trigger condition.
This trigger is determined by performing a gaussian mean hypothesis test on the first $n_{\text{test}} \vcentcolon= p_{\text{test}} T$ sample, collected during an initial fraction of the total $T$ iterations.
The test evaluates whether the sample mean corresponds to hypothesis $H_0$ (mean 0, equivalently the dataset is $S$) or to hypothesis $H_{1}$ (mean $\frac{C}{m}$, equivalently the dataset is $S'$), assuming a known variance $\sigma^2_{\operatorname{DP}}$.
This test processes $\bar\zeta$ which evaluates at  $\frac{C}{m} + \bar\xi$ under $H_{1}$, and evaluates at $\bar\xi$ under $H_{0}$, where $\bar\xi = \frac{1}{n_{\text{test}}} \sum_{t=1}^{n_{\text{test}}} \xi^{(1)}_t$.
Overall, the Byzantine participants undergo a membership inference attack to know whether the gradients from participant $1$ equals $g^{(1)}_t = \xi^{(1)}_t$ or $g^{\prime(1)}_t = \frac{C}{m} + \xi^{(1)}_t$, i.e whether the dataset is $S$ or $S'$.
The decision is made depending on a trigger value $c_{\text{trigger}}$,
\begin{equation*}
    \text{reject } H_0 \text{ if } \zeta > c_{\text{trigger}}, \text{ else accept } H_0.
\end{equation*}
We denote $Z \sim \mathcal{N}(0, 1)$ a standard gaussian random variable, and $z_\alpha \in \mathbb{R}$ the value such that $\mathbb{P}(Z > z_\alpha) = \alpha$, $\alpha \in (0,1)$. 
Conditioned on $H_0$, $\zeta | H_0 \sim \mathcal{N}(0, \frac{\sigma_{\text{DP}}^2}{n_{\text{test}}})$, hence  setting $c_{\text{trigger}} = z_\alpha \frac{\sigma_{\operatorname{DP}}}{\sqrt{n_{\text{test}}}}$, 
\[
    \mathbb{P}\left( \text{reject } H_0 | H_0 \right) 
    = \mathbb{P}\left( \zeta > c_{\text{trigger}} | H_0 \right) 
    = \mathbb{P}\left( Z > z_\alpha \right) 
    = \alpha.
\]
Let denote $E_1$ the event that the Byzantine workers successfully infer the underlying datasets from the gradient statistics (successful MIA). 
Under the optimal hypothesis test, we have
\begin{align*}
    E_1 &= \Big\{ \bar\xi \text{ does not trigger } H_0 \text{ rejection} \Big\} \cap \Big\{ \text{shifting } \bar\xi \text{ by} \frac{C}{m} \text{ triggers } H_0 \text{ rejection} \Big\} \\
    &= \Big\{ \bar\xi \leq c_{\text{trigger}} \Big\} \cap \Big\{ \bar\xi + \frac{C}{m} > c_{\text{trigger}} \Big\} \\
    &= \Big\{ c_{\text{trigger}} - \frac{C}{m} < \bar\xi \leq c_{\text{trigger}} \Big\}.
\end{align*}
And,
\begin{align*}
    \mathbb{P}\left( E_1 \right) 
    = \mathbb{P}\left( \Big\{ \bar\xi \leq c_{\text{trigger}} \Big\} \cap \Big\{ \bar\xi + \frac{C}{m} > c_{\text{trigger}} \Big\} \right)
    &= \mathbb{P}\left( c_{\text{trigger}} - \frac{C}{m} < \bar\xi \leq c_{\text{trigger}} \right) \\
    &= \mathbb{P}\left( z_\alpha - \frac{C\sqrt{n_{\text{test}}}}{m\sigma_{\operatorname{DP}}} < Z \leq z_\alpha \right) \\
    &= \mathbb{P}\bigg( z_\alpha - \frac{\varepsilon \sqrt{p_{\text{test}}}}{2\sqrt{2\ln(1/\delta)}} < Z \leq z_\alpha \bigg).
\end{align*}
This means that $\mathbb{P}\left( E_1 \right) \in \Omega(1)$ if and only if there exist a constant $k$ such that $\varepsilon \geq k \sqrt{\ln(1/\delta)}$.


\textbf{Byzantine participants behavior.} 
We assign the following behavior (\Cref{alg:algo1}) to the Byzantine participants $i \in \{n-f+1, n\}$, for $t \geq n_{\text{test}} + 1$.

{\centering
    \begin{minipage}{.45\linewidth}
        \begin{algorithm}[H]
        \caption{Byzantine Participant Behavior}\label{alg:algo1}
        \begin{algorithmic}
            \STATE \textcolor{gray}{\# $\beta$ (defined below) depends on $f$ and $n$} 
            \IF{$\bar\zeta > c_{\text{trigger}}$}
            \STATE \textbf{return} $\beta \sdp$
            \ELSE
            \STATE \textbf{return} $- \beta \sdp$
            \ENDIF
        \end{algorithmic}
        \end{algorithm}
    \end{minipage}
    \par
}\

That is, depending on the input $\bar\zeta$, which itself depends only on the observed noisy gradients, the Byzantine participants bias the trajectory toward either positive or negative values.

In what follows, we decompose our analyze depending on whether the Byzantine participants participants a successful MIA (i.e., infer whether the dataset is $\mathcal{S}$ or $\mathcal{S}'$ successfully) in the two coupled optimization trajectories ($E_1$), or not ($E_1^{\mathsf{c}}$). 
\begin{itemize}
    \item[\textit{(i)}] Under $E_1$ and the trajectory produced by $S'$ (i.e., with $g_t^{(1)\prime}= \frac{1}{m} \sum_{j=1}^m z^{(1,j)\prime} + \xi^{(1)}_t = \frac{C}{m} + \xi^{(1)}_t$), the Byzantine participants favor the positive noises. 
    This is possible because of similar computations as in~\cite[Theorem 3.2]{boudou2025generalization}.
    Following this reasoning, we set the Byzantine values to be
    \[
        \forall i \in \{n-f+1, \ldots, n\}, t\in \{n_{\text{test}} + 1, \ldots, T\}, g^{(i)}_t \vcentcolon= \beta \sigma_{\operatorname{DP}}, 
    \] 
    where 
    \[
        \beta \vcentcolon= \frac{1}{2}\sqrt{\frac{(n-f-1)(n-f)}{f(n-2f)}}.
    \]
    In what follows, we show that, with high probability
    \[
        \text{for } t\in \{n_{\text{test}} + 1, \ldots, T\}, \text{Var}({\{g_t^{(i)}\}}_{i \in S_t }) \leq \text{Var}({\{g_t^{(i)}\}}_{i \in \mathcal{H}}),
    \]
    where $S_t = \{n-f+1, \ldots, n\} \cup \text{Select}_{n-2f}\left( \mathcal{H} \right)$, i.e. $\mathrm{SMEA}(g^{(1)}_t, \ldots, g^{(n)}_t) = \frac{1}{n-f} \sum_{i \in S_t} g^{(i)}_t$
    This is ensured by a concentration argument and the fact that
    \[
        \mathbb{E} \left[ \text{Var}({\{g_t^{(i)}\}}_{i \in \mathcal{H}}) \right] 
        = \frac{n-f-1}{n-f} \sigma_{\operatorname{DP}}^2 
        > \mathbb{E} \left[ \text{Var}({\{g_t^{(i)}\}}_{i \in S_t }) \right].
    \]

    We demontrate this rigourously in the following points.
    \begin{itemize}
        \item[\textit{(a)}] 
            First, we compute the last expectation,
            \begin{equation*}
                \mathbb{E} \left[ \text{Var}({\{g_t^{(i)}\}}_{i \in S_t }) \right] 
                \leq \frac{3}{4} \mathbb{E} \left[ \operatorname{Var}({\{g_t^{(i)}\}}_{i \in \mathcal{H}}) \right]
            \end{equation*}

            In fact, we use the following indentity
            \begin{align*}
                \mathbb{E} \left[ \operatorname{Var}({\{g_t^{(i)}\}}_{i \in S_t }) \right] 
                &= \mathbb{E}\left[ \frac{1}{n-f} \sum_{i \in S_t} {(g^{(i)}_t - \bar{g}_{S_t})}^2 \right] \\
                &= \frac{1}{n-f} \mathbb{E}\left[\sum_{i \in S_t} {g^{(i)}_t}^2\right] - \mathbb{E}\left[{\bar{g}_{S_t}}^2\right] \\
                &= \frac{1}{n-f} \mathbb{E}\left[\sum_{i \in S_t} {g^{(i)}_t}^2\right] - \operatorname{Var}(\bar{g}_{S_t}) - {\mathbb{E}\left[ \bar{g}_{S_t} \right]}^2,
            \end{align*}
            and compute the three terms separetely. Let assume $1 \in S_t$ (we simply take $C=0$ otherwise),

            \begin{align*}
                \mathbb{E}\left[\sum_{i \in S_t} {g^{(i)}_t}^2\right] 
                &= \left( \frac{C^2}{m^2} + \sigma_{\operatorname{DP}}^2 \right) + (n-2f-1) \sigma_{\operatorname{DP}}^2 + f \beta^2 \sigma_{\operatorname{DP}}^2 \\
                &= \frac{C^2}{m^2} + \sigma_{\operatorname{DP}}^2 \left(n-2f + f \beta^2\right).
            \end{align*}
            Moreover, $\mathbb{E}\left[ \bar{g}_{S_t} \right] = \frac{f}{n-f} \beta \sigma_{\operatorname{DP}} + \frac{C}{(n-f)m}$,
            and by independence
            \begin{align*}
                \text{Var}(\bar{g}_t) 
                &= \frac{1}{(n-f)^2} \left( (n-2f) \operatorname{Var}(\mathcal{N}(0, \sigma_{\operatorname{DP}}^2)) + \operatorname{Var}(f \beta \sigma_{\operatorname{DP}}) \right) 
                = \frac{n-2f}{(n-f)^2} \sigma_{\operatorname{DP}}^2
            \end{align*}

            As a result,
            \begin{align*}
                \mathbb{E} \left[ \text{Var}({\{g_t^{(i)}\}}_{i \in S_t }) \right]
                &= \sigma_{\operatorname{DP}}^2 \left[ \frac{n-2f}{n-f} + \frac{f}{n-f} \beta^2 - \frac{n-2f}{(n-f)^2} - \frac{f^2}{(n-f)^2} \beta^2 \right] \\
                &\qquad+ \frac{C/m}{(n-f)^2} \left[(n-f-1) \frac{C}{m} - 2f\beta\sigma_{\text{DP}}\right].
            \end{align*}
            For clarity, our analysis will focus on the case where $1 \notin S_t$.
            The alternative scenario, $1 \in S_t$, introduces an additional cross-term proportional to $(n-f-1) \frac{C}{m} - 2f\beta\sigma_{\operatorname{DP}}$ due to the modified gradient at participant $1$.
            A detailed sign analysis of this term (which is typically negative) is not central to our main argument and does not materially alter the proof's conclusion. 
            We therefore proceed with the representative case where $1 \in S_t$.
            Moreover, 
            \begin{align*}
                \mathbb{E} \left[ \text{Var}({\{g_t^{(i)}\}}_{i \in S_t }) \right]
                &\leq \sigma_{\operatorname{DP}}^2 \frac{n-f-1}{n-f} \left[ \frac{n-2f}{n-f} + \frac{f(n-2f)}{(n-f-1)(n-f)} \beta^2 \right] \\
                &= \sigma_{\operatorname{DP}}^2 \frac{n-f-1}{n-f} \left[ \frac{n-2f}{n-f} + \frac{1}{4} \right] \\
                &= \sigma_{\operatorname{DP}}^2 \frac{n-f-1}{n-f} \left[ \frac{3}{4} - \left( \frac{f}{n-f} - \frac{1}{2} \right) \right] \\
                &\leq \frac{3}{4} \sigma_{\operatorname{DP}}^2 \frac{n-f-1}{n-f} 
                < \mathbb{E} \left[ \operatorname{Var}({\{g_t^{(i)}\}}_{i \in \mathcal{H}}) \right].
            \end{align*}

        \item[\textit{(b)}] 
            Second, we control with high probability that $\mathrm{SMEA}$ chooses $S_t$ when $1 \notin S_t$, the case $1 \in S_t$ follows similarly. 
            Let denote ${g}^{(\mathcal{H})}_t \vcentcolon= (g^{(i)}_t)_{i \in \mathcal{H}}$
            and, by the Cochran's Theorem~\citep{Cochran_1934},
            \[
                Y_t \vcentcolon= \frac{1}{\sigma_{\operatorname{DP}}^2}\sum_{i\in\mathcal{H}} {\left( g^{(i)}_t - \bar{g}_t \right)}^2 = \frac{1}{\sigma_{\operatorname{DP}}^2} {{g}^{(\mathcal{H})}_t}^\intercal {P}_{n-f} {g}^{(\mathcal{H})}_t = \sum_{i=1}^{n-f-1} {Y^{(i)}_t}^2 \sim \chi^2_{n-f-1},
            \]
            where $\forall i \in \{1, \ldots, n-f-1\}, Y^{(i)}_t \sim \mathcal{N}(0, 1)$, 
            $\chi^2_{n-f-1}$ is the chi-squared distribution with $n-f-1$ degree of freedom, 
            ${P}_{n-f} = {I}_{n-f} - \frac{1}{n-f}{1}_{n-f}{1}_{n-f} ^\intercal$ is an orthogonal projection matrix onto the $n-f-1$-dimensional subspace of vectors whose components sum to zero, 
            ${I}_{n-f}$ is the identity matrix of size $(n-f)\times(n-f)$, 
            and ${1}_{n-f}$ the vector of size $n-f$ with every entry equal to 1.
            Hence, we have from~\cite[Lemma 1]{chi-squared-concentration}, for any $s \geq 0$
            \[
                \mathbb{P}\left( Y_t \leq n-f-1 - s \right) 
                \leq e^{- \frac{s^2}{4(n-f-1)} },
            \]
            that is, because $\frac{\operatorname{Var}({\{g_t^{(i)}\}}_{i \in \mathcal{H}})}{\frac{n-f-1}{n-f} \sigma_{\operatorname{DP}}^2} = \frac{Y_t}{n-f-1}$, we have for $s = \frac{1}{8}\frac{n-f-1}{n-f} \sigma_{\text{DP}}^2$
            \begin{align*}
                \mathbb{P}\left(E^{\mathsf{c}}_{2, t}\right)
                & \vcentcolon= \mathbb{P}\left( \operatorname{Var}({\{g_t^{(i)}\}}_{i \in \mathcal{H}}) \leq \frac{n-f-1}{n-f} \sigma_{\operatorname{DP}}^2 - s =  \frac{7}{8}\frac{n-f-1}{n-f} \sigma_{\text{DP}}^2 \right) \\
                & \leq e^{- \frac{(n-f)^2}{4(n-f-1)\sigma_{\operatorname{DP}}^4} s^2} 
                = e^{- \frac{n-f-1}{256}},
            \end{align*}
            and 
            \[
                \mathbb{P}\left( \bigcap_{t\in\{n_{\text{test}} + 1, \ldots, T\}} E_{2, t} \right) 
                \geq {\left( 1 - e^{- \frac{n-f-1}{256}} \right)}^{T-n_{\text{test}}} 
                \in \Omega(1) \text{ if on only if } T \in \mathcal{O}\left( e^{\frac{n-f-1}{256}} \right).
            \]

        \item[\textit{(c)}] 
            Finally, we control the variance which include the Byzantine participant vectors with high probability.        
            Let denote ${g}_t \vcentcolon= \begin{pmatrix} {h}_t \vcentcolon= (g^{(i)}_t)_{i \in \mathcal{H} \cap S_t} \\ {b}_t \vcentcolon= \beta \sigma_{\operatorname{DP}} 1_f \end{pmatrix}$ the vector containing every values from ${\{g^{(i)}_t\}}_{i \in S_t}$.
            Then  
            \begin{align*}
                \operatorname{Var}({\{g_t^{(i)}\}}_{i \in S_t}) 
                &= \frac{1}{n-f} \sum_{i \in S_t} {(g^{(i)}_t - \bar{g}_{S_t})}^2
                = \frac{1}{n-f} {g}_t^\intercal {P}_{n-f} {g}_t\\
                &=\underbrace{\frac{1}{n-f} {h}_t^\intercal \left( {I}_{n-2f} - \frac{1}{n-f}{1}_{n-2f}{1}_{n-2f} ^\intercal\right) {h}_t}_{V_1} \\
                &\qquad\qquad- \underbrace{\frac{2}{(n-f)^2}{h}_t^\intercal {1}_{n-2f}{1}_{f} ^\intercal {b}_t}_{V_2} \\
                &\qquad\qquad+ \underbrace{\frac{1}{n-f}{b}_t^\intercal \left({I}_{f} - \frac{1}{n-f}{1}_{f}{1}_{f} ^\intercal\right) {b}_t}_{V_3}
                \Bigg] \\ 
                &= V_1 
                - \frac{2 f \beta \sigma_{\operatorname{DP}}}{(n-f)^2} \sum_{i \in \mathcal{H} \cap S_t} \xi^{(i)}_t  
                + \frac{f(n-2f)}{(n-f)^2} \beta^2 {\sigma_{\operatorname{DP}}}^2.
            \end{align*}

            \begin{itemize}
                \item[($V_1$)] The concentration property of $V_1$, whether $1 \in S_t$ or note is the same.
                    We use the Hanson-Wright inequality from~\cite[Theorem 1.1]{rudelson2013hanson}, with $M = {I}_{n-2f} - \frac{1}{n-f}{1}_{n-2f}{1}_{n-2f} ^\intercal$
                    \begin{equation*}
                        \mathbb{P}\left( V_1 > \mathbb{E} \left[ V_1 \right] + s \right) 
                        \leq e^{- c_{1} \min \{ \frac{s^2 {(n-f)}^2}{\sigma_{\operatorname{DP}}^4 \| M \|_{\text{F}}^2}, \frac{s (n-f)}{\sigma_{\text{DP}}^2 \| M \|_{\text{sp}} } \}}.
                    \end{equation*}
                    where $c_{1}>0$ is a constant, $\| M \|_{\operatorname{F}}^2 = n-2f-1+\left(\frac{f}{n-f}\right)^2 \leq n-2f$ and $\| M \|_{\text{sp}} = 1$ as the eigenvalue of $M$ are $1$ with multiplicity $n-2f-1$ and $\frac{f}{n-f}$. 
                    That is, $\forall s \geq 0$,
                    \[
                        p_1(s) \vcentcolon=
                        \mathbb{P}\left( V_1 > \mathbb{E} \left[ V_1 \right] + s \right) 
                        \leq e^{- c_{1} \min \{ \frac{s^2 (n-f)^2}{\sigma_{\operatorname{DP}}^4 (n-2f)}, \frac{s (n-f)}{\sigma_{\operatorname{DP}}^2} \}}.
                    \]

                \item[($V_2$)] For $V_2$, $t\geq n_{\operatorname{test}}$ and $i \in \mathcal{H} \cap S_t$, 
                \begin{equation*}
                    \frac{2f\beta \sigma_{\operatorname{DP}}}{(n-f)^2} \sum_{i \in \mathcal{H} \cap S_t} \xi^{(i)}_t \sim \mathcal{N}\left( 0, \left( \frac{2f\beta \sigma_{\operatorname{DP}} }{(n-f)^2}\right)^2 (n-2f) \sigma_{\operatorname{DP}}^2 \right),
                \end{equation*} 
                i.e. $\forall s \geq 0$,
                    \begin{multline*}
                        p_2(s) 
                        \vcentcolon= \mathbb{P}\left( V_2 > \mathbb{E} \left[ V_2 \right] + s = s \right)
                        = \mathbb{P}\left( V_2 < - s \right)  \\
                        \leq e^{- \frac{s^2(n-f)^4}{8 f^2 \beta^2 (n-2f) \sigma_{\operatorname{DP}}^4}}
                        = e^{- \frac{s^2(n-f)^3}{8 f (n-f-1) \sigma_{\operatorname{DP}}^4}}.
                    \end{multline*}
                
                \item[($V_3$)] For $V_3$, $t\geq n_{\operatorname{test}}$, $\frac{f(n-2f)}{(n-f)^2} \beta^2 \sigma_{\operatorname{DP}}^2$ is perfectly concentrated as Byzantine participants send constant values.
            \end{itemize}
            
            Hence, with
            \[
                s = \frac{1}{8} \sigma_{\operatorname{DP}}^2 \frac{n-f-1}{n-f},
            \]
            we have $\mathbb{E} \left[ {g}_t^\intercal \left(\frac{1}{n-f} {P}_{n-f}\right) {g}_t \right] + s \leq \frac{7}{8} \frac{n-f-1}{n-f} \sigma_{\operatorname{DP}}^2$, and using the union bound,
            \begin{align*}
                \mathbb{P}\left(E^{\mathsf{c}}_{3, t}\right) 
                &\vcentcolon= \mathbb{P}\left( {g}_t^\intercal \left(\frac{1}{n-f} {P}_{n-f}\right) {g}_t > \mathbb{E} \left[ {g}_t^\intercal \left(\frac{1}{n-f} {P}_{n-f}\right) {g}_t \right] + s \right) \\
                &= \mathbb{P}\left( V_1 - V_2 > \frac{n-f-1}{n-f} \frac{n-2f}{n-f} \sigma_{\operatorname{DP}}^2 + s \right) \\
                &\leq \mathbb{P}\left( V_1 - V_2 > s = \frac{1}{8}\frac{n-f-1}{n-f} \sigma_{\operatorname{DP}}^2 \right) \\
                &\leq p_1(\frac{s}{2}) + p_2(\frac{s}{2}) \\
                &= 
                e^{- \frac{c_1}{16} (n-f-1) \min \{ \frac{n-f-1}{16(n-2f)}, 1 \} }
                + e^{- \frac{(n-f)(n-f-1)}{2048 f}}.
            \end{align*}
            As a result, 
            \begin{multline*}
                \mathbb{P}\left( \bigcap_{t\in\{n_{\text{test}} + 1, \ldots, T\}} E_{3, t} \right) 
                \geq {\left( 1 - p_1(\frac{s}{2}) - p_2(\frac{s}{2}) \right)}^{T-n_{\text{test}}} \\
                = {\left( 1 - e^{- \frac{c_1}{16} (n-f-1) \min \{ \frac{n-f-1}{16(n-2f)}, 1 \} } - e^{- \frac{(n-f)(n-f-1)}{2048 f}} \right)}^{T-n_{\text{test}}} \\
                \in \Omega(1) \text{ if and only if } T \in \mathcal{O}\left( \min\{ e^{\frac{c_1}{16} (n-f-1) \min \{ \frac{n-f-1}{16(n-2f)}, 1 \} }, e^{\frac{(n-f)(n-f-1)}{2048 f}} \} \right).
            \end{multline*}
    \end{itemize}
    
    \item[\textit{(ii)}] Under $E_1$ and the trajectory produced by $S$ (i.e., with $g_t^{(1)}= \frac{1}{m} \sum_{j=1}^m z^{(1,j)} + \xi^{(1)}_t = \xi^{(1)}_t$), we do an analogous analysis as in \textit{(i)}, in the opposite direction. 
    That is, due to~\Cref{alg:algo1}, the Byzantine participants successfully steer the optimization trajectory in the negative direction. 
    \item[\textit{(iii)}] Otherwise, in the complementary event $E^{\mathsf{c}}_1$, 
    the influence of Byzantine participants is considered negligible, and any potential contribution to instability from this case is disregarded.
\end{itemize}

\textbf{Consequence for the lower bound.}
Recall that we controll the probability of the following event
\[ 
    E = E_1 \cap \bigcap_{t\in\{n_{\text{test}} + 1, \ldots, T\}} \big( E_{2, t} \cap E_{3, t}\big), 
\]
Conditioned on $E$, the parameter divergence is bounded by
\begin{align*}
    \mathbb{E}\left[ \left| \theta_T - \theta^\prime_T \right| | E \right] 
    \geq \mathbb{E}\left[ \left| \theta^\prime_T \right| | E \right]
    &\geq \frac{\gamma}{n-f} \sum_{t=n_{\text{test}} + 1}^T f \beta \sigma_{\text{DP}} \\
    &= \frac{1}{2} \gamma (T-n_{\text{test}}) \sigma_{\operatorname{DP}}  \frac{f}{n-f} \sqrt{\frac{(n-f-1)(n-f)}{f(n-2f)}} \\
    &= \frac{1}{2} (1-p_{\text{test}}) \gamma T \sigma_{\operatorname{DP}} \sqrt{\frac{f}{n-f}\left(1 + \frac{f-1}{n-2f}\right)}.
\end{align*}
Moreover, note that for $t, s \in\{n_{\text{test}} + 1, \ldots, T\}$, $t \neq s$, $E_1$ is independent of $E_{2,t}$ and $E_{3,t}$, $E_{2,t} \cap E_{3,t}$ is independent of $E_{2,s} \cap E_{3,s}$ and 
\begin{align*}
    \mathbb{P}\left( E \right) 
    &= \mathbb{P}\left(E_1\right) \prod_{t\in\{n_{\text{test}} + 1, \ldots, T\}} \mathbb{P}\left(E_{2, t} \cap E_{3, t}\right) \\
    &\geq \mathbb{P}\left(E_1\right) \prod_{t\in\{n_{\text{test}} + 1, \ldots, T\}} \left(1 - \mathbb{P}\left(E_{2, t}^\mathsf{c} \right) - \mathbb{P}\left( E_{3, t}^\mathsf{c}\right) \right) \\
    &\geq \mathbb{P}\left(E_1\right) {\left( 1 - e^{- \frac{n-f-1}{256}} - e^{- \frac{c_1}{16} (n-f-1) \min \{ \frac{n-f-1}{16(n-2f)}, 1 \} } - e^{- \frac{(n-f)(n-f-1)}{2048 f}} \right)}^{T-n_{\text{test}}}
    \in \Omega(1).
\end{align*}
As a result, we have
\[
    \mathbb{E}\left[ \left| \theta_T - \theta^\prime_T \right| \right] 
    \geq \mathbb{E}\left[ \left| \theta_T - \theta^\prime_T \right| | E \right] \mathbb{P}\left( E \right)
    \in \Omega \left( \gamma T \sigma_{\operatorname{DP}} \sqrt{\frac{f}{n-2f}} \right),
\]
and the uniform stability is lower bounded by
\[
    \sup_{z\in[-C,C]} \mathbb{E}\left[ \left| \ell(\theta_T, z) - \ell(\theta^\prime_T, z) \right| \right] 
    \geq C \mathbb{E}\left[ \left| \theta_T - \theta^\prime_T \right| | E \right] \mathbb{P}\left( E \right)
    \in \Omega \left( \gamma T C \sigma_{\operatorname{DP}} \sqrt{\frac{f}{n-2f}} \right).
\]

\textbf{Gluing lower bounds.} Moreover, when $\sigma_{\operatorname{DP}} = 0$, the uniform stability is lower bounded by $\Omega\left(\gamma C^2 T \left( \frac{1}{(n-f)m} + \sqrt{\frac{f}{n-2f}}\right)\right)$ when $\frac{n}{3} \leq f < \frac{n}{2}$, 
and by $\Omega\left(\gamma C^2 T \left( \frac{1}{(n-f)m} + \frac{f}{n-2f} \right)\right)$ when $f < \frac{n}{3}$~\citep[Theorem 3.2]{boudou2025generalization}. 
We conclude the proof with the formula $\max\{a,b\} \geq \frac{a+b}{2} \in \Omega(a+b)$, i.e.\ when $\frac{n}{3} \leq f < \frac{n}{2}$ the uniform stability is lower bounded by
\begin{equation*}
    \Omega\left( \gamma C^2 T \left( \frac{1}{(n-f)m} + \sqrt{\frac{f}{n-2f}} \left( 1 + \frac{\sigma_{\operatorname{DP}}}{C} \right) \right) \right).
\end{equation*}
The case $f < \frac{n}{3}$ is analogous.
Note that $\sqrt{\kappa_{\operatorname{SMEA}}}$ is of the same order as $\sqrt{\frac{f}{n-2f}}$ if we assume there exist a constant $\nu > 0$ such that $f \leq \frac{n}{2+\nu}$, i.e.\ the lower bound matches the upper bound.
\end{proof}

\subsubsection{Proof for $d>1$.}\label{app:proof-trilemma-dhigh}

\begin{assumption}\label{assump:lowerbound-d-high}
    We assume the the following.
    \begin{enumerate}
        \item 
        There exists a constant $\alpha_{\min} > 0$ such that $\alpha \vcentcolon= \frac{f}{n-f} \in [\alpha_{\min}, 1/2)$.
        \item 
        There exist constants $\mu, \nu > 0$ such that $\mu \leq \frac{d}{n-f} \leq \nu$. 
        \item $\mu > \mu_{\min}$ and $n-f-1 \geq N_{\min}$, where $\mu_{\min}$ and $N_{\min}$ are constants that depends on assumption 1. \&\ 2., and the problem spectral and finite-sample structure.
        That is,  
        \begin{equation*}
            \mu_{\min} \vcentcolon= \frac{4}{c_{\max}^2 (1-r_{\max})}, 
        \end{equation*}
        and
        \begin{equation*} 
            N_{\min} \vcentcolon= \max\left\{ \left( \frac{3(\sqrt{\nu} + \sqrt{1-\alpha_{\min}} + \max\{C_1,C_2\})}{1-\sqrt{1-\alpha_{\min}}} \right)^2, 
            \frac{16 \min(1, \nu) }{ c_{\delta} \sqrt{\mu} \sqrt{ 2 + \frac{ \alpha_{\min} \mu c^2}{\min(1, \nu)}}} \right\},
        \end{equation*}
        where $C_1, C_2 > 0$ are universal constants, and $c_{\max}, r_{\max} \in (0,1)$ are defined as  
        \begin{equation*}
            r_{\max} \vcentcolon= \left( \frac{\sqrt{\nu} + \frac{1}{3} + \frac{2}{3}\sqrt{1-\alpha_{\min}}}{\sqrt{\nu} + \frac{2}{3} + \frac{1}{3}\sqrt{1-\alpha_{\min}}} \right)^2 \geq \frac{\E\|X_\perp\|^2_{\operatorname{sp}}}{\E\|X\|^2_{\operatorname{sp}}},
        \end{equation*}
        \begin{equation*}
            c_{\max}^2 \vcentcolon= 1 - \frac{1-\alpha_{\min}}{\left(\sqrt{\nu} + 1 - \frac{1 - \sqrt{1-\alpha_{\min}}}{3} \right)^2} \leq \frac{ \mathbb{E} \left[ (DZ)^\intercal P_{n-f} D Z \right]}{ \E\|X\|^2_{\operatorname{sp}}},
        \end{equation*}
        where $Z \sim \mathcal{N}(0, I_{n-f})$, $X \! \in \! \mathbb{R}^{d \times (n-f-1)}$, $X_\perp \! \in \! \mathbb{R}^{(d-1) \times (n-2f)}$ are random matrices with i.i.d.\ standard normal entries. 
        $D \in \mathbb{R}^{(n-f) \times (n-f)}$ is a diagonal masking matrix with $n-2f$ ones and $f$ zeros, and $P_{n-f} = I_{n-f} - \frac{1}{n-f}1_{n-f}1_{n-f}^\intercal$ is a centering projection matrix.
    \end{enumerate}
\end{assumption}

\ 

\Cref{assump:lowerbound-d-high}, while technical are fairly mild. 
They are either practically always true (e.g., $T$) or assumed to bound expectation and concentration of random matrices (e.g., $\mu_{\min}$ and $N_{\min}$).

The quantities $\frac{\E\|X_\perp\|^2_{\operatorname{sp}}}{\E\|X\|^2_{\operatorname{sp}}}$ and $\frac{ \mathbb{E} \left[ (DZ)^\intercal P_{n-f} D Z \right]}{ \E\|X\|^2_{\operatorname{sp}}}$ have interpretation within the context of our proof. 
Indeed, $\frac{\E\|X_\perp\|^2_{\operatorname{sp}}}{\E\|X\|^2_{\operatorname{sp}}}$ is the ratio between the expected principal variance of the honest noise orthogonal to the attack and the expected principal variance of all honest noise. 
Intuitively, it measures how much of the honest noise aligns with the attack direction (which is arbitrary in our proof), dictating how easily the attacker can hide inside the honest variance.
Finally, $\frac{ \mathbb{E} \left[ (DZ)^\intercal P_{n-f} D Z \right]}{ \E\|X\|^2_{\operatorname{sp}}}$ is the ratio between the expected variance of the honest noise along the attack axis and the expected principal variance of all the honest noise.
Intuitively, this measures how much the orthogonal honest noise tilts the principal eigenvector away from the attack direction.
We now turn to the proof of~\Cref{th:trilemma-smea}.

\ 

\begin{reptheorem}{th:trilemma-smea}
    Consider the setting in~\Cref{II-problem-formulation} with full batch and the $\mathrm{SMEA}$ aggregation rule. 
    Suppose~\Cref{assump:lowerbound-d-high}, $T \in \mathcal{O}\big(e^{c_{\min} \left( \sqrt{d} + \sqrt{n-f-1} \right)^2 }\big)$ where $c_{\min}$ is a problem-dependent constant, $\frac{\varepsilon}{\sqrt{\ln(1/\delta)}} \in \Omega(1)$ and $\sigma_{\operatorname{DP}} = \frac{2C\sqrt{2T\ln(1/\delta)}}{m\varepsilon}$.
    Then there exist $\ell \in {\mb{R}}^{\Theta \times \mc{Z}}$ such that $\forall z \in \mc{Z}, \ell(\cdot; z)$ is $C$-Lipschitz, $L$-smooth and convex,
    such that the uniform stability of $\mc{A}$ is lower bounded by
    \begin{equation}\label{smea-byz-lb-eq-d-app}
        \Omega \left( \gamma C^2 T \left( \frac{1}{(n-f)m} + \frac{f}{n-2f} + \sqrt{\frac{f}{n-2f}} \frac{\sigma_{\operatorname{DP}}}{C} \left( 1 + \sqrt{\frac{d}{n-f}} \right) \right) \right).
    \end{equation}
    If we moreover assume $\frac{n}{3} < f < \frac{n}{2}$, the uniform stability of $\mc{A}$ is lower bounded by
    \begin{equation}\label{smea-byz-lb-eq-high-f-d-app}
        \Omega\left( \gamma C^2 T \left( \frac{1}{(n-f)m} + \sqrt{\frac{f}{n-2f}} + \sqrt{\frac{f}{n-2f}} \frac{\sigma_{\operatorname{DP}}}{C} \left( 1 + \sqrt{\frac{d}{n-f}} \right) \right) \right).
    \end{equation}
\end{reptheorem}
\begin{proof}
Let $d, n \geq 1$, $\mathcal{H} = \{1, \ldots, n-f\}$ and let $S, S' \in \mathcal{Z}^{(n-f)m}$ neighboring datasets where every sample equal $0$, except for one sample in $S'$ from the participant $1$, equal to $-C e_1$.
Let $\ell: \R^d \times \widebar{B}(0, C) \to \R, (\theta, z) \mapsto \ell(\theta; z) = \theta^\intercal z$.
For $i\in \mathcal{H}:=\{1, \ldots, n-f\}$, $0 \leq t \leq T$, let $\xi^{(i)}_t \sim \mathcal{N}(0,\sigma_{\operatorname{DP}}^2 I_d)$ be independent random gaussian noise.

\paragraph{Byzantine participants trigger.}
We analyze two attacker capabilities. 
First, if Byzantine participants have access to the random seeds of honest participants, they can circumvent the differential privacy guarantee.
Indeed, this knowledge allows them to perfectly reconstruct the added noise, subtract it, and thus deterministically identify, at the first iteration, whether the first participant's true gradient was $-\frac{C}{m}e_1$ or $0$.
In the alternative scenario, where attackers only observe the final noisy gradient (lacking the seed), they cannot perform this deterministic inference. 
We must then rely on a statistical test to bound their success probability to distinguish the two cases.
We explain a possible test in what follows.

Let $p_{\text{test}} \in (0,1)$. 
We define the behavior of Byzantine participants based on a trigger condition.
This trigger is determined by performing a gaussian mean hypothesis test on the first $n_{\text{test}} \vcentcolon= p_{\text{test}} T$ sample, collected during an initial fraction of the total $T$ iterations.
The test evaluates whether the sample mean, projected on $e_1$, corresponds to hypothesis $H_0$ (mean 0, equivalently the dataset is $S$) or to hypothesis $H_{1}$ (mean $-\frac{C}{m}$, equivalently the dataset is $S'$), assuming a known variance $\sigma^2_{\operatorname{DP}}$.
This test processes the statistic $\bar\zeta$ which evaluates at  $-\frac{C}{m} + \bar\xi$ under $H_{1}$, and evaluates at $\bar\xi$ under $H_{0}$, where $\bar\xi = \frac{1}{n_{\text{test}}} \sum_{t=1}^{n_{\text{test}}} {\xi^{(1)}_t}^\intercal e_1$.
Overall, the Byzantine participants undergo a membership inference attack to know whether the gradients from participant $1$ equals $g^{(1)}_t = \xi^{(1)}_t$ or $g^{\prime(1)}_t = -\frac{C}{m}e_1 + \xi^{(1)}_t$, i.e, whether the dataset is $S$ or $S'$.
The decision is made depending on a trigger value $c_{\text{trigger}}$,
\begin{equation*}
    \text{reject } H_0 \text{ if } \zeta < - c_{\text{trigger}}, \text{ else accept } H_0.
\end{equation*}
We denote $Z \sim \mathcal{N}(0, 1)$ a standard gaussian random variable, and $z_\alpha \in \mathbb{R}$ the value such that $\mathbb{P}(Z > z_\alpha) = \alpha$, $\alpha \in (0,1)$. 
Conditioned on $H_0$, $\zeta | H_0 \sim \mathcal{N}(0, \frac{\sigma_{\text{DP}}^2}{n_{\text{test}}})$, hence  setting $c_{\text{trigger}} = z_\alpha \frac{\sigma_{\operatorname{DP}}}{\sqrt{n_{\text{test}}}}$, 
\[
    \mathbb{P}\left( \text{reject } H_0 | H_0 \right) 
    = \mathbb{P}\left( \zeta > c_{\text{trigger}} | H_0 \right) 
    = \mathbb{P}\left( Z > z_\alpha \right) 
    = \alpha.
\]
Let denote $E_1$ the event that the Byzantine workers successfully infer the underlying datasets from the gradient statistics (successful MIA). 
Under the optimal hypothesis test, we have
\begin{align*}
    E_1 &= \Big\{ \bar\xi \text{ does not trigger } H_0 \text{ rejection} \Big\} \cap \Big\{ \text{shifting } \bar\xi \text{ by} -\frac{C}{m} \text{ triggers } H_0 \text{ rejection} \Big\} \\
    &= \Big\{ \bar\xi \geq - c_{\text{trigger}} \Big\} \cap \Big\{ \bar\xi - \frac{C}{m} < -c_{\text{trigger}} \Big\} \\
    &= \Big\{ - c_{\text{trigger}} \leq \bar\xi < \frac{C}{m} - c_{\text{trigger}} \Big\}.
\end{align*}
Because $\operatorname{Law}(\bar\xi) = \operatorname{Law}(-\bar\xi)$, we have
\begin{align*}
    \mathbb{P}\left( E_1 \right) 
    = \mathbb{P}\left( c_{\text{trigger}} - \frac{C}{m} < \bar\xi \leq c_{\text{trigger}} \right) 
    &= \mathbb{P}\left( z_\alpha - \frac{C\sqrt{n_{\text{test}}}}{m\sigma_{\text{DP}}} < Z \leq z_\alpha \right) \\
    &= \mathbb{P}\bigg( z_\alpha - \frac{\varepsilon \sqrt{p_{\text{test}}}}{2\sqrt{2\ln(1/\delta)}} < Z \leq z_\alpha \bigg).
\end{align*}
This means that $\mathbb{P}\left( E_1 \right) \in \Omega(1)$ if and only if there exist a constant $c_1$ such that $\varepsilon \geq c_1 \sqrt{\ln(1/\delta)}$.
The event $E_1$ relates to the best false negative rate $\beta_\alpha = \inf_{\Prob( \text{reject } H_0 \mid H_0) \leq \alpha} \Prob( \text{accept } H_0 \mid H_1) = \Prob(Z - \frac{C\sqrt{n_{\text{test}}}}{m\sdp} \leq z_\alpha) = \Phi(z_\alpha + \frac{\varepsilon\sqrt{p_{\text{test}}}}{2\sqrt{2\ln(1/\delta)}})$ (the Byzantine participants infer $S$ when $S'$)
for a fixed false positive rate $\alpha = \Prob( \text{reject } H_0 \mid H_0) = \Prob(Z>z_\alpha) = 1 - \Phi(z_\alpha)$ (the Byzantine participants infer $S'$ when $S$)
through $\Prob(E_1) = 1 - \alpha - \beta_\alpha = \Phi(z_\alpha) - \Phi(z_\alpha + \frac{\varepsilon\sqrt{p_{\text{test}}}}{2\sqrt{2\ln(1/\delta)}})$.
That is, $\Prob(E_1) = 1 - \Prob(\text{reject } H_0 \mid H_0) - \Prob(\text{accept } H_0 \mid H_1)$ quantifies the Byzantine participants maximum inference capability, representing the probability of MIA success beyond random chance (or MIA advantage).

\paragraph{Byzantine participants behavior under the $\operatorname{SMEA}$ aggregation rule.} 
We assign the following behavior (\Cref{alg:algo1-d-high}) to the Byzantine participants $i \in \{n-f+1, n\}$, for $t \geq n_{\text{test}} +1$.

{\centering
    \begin{minipage}{.45\linewidth}
        \begin{algorithm}[H]
        \caption{Byzantine Participant Behavior}\label{alg:algo1-d-high}
        \begin{algorithmic}
            \STATE \textcolor{gray}{\# $\beta$ (defined below) depends on $f$ and $n$} 
            \IF{$\bar\zeta \geq -c_{\text{trigger}}$}
            \STATE \textbf{return} $\beta \sdp e_1$
            \ELSE
            \STATE \textbf{return} $- \beta \sdp e_1$
            \ENDIF
        \end{algorithmic}
        \end{algorithm}
    \end{minipage}
    \par
}\

That is, depending on the input $\bar\zeta$ statistic, which itself depends only on the observed noisy gradients, the Byzantine participants attempt to steer the optimization trajectory toward either positive or negative values (with respect to the hyperplane orthogonal to $e_1$).

In what follows, we decompose our analyze depending on whether the Byzantine participants undergo a successful MIA (i.e., infer whether the dataset is $\mathcal{S}$ or $\mathcal{S}'$ successfully) in the two coupled optimization trajectories ($E_1$), or not ($E_1^{\mathsf{c}}$). 
In the complementary event $E^{\mathsf{c}}_1$, the influence of Byzantine participants is considered negligible, and any potential contribution to instability from this case is disregarded. 
Consequently, we fix their behavior to sending an arbitrary vector, for example, $0$.
Under $E_1$, the attack~\Cref{alg:algo1-d-high} output $\beta \sdp e_1$ for the trajectory produced by $S'$ (i.e., with $g_t^{(1)\prime}= \frac{1}{m} \sum_{j=1}^m z^{(1,j)\prime} + \xi^{(1)}_t = -\frac{C}{m}e_1 + \xi^{(1)}_t$), and output $- \beta \sdp e_1$ for the trajectory produced by $S$ (i.e., with $g_t^{(1)}= \frac{1}{m} \sum_{j=1}^m z^{(1,j)} + \xi^{(1)}_t = \xi^{(1)}_t$). 
That is, under $E_1$, the Byzantine participants attempt to steer the optimization trajectory in different directions depending on the dataset they have successfully inferred.
The analysis is analogous for both trajectories. 
For the remaining of this paragraph, we work under $E_1$ and the trajectory induced by $S$.

Let $\Lambda \vcentcolon= \E \| \Sigma_{\mathcal{H}} \|_{\operatorname{sp}}$. 
Following~\Cref{alg:algo1-d-high} after $n_{\text{test}}$ steps,
all Byzantine participants $i \in \mathcal{B} \vcentcolon= \{n-f+1, \ldots, n\}$ send to the server a constant vectors in the direction $e_1$, 
\begin{equation}
    \text{for all } i \in \mathcal{B}, t\in \{n_{\text{test}} + 1, \ldots, T\}, \quad g^{(i)}_t \vcentcolon= \beta \sigma_{\text{DP}} e_1,
\end{equation}
where 
\begin{equation}\label{eq:beta-def}
    \beta \vcentcolon= c \frac{n-f}{\sqrt{f(n-2f)}} \frac{\sqrt{\Lambda}}{\sigma_{\text{DP}}},
\end{equation}
where $c \in (0, c_{\operatorname{opt}})$ and $c_{\operatorname{opt}}$ are parameter-independent constants defined in step \textit{(iv)} below. 
$\beta$ is chosen so that $\mathrm{SMEA}$ consistently selects the Byzantine vectors, as explained below.

\begin{itemize}
    \item 
    \noindent\textbf{Note on $\Lambda$ and the magnitude of $\frac{\sqrt{\Lambda}}{\sigma_{\text{DP}}}$.}
    By standard derivation, $\Sigma_{\mathcal{H}}$ is identically distributed to $\frac{\sigma_{\text{DP}}^2}{n-f} X X^\intercal$, where $X \in \mathbb{R}^{d \times (n-f-1)}$ is a random matrix with i.i.d.\ standard normal entries $\mathcal{N}(0,1)$\footnote{
            Indeed, denoting $\Xi_t$ the matrix whose column are $\frac{\xi_t^{(i)}}{\sigma_{\text{DP}}}$ for $i\in\mathcal{H}$ and $P_{n-f} = I_{n-f} - \frac{1}{n-f} 1_{n-f}1_{n-f}^\intercal$ the centering matrix which is an orthogonal projection matrix, $\Sigma_t = \frac{\sigma_{\text{DP}}^2}{n-f} \Xi_t P_{n-f} \Xi_t^\intercal$.
            $P_{n-f}$ is diagonalizable, $P_{n-f} = U \Lambda U^\intercal$, where $U$ is an orthogonal matrix and $\Lambda \vcentcolon= \operatorname{diag(1, \ldots, 1, 0)}$ .
            Consequently, $\Sigma_t = \frac{\sigma_{\text{DP}}^2}{n-f} (\Xi_t U) \Lambda (\Xi_t U)^\intercal = \frac{\sigma_{\text{DP}}^2}{n-f} \sum_{j=1}^{n-f-1} (\Xi_t U)_j (\Xi_t U)_j^\intercal \vcentcolon= \frac{\sigma_{\text{DP}}^2}{n-f} X X^\intercal$, with $X_j = (\Xi_t U)_j$.
        }.
    Hence, $\lambda_{\max}(\Sigma_{\mathcal{H}}) = \frac{\sigma_{\text{DP}}^2}{n-f} s_{\max}(X)^2$, where $s_{\max}(X)$ is the maximum singular value of $X$. 
    Taking the expectation yields
    \begin{equation}\label{eq:true-lambda}
        \Lambda = \mathbb{E}[\lambda_{\max}(\Sigma_{\mathcal{H}})] = \frac{\sigma_{\text{DP}}^2}{n-f} \mathbb{E}[s_{\max}(X)^2].
    \end{equation}
    To relate $\Lambda$ to the Marchenko-Pastur baseline $\sqrt{n-f-1} + \sqrt{d}$, we multiply and divide Equation~\eqref{eq:true-lambda} by this baseline, yielding
    \begin{equation}\label{eq:true-lambda-ratio}
        \frac{\sqrt{\Lambda}}{\sigma_{\text{DP}}} = c_\Lambda \sqrt{\frac{n-f-1}{n-f}} \left( 1 + \sqrt{\frac{d}{n-f-1}} \right), \quad \text{where} \quad c_\Lambda \vcentcolon= \frac{\sqrt{\mathbb{E}[s_{\max}(X)^2]}}{\sqrt{n-f-1} + \sqrt{d}}.
    \end{equation}
    To prove that $c_\Lambda$ can be considered as a constant, we rely on results from high-dimensional probability~\citep[Chapter 7]{vershynin_high-dimensional_2018}.
    By~\cite[Theorems 7.3.1]{vershynin_high-dimensional_2018} and standard Gaussian comparison arguments, there exists a constant $C>0$ such that
    \begin{equation*}
        \sqrt{d} + \sqrt{n-f-1} - C \leq \mathbb{E}[s_{\max}(X)] \leq \sqrt{d} + \sqrt{n-f-1}.
    \end{equation*}
    Furthermore, because the map $X \mapsto s_{\max}(X)$ is $1$-Lipschitz\footnote{
        $|s_{\max}(X) - s_{\max}(Y)| = | \|X\|_{\operatorname{sp}} - \|Y\|_{\operatorname{sp}} | \leq \|X-Y\|_{\operatorname{sp}} \leq \|X-Y\|_{\operatorname{F}}$. 
    }, the Gaussian Poincaré inequality guarantees $\text{Var}(s_{\max}(X)) \leq 1$~\cite[Section 3.7]{boucheron2013concentration}.
    Using the identity $\mathbb{E}[s_{\max}(X)^2] = \big(\mathbb{E}[s_{\max}(X)]\big)^2 + \text{Var}(s_{\max}(X))$ yields
    \begin{equation*}
        (\sqrt{d} + \sqrt{n-f-1} - C)^2 \leq \mathbb{E}[s_{\max}(X)^2] \leq (\sqrt{d} + \sqrt{n-f-1})^2 + 1.
    \end{equation*}
    Taking the square root and dividing by the baseline dimension $(\sqrt{d} + \sqrt{n-f-1})$, we obtain 
    \begin{equation}\label{eq:c_lambda-bound}
        \max\{ \frac{1}{2}, 1 - \frac{C}{2} \} \leq 1 - \frac{C}{\sqrt{d} + \sqrt{n-f-1}} \leq c_\Lambda \leq \sqrt{1 + \frac{1}{(\sqrt{d} + \sqrt{n-f-1})^2}} \leq \sqrt{\frac{5}{4}},
    \end{equation}
    where the numerical lower bound is obtained through the following inequalities 
    $\E s_{\max}(X)^2 \geq \frac{\E \|X\|_F^2}{\min(d,n-f-1)} = \frac{d(n-f-1)}{\min(d,n-f-1)} = \max(d,n-f-1)$, and hence $c_\Lambda \geq \frac{\max(d,n-f-1)}{\sqrt{n-f-1}+\sqrt{d}} \geq \frac{1}{2}$.
\end{itemize}

Let $\Sigma_S \vcentcolon= \frac{1}{|S|} \sum_{i \in S} (g_t^{(i)} - \bar{g}_S)(g_t^{(i)} - \bar{g}_S)^\intercal$ denote the sample covariance matrix of a subset $S$. 
To ensure the aggregator selects the subset $S_t = H_t \cup \mathcal{B}$, where $H_t \subset \mathcal{H}$ is composed of $n-2f$ indices from $\mathcal{H}$, over the honest subset $\mathcal{H}$, we must guarantee with high probability that
\begin{equation}
    \text{for } t\in \{n_{\text{test}} + 1, \ldots, T\}, \quad \lambda_{\max}(\Sigma_{S_t}) < \lambda_{\max}(\Sigma_{\mathcal{H}}).
\end{equation}
That is $\operatorname{SMEA}(g_t^{(1)}, \ldots, g_t^{(n)}) = \bar{g}_{S_t} = \frac{1}{n-f} \sum_{i \in S_t} g_t^{(i)} = \frac{f}{n-f}\beta \sdp e_1 + \frac{n-2f}{n-f}\bar{\xi}_{H_t}$.
Indeed, all other subsets than $S_t$ and $\mathcal{H}$ necessaryly have greater spectral norm\footnote{
    For any intermediate subset $S_k$ containing exactly $0 < k < f$ Byzantine vectors, the expected variance along the direction $e_1$ introduces a between-group mixture term proportional to $k(n-f-k)$. 
    Because this forms a strictly concave quadratic function of $k$, its minimum over the closed interval $[0, f]$ is necessarily achieved at the boundaries ($k=0$ or $k=f$). 
    By tuning $\beta$ to equalize these endpoints with a slight margin favoring $k=f$, we guarantee that $\mathrm{SMEA}$ selects either the purely honest subset $\mathcal{H}$ or the subset $S_t$ over any intermediate mixture.
}.
We demonstrate this via a concentration argument, first bounding the expected spectra of both matrices.


\

To tightly bound $\lambda_{\max}(\Sigma_{S_t})$, we leverage the fact that, due to the subset $\mathcal{B}$, the dominant eigenvector of $\Sigma_{S_t}$ consistently aligns with $e_1$, thereby reducing the dimensional dependence\footnote{
    This avoids the worst-case---loose in our setting---application of Weyl's inequality.
    The reasoning is as follows.
    By definition, $S_t$ is partitioned into two disjoint groups: the honest subset $H_t$ of size $n-2f$ (with empirical mean $\bar{\xi}_{H_t}$ and covariance $\Sigma_{H_t}$) and the Byzantine subset $\mathcal{B}$ of size $f$ (with empirical mean $\beta \sigma_{\text{DP}} e_1$ and covariance $0$). 
    Using the standard decomposition formula for the empirical covariance for the partition $S_t =  H_t \cup \mathcal{B}$, we have
    \begin{align*}
        \Sigma_{S_t} &= \frac{|H_t|}{|S_t|} \Sigma_{H_t} + \frac{|\mathcal{B}|}{|S_t|} \Sigma_{\mathcal{B}} + \frac{|H_t| |\mathcal{B}|}{|S_t|^2} (\bar{\xi}_{H_t} - \beta \sigma_{\text{DP}} e_1)(\bar{\xi}_{H_t} - \beta \sigma_{\text{DP}} e_1)^\intercal \\
        &= \frac{n-2f}{n-f} \Sigma_{H_t} + \frac{f(n-2f)}{(n-f)^2} (\bar{\xi}_{H_t} - \beta \sigma_{\text{DP}} e_1)(\bar{\xi}_{H_t} - \beta \sigma_{\text{DP}} e_1)^\intercal.
    \end{align*}
    We have $\lambda_{\max}((\bar{\xi}_{H_t} - \beta \sigma_{\text{DP}} e_1)(\bar{\xi}_{H_t} - \beta \sigma_{\text{DP}} e_1)^\intercal) = \| \bar{\xi}_{H_t} - \beta \sigma_{\text{DP}} e_1 \|_2^2$ and, 
    by Weyl's inequality, the maximum eigenvalue of a sum of symmetric matrices is bounded by the sum of their maximum eigenvalues,
    \begin{equation*}
        \lambda_{\max}(\Sigma_{S_t}) 
        \leq \text{UB}_{S_t}
        \vcentcolon= \frac{n-2f}{n-f} \lambda_{\max}(\Sigma_{H_t}) + \frac{f(n-2f)}{(n-f)^2} \|\bar{\xi}_{H_t} - \beta \sigma_{\text{DP}} e_1\|_2^2.
    \end{equation*}
    $\text{UB}_{S_t}$ does not leverage the fact that, because of the subset $\mathcal{B}$, the dominant eigenvector aligns with $e_1$, which introduces a dimensional dependence and makes the analysis loose.    
    Note that general alternatives to Weyl's inequality for matrix perturbation exist, such as those based on the Davis–Kahan $\sin \Theta$ theorem~\cite[Theorem 4.1.15]{vershynin_high-dimensional_2018}
    (i.e., eigen-space perturbation theory).
    However, our technique directly leverages the problem structure, making the overall reasoning tighter and simpler.
}.
To do so, we decompose $\Sigma_{S_t}$ into a $2 \times 2$ block matrix structure, where one diagonal block accounts for the covariance along the direction $e_1$, and the other for the covariance in the $(d-1)$-dimensional orthogonal complement $e_1^\perp$.
Geometrically, we demonstrate that the principal eigenvector of $\Sigma_{S_t}$ is overwhelmingly aligned with the attack axis $e_1$. 

Any unit vector $v \in \mathbb{R}^d$ can be uniquely decomposed as $v = x e_1 + y w$, where $w \in e_1^\perp$ is a unit vector and $x^2 + y^2 = 1$. 
Hence,
\begin{align}\label{eq:22-decompo}
    v^\intercal \Sigma_{S_t} v 
    &= x^2 (e_1^\intercal \Sigma_{S_t} e_1) + y^2 (w^\intercal \Sigma_{S_t} w) + 2xy (e_1^\intercal \Sigma_{S_t} w) \nonumber \\
    &= x^2 \Sigma_{11} + y^2 (w^\intercal \Sigma_{22} w) + 2xy w^\intercal \Sigma_{12},
\end{align}
where we define the following blocks.
\begin{itemize}
    \item $\Sigma_{11} \vcentcolon= e_1^\intercal \Sigma_{S_t} e_1 \in \mathbb{R}$ is the variance along the attack axis $e_1$.
    \item $\Sigma_{22} \vcentcolon= P_{e_1^\perp} \Sigma_{S_t} P_{e_1^\perp} \in \mathbb{R}^{(d-1) \times (d-1)}$ is the orthogonal covariance block, where $P_{e_1^\perp} \vcentcolon= I_d - e_1 e_1^\intercal$ denote the orthogonal projection matrix zeroing out the first coordinate.
    Because the attack vector spans $e_1$, $\Sigma_{22}$ is entirely independent of the Byzantine vectors.
    \item $\Sigma_{12} \vcentcolon= P_{e_1^\perp} \Sigma_{S_t} e_1 \in \mathbb{R}^{d-1}$ is the cross-covariance vector.
    Note that we identify $\begin{pmatrix} 0 & \Sigma_{2, 1} & \ldots & \Sigma_{d, 1} \end{pmatrix}$ with $\Sigma_{12}^\intercal$, since $w \in e_1^\intercal$. We identify $\Sigma_{22}$ analogously.
\end{itemize}
To bound the maximal eigenvalue, $\lambda_{\max}(\Sigma_{S_t}) = \sup_{v \in \mathbb{S}^{d-1}} v^\intercal \Sigma_{S_t} v$, we aim at upper bounding~\eqref{eq:22-decompo},
\begin{equation*}
    v^\intercal \Sigma_{S_t} v 
    \leq x^2 \Sigma_{11} + y^2 \lambda_{\max}(\Sigma_{22}) + 2|xy| \|\Sigma_{12}\|_2
    = \begin{pmatrix} |x| & |y| \end{pmatrix} \begin{pmatrix} \Sigma_{11} & \|\Sigma_{12}\|_2 \\ \|\Sigma_{12}\|_2 & \lambda_{\max}(\Sigma_{22}) \end{pmatrix} \begin{pmatrix} |x| \\ |y| \end{pmatrix},
\end{equation*}
where $\begin{pmatrix} |x| & |y| \end{pmatrix}$ is a unit vector in $\R^2$.
Consequently, this decomposition yields
\footnote{
    We might use the inequality $\sqrt{x^2 + y^2} \leq \min\{|x| + |y|, |x| + \frac{y^2}{2|x|} \}$ to the square root term to simplify the proof.
    Though, we place our analysis on the edge of what the Byzantine participants can do. 
    That is, in the inequality above, it is difficult to know if $|x|$ and $|y|$ will have the same order or not. 
    As we require to be as tight as possible, we apply Sylvester’s criterion to avoid explicitly handling inequalities.
}
\begin{align}\label{eq:final-ineq-lambdamax}
    \lambda_{\max}(\Sigma_{S_t}) 
    &\leq \lambda_{\max}(\begin{pmatrix} \Sigma_{11} & \|\Sigma_{12}\|_2 \\ \|\Sigma_{12}\|_2 & \lambda_{\max}(\Sigma_{22}) \end{pmatrix}) \nonumber \\
    &= \frac{\Sigma_{11} + \lambda_{\max}(\Sigma_{22})}{2} + \frac{1}{2}\sqrt{\left(\Sigma_{11} - \lambda_{\max}(\Sigma_{22})\right)^2 + 4\|\Sigma_{12}\|_2^2}.
\end{align}
\begin{itemize}
    \item[\textit{(i)}] \textbf{Bounding $\Sigma_{11}$.}
        The variance along $e_1$ decomposes into
        \begin{equation*}
            e_1^\intercal \Sigma_{S_t} e_1 = \underbrace{ \frac{1}{n-f} h_{t,1}^\intercal P_{n-f} h_{t,1} }_{\vcentcolon= V_1} - \underbrace{ \frac{2f\beta \sigma_{\text{DP}}}{(n-f)^2} \sum_{i \in \mathcal{H} \cap S_t} \xi_{t,1}^{(i)} }_{\vcentcolon= V_2} + \underbrace{ \frac{f(n-2f)}{(n-f)^2} \beta^2 \sigma_{\text{DP}}^2 }_{\vcentcolon= V_3},
        \end{equation*}
        where $\xi_{t,1}^{(i)} = {\xi_{t}^{(i)}}^\intercal e_1$, $P_{n-f} = I_{n-f} - \frac{1}{n-f}1_{n-f}1_{n-f}^\intercal$, and $h_t$ is the vector containing the $n-2f$ honest noise from $S_t$, and filled with $0$ for the remaining $f$ dimension.
        By the definition of $\beta$, the deterministic Byzantine contribution simplifies to $V_3 = c^2 \Lambda$. 
        Taking the expectation yields\footnote{
            This expectation is derived using the trace trick for quadratic forms. 
            Because the honest variables are zero-mean, $\mathbb{E}[h_{t,1}^\intercal P_{n-f} h_{t,1}] = \text{Tr}(P_{n-f} \Sigma_h) = \text{Tr}(\Sigma_h) - \frac{1}{n-f}\text{Tr}(1_{n-f}1_{n-f}^\intercal \Sigma_h)$, where $\Sigma_h \vcentcolon= \mathbb{E}[h_{t,1} h_{t,1}^\intercal]$. 
            The vector $h_{t,1} \in \mathbb{R}^{n-f}$ contains $n-2f$ independent honest coordinates $\sim \mathcal{N}(0, \sigma_{\text{DP}}^2)$ and $f$ zeros. 
            Thus, $\Sigma_h$ is diagonal, yielding $\text{Tr}(\Sigma_h) = (n-2f)\sigma_{\text{DP}}^2$. 
            Moreover, $\text{Tr}(1_{n-f}1_{n-f}^\intercal \Sigma_h) = 1_{n-f}^\intercal \Sigma_h 1_{n-f} = (n-2f) \sigma_{\text{DP}}^2$.
            Consequently, $\frac{1}{n-f}\mathbb{E}[h_{t,1}^\intercal P_{n-f} h_{t,1}] = \frac{1}{n-f}\big( n-2f - \frac{n-2f}{n-f} \big)\sigma_{\text{DP}}^2 = \frac{(n-2f)(n-f-1)}{(n-f)^2}\sigma_{\text{DP}}^2$.
        }.
        \begin{equation}\label{eq:UB-11}
            \E[\Sigma_{11}] 
            = \sigma_0^2 + c^2 \Lambda 
            \vcentcolon= U_{11}, \text{ where } \sigma_0^2 \vcentcolon= \frac{(n-2f)(n-f-1)}{(n-f)^2}\sigma_{\text{DP}}^2.
        \end{equation}
        

    \item[\textit{(ii)}] \textbf{Bounding $\E\lambda_{\max}(\Sigma_{22})$ via Loewner partial order.} 
    Let $u_i \vcentcolon= P_{e_1^\perp} g^{(i)}_{t}$ denote the projected gradients and $\mu \vcentcolon= \frac{1}{n-f} \sum_{i \in S_t} u_i$ their subset mean. 
    We have, $\Sigma_{22} = M_{22} - \mu\mu^\intercal$, where $M_{22} \vcentcolon= \frac{1}{n-f} \sum_{i \in S_t} u_i u_i^\intercal$. 
    Because $\mu\mu^\intercal \succeq 0$, we have $\Sigma_{22} \preceq M_{22}$, and hence $\lambda_{\max}(\Sigma_{22}) \leq \lambda_{\max}(M_{22})$. 
    Crucially, because the $f$ Byzantine vectors evaluate to zero in $e_1^\perp$, they contribute zero to the uncentered sum,
    \begin{equation*}
        M_{22} = \frac{1}{n-f} \sum_{i \in H_t} \xi^{(i)}_{t, \perp} {\xi^{(i)}_{t, \perp}}^\intercal 
        \vcentcolon= \frac{\sigma_{\text{DP}}^2}{n-f} X_\perp X_\perp^\intercal,
    \end{equation*}
    where $X_\perp \in \mathbb{R}^{(d-1) \times (n-2f)}$ consists of i.i.d.\ standard normal entries. 
    Because $f \ge 1$, $X_\perp$ is formed by removing columns (discarding $f-1$ excess honest participants) and a row (projecting onto $e_1^\perp$) from the matrix $X \in \mathbb{R}^{d \times (n-f-1)}$ defined in~\eqref{eq:true-lambda}. 
    Thus, $s_{\max}(X_\perp) \leq s_{\max}(X)$ deterministically.
    Furthermore, because $X$ is composed by independent normal random variables, the extra row and columns almost surely project non-trivially onto the dominant singular vectors of $X_\perp$. 
    Consequently, the spectral norm strictly expands with probability 1 (i.e., $\mathbb{P}(s_{\max}(X_\perp) < s_{\max}(X)) = 1$\footnote{
            Let $X_\perp$ be a submatrix obtained by deleting rows or columns from $X$. 
            By definition, $\|X\|_{\text{sp}} = \sup_{\|v\|_2=1} \|Xv\|_2$. 
            For any unit vector $u$ in the domain of $X_\perp$, appending zeros to the deleted coordinates yields a unit vector $\tilde{u}$ in the domain of $X$ such that $\|X\tilde{u}\|_2 = \|X_\perp u\|_2$. 
            Taking the supremum over all such vectors guarantees $\|X_\perp\|_{\text{sp}} \leq \|X\|_{\text{sp}}$ deterministically. Consequently, $\E[s_{\max}(X_\perp)^2] \leq \E[s_{\max}(X)^2]$.
        }). 
    Taking the expectation yields $r_\Lambda \vcentcolon= \frac{\E[s_{\max}(X_\perp)^2]}{\E[s_{\max}(X)^2]} < 1$, and hence
    \begin{equation}\label{eq:UB-22}
        \E\lambda_{\max}(\Sigma_{22}) \leq \E\lambda_{\max}(M_{22}) = r_\Lambda \Lambda \vcentcolon= U_{22}.
    \end{equation}

    \item[\textit{(iii)}] \textbf{Bounding $\E\|\Sigma_{12}\|_2^2$.} 
        Because the Byzantine vectors project to zero in $e_1^\perp$ and evaluate to $\beta \sigma_{\text{DP}}$ in $e_1$, the cross-covariance simplifies to
        \begin{equation*}
            \Sigma_{12} = (1-\alpha) (\Sigma_{H_t})_{12} + \alpha(1-\alpha) (\bar{\xi}_{H_t, 1} - \beta \sigma_{\text{DP}}) \bar{\xi}_{H_t, \perp},
        \end{equation*}
        where $(\Sigma_{H_t})_{12} = \frac{1}{n-2f} \sum_{i \in H_t} (\xi_{i, \perp} - \bar{\xi}_{H_t, \perp})(\xi_{i, 1} - \bar{\xi}_{H_t, 1})$.
        To compute the expected squared norm $\E\|\Sigma_{12}\|_2^2$, we square the sum. 
        By the independence of the Gaussian sample mean and sample covariance, the expected cross-term evaluates to zero (i.e., $\E[(\Sigma_{H_t})_{12}^\intercal \bar{\xi}_{H_t, \perp}] = \E[(\Sigma_{H_t})_{12}]^\intercal \E[\bar{\xi}_{H_t, \perp}] = 0$). 
        This yields the decomposition,
        \begin{equation*}
            \E\|\Sigma_{12}\|_2^2 = (1-\alpha)^2 \E\|(\Sigma_{H_t})_{12}\|_2^2 + \alpha^2(1-\alpha)^2 \E\left[ (\bar{\xi}_{H_t, 1} - \beta \sigma_{\text{DP}})^2 \|\bar{\xi}_{H_t, \perp}\|_2^2 \right].
        \end{equation*}
        Furthermore, $\bar{\xi}_{H_t,1}$ is independent of $\bar{\xi}_{H_t, \perp}$, we have 
        \begin{equation*}
            \E\left[ (\bar{\xi}_{H_t, 1} - \beta \sigma_{\text{DP}})^2 \|\bar{\xi}_{H_t, \perp}\|_2^2 \right] 
            = \left( \frac{\sigma_{\text{DP}}^2}{n-2f} + \beta^2 \sigma_{\text{DP}}^2\right) U_V^2.
        \end{equation*} 
        Additionally\footnote{
            We write $(\Sigma_{H_t})_{12} = \frac{\sigma_{\text{DP}}^2}{n-2f}X_\perp P_{n-2f} X_1$, with the projection matrix $P_{n-2f} \vcentcolon= I_{n-2f} - \frac{1}{n-2f} 1_{n-2f}1_{n-2f}^\intercal$, and $X_\perp \in \mathbb{R}^{(d-1) \times (n-2f)}$, $X_1 \in \mathbb{R}^{n-2f}$ denote an independent standard normal matrix and vector. 
            Evaluating $\E\|X_{\perp} P_{n-2f} X_1 \|_2^2 = \E[\text{Tr}(X P_{n-2f} \E[ X_1 X_1^\intercal \mid X_\perp ] P_{n-2f} X^\intercal) ] = \E[\text{Tr}(X_\perp P_{n-2f}^2 X_\perp^\intercal)] = \text{Tr}(P_{n-2f} \E[X_\perp^\intercal X_\perp]) = \text{Tr}(P_{n-2f} (d-1) I_{n-f}) = (n-2f-1)(d-1)$ yield the result.
        },
        \begin{equation*}
            \E\|(\Sigma_{H_t})_{12}\|_2^2 
            = \frac{n-2f-1}{(n-2f)^2} (d-1) \sigma_{\text{DP}}^4 
            = \frac{n-2f-1}{n-2f} \sigma_{\text{DP}}^2 U_V^2
            \text{ where } U_V^2 \vcentcolon= \frac{d-1}{n-2f} \sigma_{\text{DP}}^2.
        \end{equation*}
        Thus, 
        \begin{align}\label{eq:UB-12}
            \E\|\Sigma_{12}\|_2^2 
            &= \frac{(n-2f)(n-2f-1)}{(n-f)^2} \sigma_{\text{DP}}^2 U_V^2 + \frac{\alpha^2(1-\alpha)^2}{n-2f}\sigma_{\text{DP}}^2 U_V^2 + \alpha (1-\alpha) U_V^2 c^2 \Lambda \nonumber \\
            &= \frac{1-\alpha}{n-f} \left( n - 2f - 1 + \alpha^2 \right) \sigma_{\text{DP}}^2 U_V^2 + \alpha(1-\alpha) c^2 U_V^2 \Lambda \nonumber \\
            &= (1-\alpha) U_V^2 \left[ \frac{n-2f-1+\alpha^2}{n-f} \sigma_{\text{DP}}^2 + \alpha c^2 \Lambda \right] \nonumber \\
            &= (1-\alpha) U_V^2 \left[ (1-\alpha)\frac{n-f-1-\alpha}{n-f}\sigma_{\text{DP}}^2 + \alpha c^2 \Lambda \right] \nonumber \\
            &\leq (1-\alpha) U_V^2 \left[ (1-\alpha) \frac{n-f-1}{n-f}\sigma_{\text{DP}}^2 + \alpha c^2 \Lambda \right] \vcentcolon= U_{12}^2
        \end{align} 

    \item[\textit{(iv)}] \textbf{The margin via Sylvester’s criterion (fixing $c_{\text{opt}}$ and $\delta$).}
        Let define the following matrices 
        \begin{equation*}
            M \vcentcolon= \begin{pmatrix} \Sigma_{11} & \|\Sigma_{12}\|_2 \\ \|\Sigma_{12}\|_2 & \lambda_{\max}(\Sigma_{22}) \end{pmatrix}
        \end{equation*}
        and, where we used~\eqref{eq:UB-11}~\eqref{eq:UB-22}~\eqref{eq:UB-12},
        \begin{equation*}
            M_{\text{exp}} \vcentcolon= \begin{pmatrix} U_{11} & U_{12} \\ U_{12} & U_{22} \end{pmatrix} 
            = \begin{pmatrix} \sigma_0^2 + c^2 \Lambda & U_{12} \\ U_{12} & r_\Lambda \Lambda \end{pmatrix}.
        \end{equation*}
        We have 
        \begin{equation*}
            \lambda_{\max}(\Sigma_{S_t}) \leq \lambda_{\max}(M) = \lambda_{\max}(M_{\text{exp}} + (M-M_{\text{exp}})) \leq \lambda_{\max}(M_{\text{exp}}) + \lambda_{\max}(\Delta M),
        \end{equation*} 
        where we define $\Delta M = M-M_{\text{exp}}$.
        We require $\lambda_{\max}(M_{\text{exp}}) \leq \Lambda - 2\delta$ for a chosen margin $\delta > 0$. 
        This is equivalent to requiring the shifted matrix $(\Lambda - 2\delta)I - M_{\text{exp}}$ to be positive semi-definite. 
        By Sylvester's Criterion, a $2 \times 2$ symmetric matrix is positive semi-definite if and only if its diagonal terms and its determinant are non-negative
        \begin{align}
            \Lambda - U_{11} - 2\delta 
            &\geq 0, \label{eq:sylvester_1} \\
            (\Lambda - U_{11} - 2\delta) (\Lambda - U_{22} - 2\delta) 
            &\geq U_{12}^2. \label{eq:sylvester_2}
        \end{align}
        
        \textit{Satisfying condition~\eqref{eq:sylvester_1}.} 
        Recall that $\Lambda = \frac{\sigma_{\text{DP}}^2}{n-f} \E[s_{\max}(X)^2]$ for a $d \times (n-f-1)$ standard Gaussian matrix $X$, and $\E[s_{\max}(X)^2] \geq \frac{\E\|X\|_F^2}{\min(d, n-f-1)} = \max(d, n-f-1)$.
        Hence, 
        \begin{equation*}
            (1-\alpha) \Lambda 
            \geq \frac{(1-\alpha)\sigma_{\text{DP}}^2}{n-f} \max(d, n-f-1) 
            = \sigma_0^2 \max(\frac{d}{n-f-1}, 1).
        \end{equation*}
        Consequently,~\eqref{eq:sylvester_1} is possible with the following mild assumptions.
        \begin{itemize}
            \item[(A)] \textit{The Byzantine participant number scale with the network scale}. 
            There exists $\iota$ such that $f \geq \iota n$, hence $\alpha \geq \frac{\iota}{1-\iota}$. 
            Consequently, $\frac{\sigma_0^2}{\Lambda} \leq (1-\frac{\iota}{1-\iota}) \vcentcolon= \rho_{\max} < 1$.
            \item[(B)] \textit{Overparametrization regime}. 
            There exists $\eta > 0$ such that $d \geq (1+\eta)(n-f-1)$. 
            Hence $\frac{\sigma_0^2}{\Lambda} \leq \frac{(1-\alpha)(n-f-1)}{d} \leq \frac{1}{1+\eta} \vcentcolon= \rho_{\max} < 1$.
            \item[(C)] 
            We can use a more refined bound on $\frac{\sigma_0^2}{\Lambda} = (1-\alpha) \frac{n-f-1}{\E[s_{\max}(X)^2]}$.
            Assume there exists $\nu > 0$ such that $d \geq \nu (n-f-1)$.
            We prove that, by~\citep[Chapter 7]{vershynin_high-dimensional_2018}, there exists a constant $C>0$ such that $\E[s_{\max}(X)^2] \geq \left(\sqrt{d}+\sqrt{n-f-1} - C\right)^2$. 
            Because of~\Cref{assump:lowerbound-d-high}, $\alpha \geq \alpha_{\min}$ and $\frac{C}{\sqrt{n-f-1}} \leq \bar\epsilon$. 
            Hence,
            \begin{equation*}
                \frac{\sigma_0^2}{\Lambda}
                \leq \frac{1-\alpha}{\left(\sqrt{\frac{d}{n-f-1}} + 1 - \frac{C}{\sqrt{n-f-1}} \right)^2}
                \leq \frac{1-\alpha_{\min}}{\left(\sqrt{\nu} + 1 - \bar\epsilon\right)^2}
                \vcentcolon= \rho_{\max} < 1.
            \end{equation*}
        \end{itemize}
        Assuming $\rho_{\max} < 1$ exists, and define $c_{\max}^2 \vcentcolon= 1 - \rho_{\max} > 0$. 
        Because $c \in (0, c_{\max})$, and to satisfy the condition~\eqref{eq:sylvester_1}, we fix the margin $\delta$ as follows.
        \begin{equation}\label{eq:delta-def}
            \delta \vcentcolon= c_\delta \Lambda = \frac{\min\{c_{\max}^2 - c^2, 1 - r_\Lambda \}}{4} \Lambda.
        \end{equation}
        Substituting this margin into condition~\eqref{eq:sylvester_1} yields
        \begin{equation*}
            \Lambda - U_{11} - 2\delta 
            = (1-c^2)\Lambda - \sigma_0^2  - 2 \delta 
            \geq (1-\rho_{\max}-c^2)\Lambda - 2 \delta
            = \frac{c_{\max}^2-c^2}{2} \Lambda 
            > 0.
        \end{equation*}        

        \textit{Satisfying condition~\eqref{eq:sylvester_2}.} 
        We require 
        \begin{equation*}
            (\Lambda - U_{11} - 2\delta) (\Lambda - U_{22} - 2\delta) 
            \geq U_{12}^2.
        \end{equation*}
        By factoring out $\Lambda$, and substituting~\eqref{eq:UB-11},~\eqref{eq:UB-22}~\eqref{eq:delta-def}, yields
        \begin{equation}\label{eq;sylverster-inter_gogo}
            (\Lambda - U_{11} - 2\delta) (\Lambda - U_{22} - 2\delta)
            \geq \frac{(c_{\max}^2 - c^2)(1 - r_\Lambda)}{4} \Lambda^2.
        \end{equation}
        Conversely, we upper bound $U_{12}^2$ from~\eqref{eq:UB-12}. 
        Because $\Lambda \geq \frac{d}{n-f} \sigma_{\text{DP}}^2$, hence $\sigma_{\text{DP}}^2 \leq \frac{n-f}{d}\Lambda$. 
        Substituting this alongside $U_V^2 \leq \frac{1}{1-\alpha}\Lambda$ yields
        \begin{equation}\label{eq:sylvester_rhs_final}
            U_{12}^2 
            = (1-\alpha) U_V^2 \left[ (1-\alpha) \frac{n-f-1}{n-f}\sigma_{\text{DP}}^2 + \alpha c^2 \Lambda \right]
            \leq \Lambda^2 \left[ \alpha c^2 + (1-\alpha)\frac{n-f}{d} \right].
        \end{equation}
        Combining the bounds from~\eqref{eq;sylverster-inter_gogo} and~\eqref{eq:sylvester_rhs_final} gives the condition
        \begin{equation*}
            \frac{(c_{\max}^2 - c^2)(1 - r_\Lambda)}{4} \geq \alpha c^2 + (1-\alpha)\frac{n-f}{d}.
        \end{equation*}
        Rearranging,
        \begin{equation}\label{eq:c_bound}
            c^2 \left( \alpha + \frac{1-r_\Lambda}{4} \right) \leq \frac{c_{\max}^2 (1-r_\Lambda)}{4} - (1-\alpha)\frac{n-f}{d}.
        \end{equation}
        Consequently, by~\citep[Chapter 7]{vershynin_high-dimensional_2018}, combined with standard Gaussian concentration, there exist absolute constants $C_1, C_2 > 0$ such that 
        \begin{align*}
            r_\Lambda 
            = \frac{\E[s_{\max}(X_\perp)^2]}{\E[s_{\max}(X)^2]}
            &\leq \left( \frac{\sqrt{d-1} + \sqrt{n-2f} + C_1}{\sqrt{d} + \sqrt{n-f-1} - C_2} \right)^2  \\
            &= \left( \frac{\sqrt{\frac{d-1}{n-f-1}} + \sqrt{\frac{n-2f}{n-f-1}} + \frac{C_1}{\sqrt{n-f-1}}}{\sqrt{\frac{d}{n-f-1}} + 1 - \frac{C_2}{\sqrt{n-f-1}}} \right)^2.
        \end{align*}
        By~\Cref{assump:lowerbound-d-high}, $\sqrt{\frac{d}{n-f-1}} = \sqrt{\frac{d}{n-f}} \sqrt{1+ \frac{1}{n-f-1}} \leq \sqrt{\nu} (1+ \sqrt{\frac{1}{n-f-1}})$,
        and $\sqrt{\frac{n-2f}{n-f-1}} = \sqrt{\frac{n-2f}{n-f}} \sqrt{1+ \frac{1}{n-f-1}} \leq \sqrt{1-\alpha} (1+ \sqrt{\frac{1}{n-f-1}})$.
        We aggregate these terms and define the residual $\epsilon_n \vcentcolon= \frac{\sqrt{\nu} + \sqrt{1-\alpha} + \max(C_1, C_2)}{\sqrt{n-f-1}}$.
        Again, by~\Cref{assump:lowerbound-d-high}, because we assume, $n-f-1 \geq \left( \frac{3(\sqrt{\nu} + \sqrt{1-\alpha_{\min}} + \max\{C_1,C_2\})}{1-\sqrt{1-\alpha_{\min}}} \right)^2$ 
        we have $\epsilon_n \leq \bar{\epsilon} \vcentcolon= \frac{1 - \sqrt{1-\alpha_{\min}}}{3}$. 
        Consequently, because $r_\Lambda$ decreases with respect to $\alpha$ and increases with respect to the residual $\epsilon_n$, substituting their bound yield
        \begin{multline*}
            r_\Lambda 
            = \frac{\E[s_{\max}(X_\perp)^2]}{\E[s_{\max}(X)^2]}
            < \left( \frac{\sqrt{\nu} + \sqrt{1-\alpha} + \epsilon_n}{\sqrt{\nu} + 1 - \epsilon_n} \right)^2 \\
            \leq \left( \frac{\sqrt{\nu} + \frac{1}{3} + \frac{2}{3}\sqrt{1-\alpha_{\min}}}{\sqrt{\nu} + \frac{2}{3} + \frac{1}{3}\sqrt{1-\alpha_{\min}}} \right)^2 
            \vcentcolon= r_{\max} \in (0,1).
        \end{multline*}
        Concurrently, satisfying Sylvester's criterion requires
        \begin{equation*}
            c^2 \left( \alpha + \frac{1-r_\Lambda}{4} \right) 
            \leq \frac{c_{\max}^2 (1-r_\Lambda)}{4} - (1-\alpha)\frac{n-f}{d}.
        \end{equation*}
        This condition is implied when, defining $c_{\text{opt}}^2 \vcentcolon= \frac{c_{\max}^2 (1-r_{\max}) - 4\mu^{-1}}{2 + (1-r_{\max})}$,
        \begin{equation*}
            c^2 
            \leq c_{\text{opt}}^2
            \leq \frac{c_{\max}^2 (1-r_\Lambda) - 4(1-\alpha)\frac{n-f}{d}}{4\alpha + 1-r_\Lambda}.
        \end{equation*}
        Note that, by~\Cref{assump:lowerbound-d-high}, $\mu > \frac{4}{c_{\max}^2 (1-r_{\max})}$ and hence $c_{\text{opt}}^2 \in (0, 1)$.

    \item[\textit{(v)}] \textbf{Probabilistic consequence.}
        Recall that
        \begin{equation*}
            \lambda_{\max}(\Sigma_{S_t}) \leq \lambda_{\max}(M_{\text{exp}}) + \lambda_{\max}(\Delta M),
            \text{ with } \Delta M = M - M_{\text{exp}} \vcentcolon= \begin{pmatrix} \Delta_{11} & \Delta_{12} \\ \Delta_{12} & \Delta_{22} \end{pmatrix}.
        \end{equation*} 
        We prove $\lambda_{\max}(M_{\text{exp}}) \leq \Lambda - 2\delta$ for $\delta = \frac{\min\{c_{\max}^2 - c^2, 1 - r_\Lambda \}}{4} \Lambda$.
        We now control the following tail failure events for $t \in \{n_{\text{test}}+1, \ldots, T\}$ with high probability
        \begin{align*}
            E^{\mathsf{c}}_{2, t} &\vcentcolon= \left\{ \lambda_{\max}(\Sigma_{\mathcal{H}}) \leq \Lambda - \delta \right\}, \\
            E^{\mathsf{c}}_{3, t} &\vcentcolon= \left\{ \lambda_{\max}(\Delta M) \geq \delta \right\}
            \subset \left\{ |\Delta_{11}| \geq \frac{\delta}{4} \right\} \cup \left\{ |\Delta_{12}| \geq \frac{\delta}{4} \right\} \cup \left\{ |\Delta_{22}| \geq \frac{\delta}{4} \right\},
        \end{align*}
        where we used $\lambda_{\max}(\Delta M) \leq \|\Delta M\|_{\operatorname{F}} \leq \sqrt{\Delta_{11}^2 + 2 \Delta_{12}^2 + \Delta_{22}^2} \leq |\Delta_{11}| + 2 |\Delta_{12}| + |\Delta_{22}|$.

        \

        \textit{High probability bound for $E^{\mathsf{c}}_{2, t}$.} 
        Recall that $\Sigma_{\mathcal{H}}$ is identically distributed to $\frac{\sigma_{\text{DP}}^2}{n-f} X X^\intercal$, where $X \in \mathbb{R}^{d \times (n-f-1)}$ is a random matrix with i.i.d.\ standard normal entries $\mathcal{N}(0,1)$.
        Hence, $\lambda_{\max}(\Sigma_{\mathcal{H}})$ behaves as the squared spectral norm of a Gaussian matrix, scaled by $\frac{\sigma_{\text{DP}}^2}{n-f}$. 
        Because the spectral norm map $X \mapsto s_{\max}(X)$ is $1$-Lipschitz continuous, Gaussian Lipschitz concentration~\citep[Theorem 5.6]{boucheron2013concentration} guarantees sub-Gaussian tails for $s_{\max}(X)$. 
        By definition, $\Lambda = \frac{\sigma_{\text{DP}}^2}{n-f} \E[s_{\max}(X)^2]$, hence $E^{\mathsf{c}}_{2, t}$ is equivalent to $s_{\max}(X)^2 \leq \E[s_{\max}(X)^2] - \delta \frac{n-f}{\sigma_{\text{DP}}^2}$. 
        Moreover, by the Gaussian Poincaré inequality~\cite[Section 3.7]{boucheron2013concentration}, the variance of the 1-Lipschitz function is bounded, meaning $\E[s_{\max}(X)^2] \leq (\E[s_{\max}(X)])^2 + 1$. 
        Substituting this yields the implication
        \begin{equation*}
            E^{\mathsf{c}}_{2, t}
            \implies
            \left\{ 
            s_{\max}(X) 
            \leq \E[s_{\max}(X)] - \frac{\delta \frac{n-f}{\sigma_{\text{DP}}^2} - 1}{2\E[s_{\max}(X)]}
            \right\},
        \end{equation*}
        where we used the inequality $\sqrt{x - y} \leq \sqrt{x} - \frac{y}{2\sqrt{x}}$ for $0 < y < x$.
        Because the margin $\delta$ scales proportionally with $\Lambda$, 
        there exists a constant $\tilde c_2 \in (0, 1)$ such that $\delta \frac{n-f}{\sigma_{\text{DP}}^2} - 1 \geq \tilde c_2 \delta \frac{n-f}{\sigma_{\text{DP}}^2}$.
        Furthermore, applying Jensen's inequality, we have $\E[s_{\max}(X)] \leq \sqrt{\E[s_{\max}(X)^2]} = \sqrt{\frac{n-f}{\sigma_{\text{DP}}^2} \Lambda}$. 
        Consequently, the required deviation $s$ from the mean is bounded by
        \begin{equation*}
            s \geq \frac{\tilde c_2 \delta \frac{n-f}{\sigma_{\text{DP}}^2}}{2 \sqrt{\frac{n-f}{\sigma_{\text{DP}}^2} \Lambda}} 
            = \frac{\tilde c_2}{2} \delta \sqrt{\frac{n-f}{\sigma_{\text{DP}}^2 \Lambda}}.
        \end{equation*}
        By Gaussian Lipschitz concentration~\cite[Theorem 5.6]{boucheron2013concentration}, for any $s>0$ we have $\mathbb{P}(s_{\max}(X) \leq \E[s_{\max}(X)] - s) \leq e^{-s^2/2}$. 
        That is, there exists a constant $c_2 > 0$ such that
        \begin{equation*}
            \mathbb{P}\left(E^{\mathsf{c}}_{2, t}\right) 
            \leq \mathbb{P}(s_{\max}(X) \leq \E[s_{\max}(X)] - s)
            \leq \exp\left( - c_2 \frac{\delta^2 (n-f)}{\sigma_{\text{DP}}^2 \Lambda} \right).
        \end{equation*}

        \

        \textit{High probability bound for $E^{\mathsf{c}}_{3, t}$.}
        Concurrently, we bound the entries of $\Delta M$ using concentration argument. 

        \textbf{Bounding $\Prob(|\Delta_{11}| \geq \frac{\delta}{4})$.} Recall that $\Delta_{11} = (V_1 - \E[V_1]) - V_2$. 
        By the Hanson-Wright inequality~\citep[Theorem 1.1]{rudelson2013hanson} applied to $V_1$, and standard Gaussian tail bounds for the term $V_2$, we have for universal constants $\tilde c_1, \tilde c_3 > 0$
        \begin{align*}
            \mathbb{P}\left( |V_1 - \E[V_1]| \geq \frac{\delta}{8} \right) 
            &\leq 2\exp\left( - \tilde c_1 \min \left\{ \frac{\delta^2 (n-f)^2}{64 \sigma_{\text{DP}}^4 (n-2f)}, \frac{\delta (n-f)}{8 \sigma_{\text{DP}}^2} \right\} \right), \\
            \mathbb{P}\left( |V_2| \geq \frac{\delta}{8} \right) 
            &\leq 2\exp\left( - \tilde c_3 \frac{\delta^2 (n-f)^3}{128 f (n-f-1) \beta^2 \sigma_{\text{DP}}^4} \right).
        \end{align*}
        Hence, by the union bound there exisits $\tilde c_{4}>0$ such that
        \begin{align*}
            \Prob\left( |\Delta_{11}| \geq \frac{\delta}{4} \right) 
            &\leq \mathbb{P}\left( |V_1 - \E[V_1]| \geq \frac{\delta}{8} \right) + \mathbb{P}\left( |V_2| \geq \frac{\delta}{8} \right)  \\
            &\leq 4e^{ - \tilde c_{4} \min \left\{ 
                    \frac{\delta^2 (n-f)^2}{64 \sigma_{\text{DP}}^4 (n-2f)}, \frac{\delta (n-f)}{8 \sigma_{\text{DP}}^2},  
                \frac{\delta^2 (n-f)^3}{128 f (n-f-1) \beta^2 \sigma_{\text{DP}}^4}
            \right\} }.
        \end{align*}

        \textbf{Bounding $\Prob(|\Delta_{22}| \geq \frac{\delta}{4})$.} 
        We bound the absolute deviation of $\Delta_{22} \vcentcolon= \lambda_{\max}(\Sigma_{22}) - U_{22}$. 
        Recall that $\Sigma_{22} = M_{22} - \mu\mu^\intercal$, where $\mu \vcentcolon= \frac{1}{n-f} \sum_{i \in H_t} \xi^{(i)}_{t,\perp} = \frac{\sigma_{\text{DP}}}{n-f} X_\perp 1_{n-2f}$, $M_{22} = \frac{\sigma_{\text{DP}}^2}{n-f} X_\perp X_\perp^\intercal$, and by definition, $U_{22} = \E[\lambda_{\max}(M_{22})]$.
        We recall that $X_\perp \in \mathbb{R}^{(d-1) \times (n-2f)}$ is the matrix with standard normal entries and that $1_{n-2f} \in \mathbb{R}^{n-2f}$ is the all-ones vector. 

        For the upper tail, because $\mu\mu^\intercal \succeq 0$, we have $\Sigma_{22} \preceq M_{22}$, which implies $\lambda_{\max}(\Sigma_{22}) \leq \lambda_{\max}(M_{22})$. 
        Consequently, $\Delta_{22} \leq \lambda_{\max}(M_{22}) - U_{22}$.
        By a analogous derivation done for bounding $E^{\mathsf{c}}_{3, t}$ in high probability, there exists an absolute constant $\tilde c_{5} > 0$ such that
        \begin{equation*}
            \mathbb{P}\left(\Delta_{22} \geq \frac{\delta}{4}\right) 
            \leq \mathbb{P}\left( \lambda_{\max}(M_{22}) - U_{22} \geq \frac{\delta}{4} \right) 
            \leq \exp\left( - \tilde c_{5} \frac{\delta^2 (n-f)}{\sigma_{\text{DP}}^2 \Lambda} \right).
        \end{equation*}

        For the lower tail, because $n-f > n-2f$, we have $\frac{1}{(n-f)^2} < \frac{1}{(n-f)(n-2f)}$ and we define a surrogate matrix $\Sigma_0$ with the larger penalty to obtain the positive-semidefinite ordering
        \begin{align*}
            \Sigma_{22} 
            &\succeq M_{22} - \frac{\sigma_{\text{DP}}^2}{(n-f)(n-2f)} X_\perp 1_{n-2f} 1_{n-2f}^\intercal X_\perp^\intercal \\
            &= \frac{\sigma_{\text{DP}}^2}{n-f} X_\perp \left( I_{n-2f} - \frac{1}{n-2f} 1_{n-2f} 1_{n-2f}^\intercal \right) X_\perp^\intercal 
            \vcentcolon= \Sigma_0.
        \end{align*}
        The matrix $I_{n-2f} - \frac{1}{n-2f} 1_{n-2f} 1_{n-2f}^\intercal$ is a projection matrix of rank $n-2f-1$. 
        Therefore, by the rotational invariance of the Gaussian matrix $X_\perp$, the surrogate matrix $\Sigma_0$ is identically distributed to $\frac{\sigma_{\text{DP}}^2}{n-f} \tilde{X}_\perp \tilde{X}_\perp^\intercal$, where $\tilde{X}_\perp \in \mathbb{R}^{(d-1) \times (n-2f-1)}$ consists of i.i.d.\ standard normal entries.
        Because $\Sigma_{22} \succeq \Sigma_0$, we have $\lambda_{\max}(\Sigma_{22}) \geq \lambda_{\max}(\Sigma_0)$. 
        Thus, the lower deviation of $\Sigma_{22}$ is contained within the lower tail of $\Sigma_0$. 
        Again, we apply the Gaussian Lipschitz concentration~\citep[Theorem 5.6]{boucheron2013concentration} and analogous derivation as done for $E^{\mathsf{c}}_{2, t}$.

        Consequently, combining the upper tail of $M_{22}$ and the lower tail of $\Sigma_0$, there exists a constant $\tilde c_6 > 0$ such that
        \begin{equation*}
            \mathbb{P}\left( |\Delta_{22}| \geq \frac{\delta}{4} \right) 
            \leq 2 \exp\left( - \tilde c_6 \frac{\delta^2 (n-f)}{\sigma_{\text{DP}}^2 \Lambda} \right).
        \end{equation*}

        \textbf{Bounding $\mathbb{P}(|\Delta_{12}| \geq \frac{\delta}{4})$.} 
        We bound the absolute deviation $|\Delta_{12}| = |\|\Sigma_{12}\|_2 - U_{12}|$. 
        To map this to concentration inequalities, we transition to the squared norm and prove that it is sufficient to control the deviation $|\|\Sigma_{12}\|_2^2 - \E\|\Sigma_{12}\|_2^2|$. 
        From the expansion in~\eqref{eq:UB-12}, we have $U_{12}^2 = \E\|\Sigma_{12}\|_2^2 + \frac{\alpha(1-\alpha)^2}{n-f}\sigma_{\text{DP}}^2 U_V^2$.
        For readibility, define $G \vcentcolon= \frac{\alpha(1-\alpha)^2}{n-f}\sigma_{\text{DP}}^2 U_V^2 = \frac{\alpha(1-\alpha)^2}{n-f} \frac{d-1}{n-2f} \sigma_{\text{DP}}^4$. 
        Squaring the bounds of the absolute deviation event $\{|\|\Sigma_{12}\|_2 - U_{12}| \geq \delta/4\}$, we have the following.
        \begin{itemize}
            \item \textit{The upper tail $\|\Sigma_{12}\|_2 \geq U_{12} + \frac{\delta}{4}$.} 
            The squared upper tail, requires $\|\Sigma_{12}\|_2^2 \geq U_{12}^2 + \frac{\delta}{2}U_{12} + \frac{\delta^2}{16}$. 
            Define $\tau \vcentcolon= \frac{\delta U_{12}}{8}$, we have
            \begin{align*}
                \Prob \left( \|\Sigma_{12}\|_2 \geq U_{12} + \frac{\delta}{4} \right) 
                &= \Prob \left( \|\Sigma_{12}\|_2^2 \geq U_{12}^2 + \frac{\delta}{2}U_{12} + \frac{\delta^2}{16} \right) \\
                &= \Prob \left( \|\Sigma_{12}\|_2^2 \geq \E\|\Sigma_{12}\|_2^2 + \tau + \left( G + \frac{3 \delta U_{12}}{8} + \frac{\delta^2}{16}\right) \right) \\
                &\geq \Prob \left( \|\Sigma_{12}\|_2^2 \geq \E\|\Sigma_{12}\|_2^2 + \tau \right).
            \end{align*}
            \item \textit{The lower tail $\|\Sigma_{12}\|_2 \leq U_{12} - \frac{\delta}{4}$.} 
            First, note that if $U_{12} < \frac{\delta}{4}$ the probability of the event is 0. 
            We only need to control the regime where $U_{12} > \frac{\delta}{4}$, which is equivalent to $2\tau = \frac{\delta U_{12}}{4} \geq \frac{\delta^2}{16}$.
            Under this condition, the squared lower tail event write 
            \begin{equation*}
                \left\{
                \|\Sigma_{12}\|_2^2 
                \leq U_{12}^2 - \frac{\delta}{2}U_{12} + \frac{\delta^2}{16} 
                = \E\|\Sigma_{12}\|_2^2 - \tau 
                - (2 \tau - \frac{\delta^2}{16})
                - (\tau - G)
                \right\}
            \end{equation*} 
            Hence, if $G \leq \tau$, which we prove below, we have
            \begin{align*}
                \Prob \left( \|\Sigma_{12}\|_2 \leq U_{12} - \frac{\delta}{4} \right) 
                &= \Prob \left( \|\Sigma_{12}\|_2^2 \leq U_{12}^2 - \frac{\delta}{2}U_{12} + \frac{\delta^2}{16} \right) \\
                &= \Prob \left( \|\Sigma_{12}\|_2^2 \leq \E\|\Sigma_{12}\|_2^2 - \tau - (2 \tau - \frac{\delta^2}{16}) - (\tau - G) \right) \\
                &\geq \Prob \left( \|\Sigma_{12}\|_2^2 \leq \E\|\Sigma_{12}\|_2^2 - \tau \right).
            \end{align*}

            \item \textit{Ensuring $G \leq \tau$.}
                Recall that $G^2 = \frac{\alpha^2(1-\alpha)^4}{(n-f)^2}\sigma_{\text{DP}}^4 U_V^4$, where $U_V^2 = \frac{d-1}{n-2f} \sigma_{\text{DP}}^2$, and
                \begin{equation*}
                    U_{12}^2 = (1-\alpha) U_V^2 \left[ (1-\alpha) \frac{n-f-1}{n-f}\sigma_{\text{DP}}^2 + \alpha c^2 \Lambda \right].
                \end{equation*}
                Hence, the requirement $G \leq \tau = \frac{c_{\delta} \Lambda U_{12}}{8}$ is equivalent to
                \begin{align*}
                    \frac{\alpha^2(1-\alpha)^3}{(n-f)^2}\sigma_{\text{DP}}^4 U_V^2 
                    \leq \left(\frac{c_{\delta}}{8}\right)^2 \Lambda^2 \left[ (1-\alpha) \frac{n-f-1}{n-f}\sigma_{\text{DP}}^2 + \alpha c^2 \Lambda \right].
                \end{align*}
                Applying~\eqref{eq:true-lambda-ratio} and Assumption~\ref{assump:lowerbound-d-high}, we have
                \begin{equation*}
                    \Lambda \geq c_\Lambda^2 \frac{d}{n-f} \sigma_{\text{DP}}^2 \left( 1 + \sqrt{\frac{n-f-1}{d}} \right)^2 
                    \geq \frac{c_\Lambda^2 \mu}{\min(1, \nu)} \sigma_{\text{DP}}^2.
                \end{equation*}
                Substituting this lower bound and $U_V^2$, the condition $G \leq \tau$ is implied by
                \begin{align*}
                    \frac{\alpha^2(1-\alpha)^2}{n-f} 
                    \leq d \left( \frac{c_{\delta} c_\Lambda^2}{8 \min(1, \nu)} \right)^2 \left[ (1-\alpha) \frac{n-f-1}{n-f} + \alpha \frac{\mu c^2 c_\Lambda^2}{\min(1, \nu)} \right].
                \end{align*}
                The following is stronger inequality, as $d \geq \mu (n-f)$,
                \begin{align*}
                    \mu \left[ (1-\alpha) (n-f-1)(n-f) + \alpha \frac{\mu c^2 c_\Lambda^2}{\min(1, \nu)} (n-f)^2 \right]
                    \geq
                    \left( \frac{8 \min(1, \nu)}{c_{\delta} c_\Lambda^2} \right)^2 \alpha^2(1-\alpha)^2. 
                \end{align*}
                Because $(n-f)^2 \geq (n-f)(n-f-1) \geq (n-f-1)^2$, we also have the condition
                \begin{align*}
                    \mu (n-f-1)^2 \left[ 1-\alpha + \alpha \frac{\mu c^2 c_\Lambda^2}{\min(1, \nu)} \right]
                    \geq
                    \left( \frac{8 \min(1, \nu)}{c_{\delta} c_\Lambda^2} \right)^2 \alpha^2(1-\alpha)^2. 
                \end{align*}
                The following is stronger inequality
                \begin{align*}
                    (n-f-1)^2 \left[ \frac{1}{2} + \frac{ \alpha_{\min} \mu c^2 c_\Lambda^2}{\min(1, \nu)} \right]
                    \geq \left( \frac{2 \min(1, \nu) }{ \sqrt{\mu} c_{\delta} c_\Lambda^2} \right)^{2}.
                \end{align*}
                Hence
                \begin{align*}
                    n-f-1
                    \geq \frac{2 \min(1, \nu) }{ c_{\delta} c_\Lambda^2 \sqrt{\mu} \sqrt{\frac{1}{2} + \frac{ \alpha_{\min} \mu c^2 c_\Lambda^2}{\min(1, \nu)}}}.
                \end{align*}
                Using the fact that $\frac{1}{2} \leq c_\Lambda \leq \sqrt{\frac{5}{4}}$, we have 
                \begin{align*}
                    n-f-1
                    \geq \frac{16 \min(1, \nu) }{ c_{\delta} \sqrt{\mu} \sqrt{ 2 + \frac{ \alpha_{\min} \mu c^2}{\min(1, \nu)}}},
                \end{align*}
                which is ensured by~\Cref{assump:lowerbound-d-high}.
        \end{itemize}
        Consequently, we prove $\Prob(|\Delta_{12}| \geq \frac{\delta}{4}) \geq \Prob(|\|\Sigma_{12}\|_2^2 - \E\|\Sigma_{12}\|_2^2| \geq \tau)$. We next control the right hand side.

        \textit{Bounding $\mathbb{P}(|\|\Sigma_{12}\|_2^2 - \E\|\Sigma_{12}\|_2^2| \geq \tau)$.} 
        Recall the cross-covariance decomposition
        \begin{equation*}
            \Sigma_{12} = (1-\alpha) (\Sigma_{H_t})_{12} + \alpha(1-\alpha) (\bar{\xi}_{H_t, 1} - \beta \sigma_{\text{DP}}) \bar{\xi}_{H_t, \perp}.
        \end{equation*}
        Both terms are sums of scaled products of independent standard Gaussian vectors and matrices. 
        Indeed, for the first term, we use the centering projection $P_{n-2f} \vcentcolon= I_{n-2f} - \frac{1}{n-2f} 1_{n-2f}1_{n-2f}^\intercal$. 
        Let $X_\perp \in \mathbb{R}^{(d-1) \times (n-2f)}$ and $X_1 \in \mathbb{R}^{n-2f}$ denote an independent standard normal matrix and vector. 
        We can write
        \begin{equation*}
            (\Sigma_{H_t})_{12} 
            = \frac{\sigma_{\text{DP}}^2}{n-2f} X_\perp \left(I_{n-2f} - \frac{1}{n-2f} 1_{n-2f} 1_{n-2f}^\intercal\right) X_1
            = \frac{\sigma_{\text{DP}}^2}{n-2f} X_\perp P_{n-2f} X_1.
        \end{equation*}
        For the second term, we have $\bar{\xi}_{H_t, 1} \sim \mathcal{N}(0, \frac{\sigma_{\text{DP}}^2}{n-2f})$ and $\bar{\xi}_{H_t, \perp} \sim \mathcal{N}(0, \frac{\sigma_{\text{DP}}^2}{n-2f} I_{d-1})$. 
        Reparameterizing with independent standard normal variable $z_1 \sim \mathcal{N}(0, 1)$ and $\bar{\xi}_{H_t, \perp} = \frac{\sigma_{\text{DP}}}{n-2f} X_\perp 1_{n-2f}$, where $\frac{X_\perp 1_{n-2f}}{\sqrt{n-2f}} \sim \mathcal{N}(0, I_{d-1})$, we have
        \begin{equation*}
            ( \bar{\xi}_{H_t, 1} - \beta \sigma_{\text{DP}} ) 
            = \frac{\sigma_{\text{DP}}^2}{(n-2f)^{\frac{3}{2}}} (z_1 - \beta \sqrt{n-2f})  X_\perp 1_{n-2f}.
        \end{equation*}
        Consequently, we have
        \begin{align*}
            \Sigma_{12} 
            &= \frac{(1-\alpha)\sigma_{\text{DP}}^2}{n-2f} X_\perp P_{n-2f} X_1 
            + \frac{\alpha(1-\alpha)\sigma_{\text{DP}}^2}{(n-2f)^{\frac{3}{2}}} \left( z_1 + \beta \sqrt{n-2f} \right) X_\perp 1_{n-2f}, \\
            &= X_\perp \underbrace{\left[ \frac{(1-\alpha)\sigma_{\text{DP}}^2}{n-2f} P_{n-2f} X_1 + \frac{\alpha(1-\alpha)\sigma_{\text{DP}}^2}{(n-2f)^{\frac{3}{2}}} (z_1 - \beta \sqrt{n-2f}) 1_{n-2f} \right]}_{\vcentcolon= v_{\text{cond}}}.
        \end{align*}



        We condition on the $\sigma$-algebra $\mathcal{F}_1 \vcentcolon= \sigma(X_1, z_1)$. 
        Given $\mathcal{F}_1$, the vector $v_{\text{cond}} \in \mathbb{R}^{n-2f}$ is deterministic. 
        Because $X_\perp$ consists of independent standard Gaussians, right-multiplying by $v_{\text{cond}}$ guarantees $\Sigma_{12} \mid \mathcal{F}_1 \sim \mathcal{N}(0, \|v_{\text{cond}}\|_2^2 I_{d-1})$. 
        The squared norm of such combination is $\|a (P_{n-2f} X_1) + b (1_{n-2f})\|_2^2 = a^2 \|P_{n-2f} X_1\|_2^2 + b^2 (n-2f) + 2ab X_1^\intercal P_{n-2f}^\intercal 1_{n-2f} = a^2 \|P_{n-2f} X_1\|_2^2 + b^2 (n-2f)$\footnote{
            $P_{n-2f}^\intercal 1_{n-2f} = 1_{n-2f} -  \frac{1}{n-2f} 1_{n-2f} 1_{n-2f}^\intercal 1_{n-2f} = 0$.
        }.
        Consequently, the conditional variance simplifies to
        \begin{equation*}
            \sigma_{\text{cond}}^2 \vcentcolon= \|v_{\text{cond}}\|_2^2 = \left( \frac{(1-\alpha)\sigma_{\text{DP}}^2}{n-2f} \right)^2 \|P_{n-2f} X_1\|_2^2 + \left( \frac{\alpha(1-\alpha)\sigma_{\text{DP}}^2}{n-2f} \right)^2 (z_1 - \beta \sigma_{\text{DP}})^2.
        \end{equation*}
        As a result, we have the following distributional equality $\|\Sigma_{12}\|_2^2 \mid \mathcal{F}_1 \stackrel{d}{=} \sigma_{\text{cond}}^2  W$, where $W \sim \chi^2_{d-1}$. 
        Crucially, the standard $\chi^2_{d-1}$ distribution of $W$ originates from $X_\perp$ and is hence independent of $\mathcal{F}_1$. 
        By the law of total probability, integrating out $\mathcal{F}_1$ proves that the unconditional distribution of $\|\Sigma_{12}\|_2^2$ is equivalent to the product of the two independent random variables, $\|\Sigma_{12}\|_2^2 \overset{d}{=} \sigma_{\text{cond}}^2 W$. 
        Thus, $\E\|\Sigma_{12}\|_2^2 = \E[\sigma_{\text{cond}}^2 W] = \E[\sigma_{\text{cond}}^2 ]\E[W]$.
        We bound the product deviation using the identity
        \begin{align*}
            \|\Sigma_{12}\|_2^2 - \E\|\Sigma_{12}\|_2^2 
            &= \sigma_{\text{cond}}^2 W - \E[\sigma_{\text{cond}}^2 ] \E[W] \\
            &= (\sigma_{\text{cond}}^2-\E[\sigma_{\text{cond}}^2])(W-\E[W]) \\
            &\qquad+ (\sigma_{\text{cond}}^2-\E[\sigma_{\text{cond}}^2])\E[W] + \E[\sigma_{\text{cond}}^2](W-\E[W]).
        \end{align*}
        We define the intersection of independent high-probability events 
        \begin{equation*}
            \mathcal{E} \vcentcolon= \{ |\sigma_{\text{cond}}^2-\E[\sigma_{\text{cond}}^2]| \leq t_A \} \cap \{ |W - \E[W]| \leq t_B \}.
        \end{equation*} 
        Define $t_A = \min\{ \frac{\tau}{3\E[W]}, \sqrt{\frac{\tau}{3}} \}$ and $t_B = \min\{ \frac{\tau}{3\E[\sigma_{\text{cond}}^2]}, \sqrt{\frac{\tau}{3}} \}$.
        By the triangle inequality, conditioned on $\mathcal{E}$, we have 
        \begin{equation*}
            |\|\Sigma_{12}\|_2^2 - \E\|\Sigma_{12}\|_2^2| \leq t_A t_B + t_A \E[W] + t_B \E[\sigma_{\text{cond}}^2] \leq \tau.
        \end{equation*}
        Consequently,
        \begin{equation*}
            \Prob(|\Delta_{12}| \geq \frac{\delta}{4}) 
            \leq \Prob(|\|\Sigma_{12}\|_2^2 - \E\|\Sigma_{12}\|_2^2| \geq \tau)  
            \leq \Prob(\mathcal{E}).
        \end{equation*}

        \textit{Bounding $\Prob(|W - \E[W]| \geq t_B)$.}
        We evaluate the concentration of the standard $\chi^2_{d-1}$ variable $W$. 
        We apply the Hanson-Wright inequality~\cite[Theorem 6.2.2]{vershynin_high-dimensional_2018}\footnote{
            With the theorem notation, we have $A=I_{d-1}$, hence $\|I_{d-1}\|_{\text{F}}^2 = d-1$ and $\|I_{d-1}\|_{\operatorname{sp}} = 1$.
            Moreover the sub-Gaussian norm $K$ is a constant for standard normals, hence we absorb it in the constants.
        }, there exists $C_B >0$ such that
        \begin{equation}\label{eq:HS-W}
            \Prob(|W - \E[W]| \leq t_B) \leq 2 \exp \left( -C_B \min\left\{ \frac{t_B^2}{(d-1)}, t_B \right\} \right).
        \end{equation}
        Substituting $t_B = \min\{ \frac{\tau}{3\E[\sigma_{\text{cond}}^2]}, \sqrt{\frac{\tau}{3}} \}$ yields two possible regimes for the tail. 
        Because $\E[\sigma_{\text{cond}}^2]\E[W] = \E \|\Sigma_{12}\|_2^2 \leq U_{12}^2$ and $\E[W] = d-1$, we have $\frac{\tau}{3\E[\sigma_{\text{cond}}^2]} \geq \frac{\delta U_{12} / 8}{3 U_{12}^2 / (d-1)} = \frac{\delta (d-1)}{24 U_{12}}$. 
        Consequently, the exponent in~\eqref{eq:HS-W} is lower-bounded by
        \begin{equation*}
            E_B \vcentcolon= C_B \min \left\{ \frac{\delta^2 (d-1)}{576 U_{12}^2}, \frac{\delta U_{12}}{24 (d-1)}, \frac{\delta (d-1)}{24 U_{12}}, \sqrt{\frac{\delta U_{12}}{24}} \right\}.
        \end{equation*}

        \ 

        \textit{Bounding $\Prob(|\sigma_{\text{cond}}^2-\E[\sigma_{\text{cond}}^2]| \geq t_A)$.}
        We expand $\sigma_{\text{cond}}^2$ as the sum of a centered quadratic form and a linear Gaussian term. 
        Let $v_1 \vcentcolon= \frac{(1-\alpha)\sigma_{\text{DP}}^2}{n-2f}$ and $v_2 \vcentcolon= \frac{\alpha(1-\alpha)\sigma_{\text{DP}}^2}{n-2f}$. 
        We define $V \vcentcolon= [X_1; z_1] \in \mathbb{R}^{n-2f+1}$ and $M_A \vcentcolon= \operatorname{diag}(v_1^2 P_{n-2f}, v_2^2)$. 
        The non-centrality stems from $(z_1 - \beta \sqrt{n-2f})^2 = z_1^2 - 2\beta \sqrt{n-2f} z_1 + \beta^2(n-2f)$. 
        Thus, 
        \begin{equation*}
            \sigma_{\text{cond}}^2-\E[\sigma_{\text{cond}}^2] = \underbrace{V^\top M_A V - \E[V^\top M_A V]}_{\vcentcolon= Q} \underbrace{- 2 c_2^2 \beta \sqrt{n-2f} z_1}_{\vcentcolon= L}.
        \end{equation*}
        To bound $\Prob(|Q + L| \geq t_A)$, we apply the union bound $\Prob(|Q| \geq t_A/2) + \Prob(|L| \geq t_A/2)$. 
        For the quadratic form $Q$, we apply the Hanson-Wright inequality, with $\|M_A\|_{\text{F}}^2 = v_1^4 (n-2f-1) + v_2^4$ and $\|M_A\|_{\operatorname{sp}} = v_1^2$. 
        Substituting $t_A = \min\{ \frac{\tau}{3(d-1)}, \sqrt{\frac{\tau}{3}} \} = \min\{ \frac{\delta U_{12}}{24(d-1)}, \sqrt{\frac{\delta U_{12}}{24}} \}$, the exponent for $\Prob(|Q| \geq t_A/2)$ is
        \begin{equation*}
            E_Q \vcentcolon= C_A \min \left\{ 
                    \frac{\delta^2 U_{12}^2}{2304 \|M_A\|_F^2 (d-1)^2}, 
                    \frac{\delta U_{12}}{96 \|M_A\|_F^2}, 
                    \frac{\delta U_{12}}{48 v_1^2 (d-1)}, 
                    \frac{\sqrt{\delta U_{12}}}{2\sqrt{24} v_1^2} 
                \right\}.
        \end{equation*}
        For $L \sim \mathcal{N}(0, 4 v_2^4 \beta^2 (n-2f))$, we apply the Gaussian tail bound $\exp(-\frac{(t_A/2)^2}{2\sigma_L^2})$, 
        \begin{equation*}
            E_L \vcentcolon= \min \left\{ 
                    \frac{\delta^2 U_{12}^2}{18432 v_2^4 \beta^2 (n-2f) (d-1)^2}, 
                    \frac{\delta U_{12}}{768 v_2^4 \beta^2 (n-2f)} 
                \right\}.
        \end{equation*}

        \ 

        \textit{Wrapping-up.}
        Applying the union bound, we have 
        \begin{align*}
            \Prob\left( |\Delta_{12}| \geq \frac{\delta}{4} \right) 
            &\leq \Prob(|W - \E[W]| \geq t_B) + \Prob(|Q| \geq t_A/2) + \Prob(|L| \geq t_A/2) \\
            &\leq 2e^{-E_B} + 2e^{-E_Q} + 2e^{-E_L} \\
            &\leq 6\exp\left( -\min\{E_B, E_Q, E_L\} \right).
        \end{align*}

        \textbf{Wrapping the high probability bounds.} 
        Crucially, conditioned on the intersection of the success events ($E_{2,t} \cap E_{3,t}$), we have
        \begin{equation}
            \lambda_{\max}(\Sigma_{S_t}) \leq \lambda_{\max}(M_{\text{exp}}) + \delta \leq (\Lambda - 2\delta) + \delta = \Lambda - \delta < \lambda_{\max}(\Sigma_{\mathcal{H}}).
        \end{equation}
        This proves that the attack systematically lures $\operatorname{SMEA}$ with high probability, masking their attack beneath the honest spectral norm.
        Finally, we search for the bottleneck exponent
        \begin{equation*}
            E_{\text{SMEA}} \vcentcolon= \min \{ E_2, E_{11}, E_{22}, E_{12} \},
        \end{equation*}
        where 
        \begin{align*}
            E_2 \ &\vcentcolon= c_2 \frac{\delta^2 (n-f)}{\sigma_{\text{DP}}^2 \Lambda}, \\
            E_{11} &\vcentcolon= \tilde{c}_4 \min\left\{ \frac{\delta^2 (n-f)^2}{64 \sigma_{\text{DP}}^4 (n-2f)}, \frac{\delta (n-f)}{8 \sigma_{\text{DP}}^2}, \frac{\delta^2 (n-f)^3}{128 f (n-f-1) \beta^2 \sigma_{\text{DP}}^4} \right\}, \\
            E_{22} &\vcentcolon= \tilde{c}_5 \frac{\delta^2 (n-f)}{\sigma_{\text{DP}}^2 \Lambda}, \\
            E_{12} &\vcentcolon= \min\{E_B, E_Q, E_L\},
        \end{align*}
        Recall that 
        \begin{equation*}
            \mathbb{P}(E_{2,t}^\mathsf{c}) \leq e^{-E_{2}},
        \end{equation*}
        and
        \begin{equation*}
            E_{3,t}^\mathsf{c} \subset \left\{ |\Delta_{11}| \geq \frac{\delta}{4} \right\} \cup \left\{ |\Delta_{22}| \geq \frac{\delta}{4} \right\} \cup \left\{ |\Delta_{12}| \geq \frac{\delta}{4} \right\}.
        \end{equation*}
        Applying the union bound yields
        \begin{equation*}
            \mathbb{P}(E_{3,t}^\mathsf{c}) \leq 4e^{-E_{11}} + 2e^{-E_{22}} + 6e^{-E_{12}}.
        \end{equation*}
        Consequently, we finally have
        \begin{equation*}
            \mathbb{P}(E_{2,t}^\mathsf{c} \cap E_{2,t}^\mathsf{c} ) 
            \geq 1 - \mathbb{P}(E_{2,t}^\mathsf{c}) - \mathbb{P}(E_{3,t}^\mathsf{c})
            \geq 1 - 13 e^{-E_{\operatorname{SMEA}}}.
        \end{equation*}

        \textbf{The scaling of $E_{\text{SMEA}}$.}
        Observe that, substituting $\delta = c_\delta \Lambda$, all exponents are proportional to  
        \begin{equation*}
            R \vcentcolon= \frac{\Lambda (n-f)}{\sigma_{\text{DP}}^2},
        \end{equation*}
        where, using~\eqref{eq:true-lambda-ratio} and Assumption~\ref{assump:lowerbound-d-high} gives 
        \begin{equation}\label{eq:R-scaling}
            R = c_\Lambda^2 \left( \sqrt{d} + \sqrt{n-f-1} \right)^2 \geq c_\Lambda^2 \max\{ (1+\sqrt{\mu})^2 (n-f-1) , (1 + \sqrt{\max\{\nu^{-1} - 1, 0\}})^2 d\}.
        \end{equation} 
        Indeed, $E_2 = c_2 c_\delta^2 R$, $E_{22} = \tilde{c}_5 c_\delta^2 R$. 
        Moreover, using~\eqref{eq:beta-def} the definition of $\beta$,
        \begin{equation*}
            E_{11} = \tilde{c}_4 \min\left\{ \frac{c_\delta^2}{64} \frac{1}{(1-\alpha)} \frac{R}{n-f}, \; \frac{c_\delta}{8}, \; \frac{c_\delta^2}{128 c^2} (1-\alpha) \frac{n-f}{n-f-1} \right\} R.
        \end{equation*}
        Note that $1-\alpha > \frac{1}{2}$, $\frac{1}{1-\alpha} \geq \frac{1}{1-\alpha_{\min}}$, $\frac{R}{n-f} = \frac{\Lambda}{\sigma_{\text{DP}}^2} \geq c_\Lambda^2 \frac{d}{n-f} \geq c_\Lambda^2 \mu$, hence 
        \begin{equation*}
            E_{11} \geq \frac{\tilde{c}_4 c_\delta}{8} \min\left\{ \frac{c_\delta c_\Lambda^2 \mu}{8 (1-\alpha_{\min})}, 1, \frac{c_\delta}{32 c^2 } \right\} R.
        \end{equation*}

        For $E_{12} = \min\{E_B, E_Q, E_L\}$, 
        to recover $E_{12} \in \Omega(R)$, we actually require to revisit how we fix $t_A$ and $t_B$ as we treat them symmetrically but the yield different concentration.
        Indeed, fixing $\frac{t_B^2}{d-1} = \frac{t_A^2}{\|M_A\|_{\operatorname{F}}}$ and $t_A t_B = \frac{\tau}{3}$ yields $E_{12} \in \Omega(R)$.


\end{itemize}

\ 


Finally, as explained in the beginning of the paragraph, under $E_1$ and the trajectory produced by $S'$ (i.e., with $g_t^{(1)\prime} = \frac{1}{m} \sum_{j=1}^m z^{(1,j)\prime} + \xi^{(1)}_t = -\frac{C}{m}e_1 + \xi^{(1)}_t$), we do an analogous analysis as above, in the opposite direction (i.e., $-e_1$). 
That is, due to~\Cref{alg:algo1-d-high}, the Byzantine participants successfully steer the optimization trajectory in the negative direction. 
Otherwise, in the complementary event $E^{\mathsf{c}}_1$, the influence of Byzantine participants is considered negligible, and any potential contribution to instability from this case is disregarded.

\paragraph{Consequence for the lower bound.} 
Recall that we control the probability of the event $E$, defined as the intersection between the successful hypothesis test $E_1$ and the high-probability event under which $\mathrm{SMEA}$ continuously selects $S_t$,
\[ 
    E = E_1 \cap \bigcap_{t\in\{n_{\text{test}} + 1, \ldots, T\}} \big( E_{2, t} \cap E_{3, t}\big).
\]
For $t, s \in\{n_{\text{test}} + 1, \ldots, T\}$, $t \neq s$, $E_1$ is independent of $E_{2,t}$ and $E_{3,t}$, $E_{2,t} \cap E_{3,t}$ is independent of $E_{2,s} \cap E_{3,s}$.
Consequently, we have
\begin{align}\label{eq:lb-dynamic-hp}
    \mathbb{P}\left( E \right) 
    &= \mathbb{P}\left(E_1\right) \prod_{t\in\{n_{\text{test}} + 1, \ldots, T\}} \mathbb{P}\left(E_{2, t} \cap E_{3, t}\right) \nonumber \\
    &\geq \mathbb{P}\left(E_1\right) \prod_{t\in\{n_{\text{test}} + 1, \ldots, T\}} \left(1 - \mathbb{P}\left(E_{2, t}^\mathsf{c} \right) - \mathbb{P}\left( E_{3, t}^\mathsf{c}\right) \right) \nonumber \\
    &\geq \mathbb{P}\left(E_1\right) {\left( 1 - 13 e^{-E_{\text{SMEA}}} \right)}^{T-n_{\text{test}}} 
    \in \Omega(1)  \text{ if on only if } T \in \mathcal{O}\left( e^{E_{\text{SMEA}}} \right),
\end{align}
which is guaranteed by~\Cref{assump:lowerbound-d-high}.

Conditioned on $E$, the accumulated parameter divergence occurs along the $e_1$ coordinate,
\begin{align*}
    \mathbb{E}\left[ \| \theta_T - \theta^\prime_T \|_2 | E \right] 
    &\geq \mathbb{E}\left[ | (\theta_T - \theta^\prime_T)^\intercal e_1 | | E \right] \\
    &\geq \frac{\gamma}{n-f} \sum_{t=n_{\text{test}} + 1}^T f \beta \sigma_{\text{DP}} \\
    &= c (1-p_{\text{test}}) \gamma T \sigma_{\text{DP}} \frac{f}{n-f} \frac{n-f}{\sqrt{f(n-2f)}} \frac{\sqrt{\Lambda}}{\sigma_{\text{DP}}} \\
    &= c c_\Lambda (1-p_{\text{test}}) \gamma T \sigma_{\text{DP}} \sqrt{\frac{f}{n-2f}} \sqrt{\frac{n-f-1}{n-f}} \left( 1 + \sqrt{\frac{d}{n-f-1}} \right) \\
    &\geq c c_\Lambda (1-p_{\text{test}}) \gamma T \sigma_{\text{DP}} \sqrt{\frac{f}{n-2f}} \left( 1 + \sqrt{\frac{d}{n-f}} \right).
\end{align*}
Because~\eqref{eq:lb-dynamic-hp}, we have $\mathbb{P}\left( E \right) \in \Omega(1)$, which yields 
\begin{align*}
    \sup_{z\in \widebar B(0, C)} \mathbb{E}\left[ \left| \ell(\theta_T, z) - \ell(\theta^\prime_T, z) \right| \right] 
    & \geq C \mathbb{E}\left[ \left\| \theta_T - \theta^\prime_T \right\| | E \right] \mathbb{P}\left( E \right) \\
    & \in \Omega \left( \gamma T C \sigma_{\text{DP}} \sqrt{\frac{f}{n-2f}} \left( 1 + \sqrt{\frac{d}{n-f}} \right) \right).
\end{align*}

\paragraph{Gluing lower bounds.} Moreover, when $\sigma_{\operatorname{DP}} = 0$, the uniform stability is lower bounded by $\Omega\left(\gamma C^2 T \left( \frac{1}{(n-f)m} + \sqrt{\frac{f}{n-2f}}\right)\right)$ when $\frac{n}{3} \leq f < \frac{n}{2}$, 
and by $\Omega\left(\gamma C^2 T \left( \frac{1}{(n-f)m} + \frac{f}{n-2f} \right)\right)$ when $f < \frac{n}{3}$~\citep[Theorem 3.2]{boudou2025generalization}. 
We conclude the proof with the formula $\max\{a,b\} \geq \frac{a+b}{2} \in \Omega(a+b)$, i.e.\ when $\frac{n}{3} \leq f < \frac{n}{2}$ the uniform stability is lower bounded by
\begin{equation*}
    \Omega\left( \gamma C^2 T \left( \frac{1}{(n-f)m} + \sqrt{\frac{f}{n-2f}} \left( 1 + \frac{\sigma_{\operatorname{DP}}}{C} \left( 1 + \sqrt{\frac{d}{n-f}} \right) \right) \right) \right).
\end{equation*}
The case $f < \frac{n}{3}$ is analogous.
Note that $\sqrt{\kappa_{\operatorname{SMEA}}}$ is of the same order as $\sqrt{\frac{f}{n-2f}}$ if we assume there exist a constant $\nu > 0$ such that $f \leq \frac{n}{2+\nu}$, i.e.\ the lower bound matches the upper bound.
\end{proof}

\section{\texorpdfstring{Experimental evaluation details of section~\ref{sec:numerical}}{Experimental evaluations details of section 5}}\label{app:numerical}

We report table of the expected train and test accuracies under varying noise multipliers ($\sigma$) and Byzantine participants ($f$) shown in the main text. 
They are reported as mean $\pm$ standard error.



\begin{table}[!ht]
    \centering
    \caption{\Cref{fig:generalization_error_comparison} (a) expected train and test accuracies under varying noise multipliers ($\sigma$) and Byzantine participants ($f$). 
    Reported as mean $\pm$ standard error.}
    \label{tab:accuracy_results}
    \begin{tabular}{lcccccc}
    \toprule
    \multirow{2}{*}{Noise ($\sigma$)} & \multicolumn{2}{c}{$f=1$} & \multicolumn{2}{c}{$f=2$} & \multicolumn{2}{c}{$f=3$} \\
    \cmidrule(lr){2-3} \cmidrule(lr){4-5} \cmidrule(lr){6-7}
    & Train Acc. & Test Acc. & Train Acc. & Test Acc. & Train Acc. & Test Acc. \\
    \midrule
    32 & $10.41_{\pm 0.18}$ & $10.44_{\pm 0.17}$ & $9.74_{\pm 0.17}$ & $9.72_{\pm 0.17}$ & $9.60_{\pm 0.22}$ & $9.67_{\pm 0.20}$ \\
    16 & $12.30_{\pm 0.23}$ & $12.24_{\pm 0.21}$ & $9.60_{\pm 0.21}$ & $9.61_{\pm 0.19}$ & $9.27_{\pm 0.17}$ & $9.34_{\pm 0.18}$ \\
    8 & $16.22_{\pm 0.22}$ & $16.04_{\pm 0.21}$ & $9.43_{\pm 0.27}$ & $9.47_{\pm 0.26}$ & $8.88_{\pm 0.21}$ & $8.93_{\pm 0.20}$ \\
    4 & $22.37_{\pm 0.19}$ & $21.93_{\pm 0.19}$ & $8.10_{\pm 0.22}$ & $8.05_{\pm 0.22}$ & $6.56_{\pm 0.16}$ & $6.68_{\pm 0.17}$ \\
    2 & $31.83_{\pm 0.22}$ & $31.02_{\pm 0.20}$ & $6.67_{\pm 0.19}$ & $6.73_{\pm 0.18}$ & $2.69_{\pm 0.11}$ & $2.85_{\pm 0.11}$ \\
    1 & $44.43_{\pm 0.19}$ & $43.00_{\pm 0.17}$ & $5.01_{\pm 0.17}$ & $5.12_{\pm 0.18}$ & $2.14_{\pm 0.08}$ & $2.24_{\pm 0.09}$ \\
    0.5 & $54.13_{\pm 0.10}$ & $52.13_{\pm 0.10}$ & $4.26_{\pm 0.25}$ & $4.34_{\pm 0.24}$ & $1.44_{\pm 0.09}$ & $1.52_{\pm 0.09}$ \\
    0.25 & $59.23_{\pm 0.09}$ & $56.86_{\pm 0.07}$ & $4.35_{\pm 0.28}$ & $4.41_{\pm 0.28}$ & $0.79_{\pm 0.04}$ & $0.90_{\pm 0.05}$ \\
    0.125 & $61.22_{\pm 0.06}$ & $58.74_{\pm 0.06}$ & $11.30_{\pm 0.57}$ & $11.12_{\pm 0.55}$ & $0.47_{\pm 0.03}$ & $0.53_{\pm 0.02}$ \\
    \bottomrule
    \end{tabular}
\end{table}

\begin{table}[!ht]
    \centering
    \caption{\Cref{fig:generalization_error_comparison} (b) expected train and test accuracy under varying noise multipliers and heterogeneity $\alpha=\beta$~\cite{li2020federated_synthetic_data}.
    Reported as mean $\pm$ standard srror.}
    \label{tab:accuracy_results_alpha}
    \begin{tabular}{lcccccc}
    \toprule
    \multirow{2}{*}{Noise ($\sigma$)} & \multicolumn{2}{c}{$\alpha=1$} & \multicolumn{2}{c}{$\alpha=5$} & \multicolumn{2}{c}{$\alpha=10$} \\
    \cmidrule(lr){2-3} \cmidrule(lr){4-5} \cmidrule(lr){6-7}
    & Train Acc. & Test Acc. & Train Acc. & Test Acc. & Train Acc. & Test Acc. \\
    \midrule
    32 & $12.31_{\pm 0.46}$ & $12.34_{\pm 0.46}$ & $12.70_{\pm 0.67}$ & $12.69_{\pm 0.67}$ & $17.97_{\pm 1.19}$ & $17.94_{\pm 1.19}$ \\
    16 & $19.49_{\pm 0.61}$ & $19.52_{\pm 0.62}$ & $19.26_{\pm 0.70}$ & $19.21_{\pm 0.70}$ & $41.41_{\pm 1.18}$ & $41.38_{\pm 1.17}$ \\
    8 & $38.66_{\pm 0.71}$ & $38.61_{\pm 0.71}$ & $33.04_{\pm 0.70}$ & $32.95_{\pm 0.70}$ & $73.72_{\pm 1.02}$ & $73.62_{\pm 1.01}$ \\
    4 & $62.93_{\pm 0.52}$ & $62.87_{\pm 0.53}$ & $49.58_{\pm 0.65}$ & $49.51_{\pm 0.65}$ & $84.97_{\pm 0.44}$ & $84.87_{\pm 0.43}$ \\
    2 & $83.49_{\pm 0.37}$ & $83.43_{\pm 0.37}$ & $66.02_{\pm 0.65}$ & $65.92_{\pm 0.66}$ & $89.12_{\pm 0.18}$ & $88.99_{\pm 0.17}$ \\
    1 & $90.16_{\pm 0.08}$ & $90.07_{\pm 0.07}$ & $74.34_{\pm 0.44}$ & $74.19_{\pm 0.43}$ & $90.83_{\pm 0.03}$ & $90.69_{\pm 0.03}$ \\
    0.5 & $90.99_{\pm 0.04}$ & $90.83_{\pm 0.04}$ & $76.22_{\pm 0.28}$ & $75.97_{\pm 0.28}$ & $90.83_{\pm 0.01}$ & $90.70_{\pm 0.01}$ \\
    0.25 & $91.40_{\pm 0.05}$ & $91.11_{\pm 0.04}$ & $75.49_{\pm 0.17}$ & $75.26_{\pm 0.17}$ & $90.73_{\pm 0.00}$ & $90.61_{\pm 0.00}$ \\
    0.125 & $91.61_{\pm 0.03}$ & $91.25_{\pm 0.03}$ & $73.84_{\pm 0.23}$ & $73.66_{\pm 0.22}$ & $90.71_{\pm 0.00}$ & $90.59 _{\pm 0.00}$ \\
    \bottomrule
    \end{tabular}
\end{table}

\begin{table}[!ht]
\centering
\small{
    \caption{
        \Cref{fig:generalization_error_comparison} (c) expected train and test accuracies ($\%$) under varying noise multipliers ($\sigma$) and dimensionality ($d$). 
        These values, reported as mean $\pm$ standard error, are from~\Cref{fig:generalization_error_comparison}, right panel.
    }\label{tab:accuracy_results_d}
    \begin{tabular}{lcccccc}
    \toprule
    \multirow{2}{*}{Noise ($\sigma$)} & \multicolumn{2}{c}{$d=12$} & \multicolumn{2}{c}{$d=25$} & \multicolumn{2}{c}{$d=50$} \\
    \cmidrule(lr){2-3} \cmidrule(lr){4-5} \cmidrule(lr){6-7}
    & Train Acc. & Test Acc. & Train Acc. & Test Acc. & Train Acc. & Test Acc. \\
    \midrule
    32 & $10.77_{\pm 0.21}$ & $10.73_{\pm 0.21}$ & $12.41_{\pm 0.32}$ & $12.32_{\pm 0.32}$ & $10.41_{\pm 0.18}$ & $10.44_{\pm 0.17}$ \\
    16 & $12.98_{\pm 0.22}$ & $12.95_{\pm 0.22}$ & $18.43_{\pm 0.39}$ & $18.32_{\pm 0.40}$ & $12.30_{\pm 0.23}$ & $12.24_{\pm 0.21}$ \\
    8 & $16.44_{\pm 0.29}$ & $16.27_{\pm 0.27}$ & $29.03_{\pm 0.44}$ & $28.67_{\pm 0.45}$ & $16.22_{\pm 0.22}$ & $16.04_{\pm 0.21}$ \\
    4 & $23.51_{\pm 0.32}$ & $23.28_{\pm 0.31}$ & $43.21_{\pm 0.44}$ & $42.52_{\pm 0.45}$ & $22.37_{\pm 0.19}$ & $21.93_{\pm 0.19}$ \\
    2 & $34.67_{\pm 0.33}$ & $34.22_{\pm 0.32}$ & $59.17_{\pm 0.27}$ & $58.13_{\pm 0.28}$ & $31.83_{\pm 0.22}$ & $31.02_{\pm 0.20}$ \\
    1 & $48.85_{\pm 0.23}$ & $48.16_{\pm 0.23}$ & $71.27_{\pm 0.10}$ & $70.01_{\pm 0.11}$ & $44.43_{\pm 0.19}$ & $43.00_{\pm 0.17}$ \\
    0.5 & $57.68_{\pm 0.14}$ & $56.93_{\pm 0.14}$ & $75.20_{\pm 0.06}$ & $73.90_{\pm 0.06}$ & $54.13_{\pm 0.10}$ & $52.13_{\pm 0.10}$ \\
    0.25 & $60.80_{\pm 0.10}$ & $60.11_{\pm 0.09}$ & $76.14_{\pm 0.04}$ & $74.82_{\pm 0.03}$ & $59.23_{\pm 0.09}$ & $56.86_{\pm 0.07}$ \\
    0.125 & $62.25_{\pm 0.07}$ & $61.51_{\pm 0.07}$ & $76.60_{\pm 0.03}$ & $75.25_{\pm 0.03}$ & $61.22_{\pm 0.06}$ & $58.74_{\pm 0.06}$ \\
    \bottomrule
    \end{tabular}
}
\end{table}

\begin{table}[!ht]
\centering
\caption{\Cref{fig:mc_gen_error_mnist-lr} expected train accuracy, test accuracy, and generalization error under varying noise multipliers ($\sigma$). 
Reported as mean $\pm$ standard error across data splits.}
\label{tab:overall_results}
\begin{tabular}{lccc}
\toprule
Noise ($\sigma$) & Train Acc. (\%) & Test Acc. (\%) \\
\midrule
32 & $10.15_{\pm 0.10}$ & $10.12_{\pm 0.13}$ \\
16 & $10.37_{\pm 0.20}$ & $10.30_{\pm 0.19}$  \\
8 & $14.52_{\pm 0.30}$ & $14.10_{\pm 0.22}$  \\
4 & $31.15_{\pm 0.56}$ & $29.95_{\pm 0.38}$  \\
2 & $54.88_{\pm 0.51}$ & $52.65_{\pm 0.38}$  \\
1 & $72.59_{\pm 0.32}$ & $70.00_{\pm 0.14}$  \\
0.5 & $79.88_{\pm 0.38}$ & $77.12_{\pm 0.16}$ \\
0.25 & $82.41_{\pm 0.28}$ & $79.64_{\pm 0.12}$  \\
0.125 & $82.79_{\pm 0.43}$ & $80.06_{\pm 0.19}$ \\
\bottomrule
\end{tabular}
\end{table}

\begin{table}[!ht]
    \centering
    \caption{~\Cref{fig:mc_gen_error_cifar_cnn} expected train and test accuracy under varying noise multipliers. 
    Reported as mean $\pm$ standard error across data splits.}
    \label{tab:overall_results}
    \begin{tabular}{lccc}
    \toprule
    Noise ($\sigma$) & Train Acc. (\%) & Test Acc. (\%) \\
    \midrule
    32 & $10.10_{\pm 0.13}$ & $10.08_{\pm 0.10}$  \\
    16 & $9.90_{\pm 0.03}$ & $10.02_{\pm 0.07}$  \\
    8 & $10.10_{\pm 0.15}$ & $10.13_{\pm 0.15}$  \\
    4 & $11.12_{\pm 0.23}$ & $10.66_{\pm 0.22}$  \\
    2 & $18.87_{\pm 0.50}$ & $17.83_{\pm 0.46}$  \\
    1 & $30.10_{\pm 0.85}$ & $26.74_{\pm 0.71}$ \\
    0.5 & $32.21_{\pm 1.53}$ & $27.31_{\pm 0.71}$ \\
    0.25 & $36.41_{\pm 1.75}$ & $30.67_{\pm 1.08}$ \\
    0.125 & $34.45_{\pm 0.87}$ & $29.74_{\pm 0.75}$ \\
    \bottomrule
    \end{tabular}
\end{table}

\clearpage

\end{document}